\documentclass{article}

\usepackage[a4paper,margin=1in]{geometry}

\usepackage[utf8]{inputenc}
\usepackage[T1]{fontenc}
\usepackage{lmodern}
\usepackage{microtype}

\usepackage{amsmath,amssymb,amsfonts,amsthm,mathtools}
\usepackage{autobreak}

\usepackage{graphicx}
\usepackage{subcaption}
\usepackage{booktabs}
\usepackage{wrapfig}
\usepackage{caption}
\usepackage{placeins}
\usepackage{afterpage}

\usepackage{comment}

\usepackage{algorithm}

\usepackage[numbers,sort&compress]{natbib}
\usepackage{xcolor}
\usepackage{url}
\usepackage{hyperref}

\theoremstyle{plain}
\newtheorem{thm}{Theorem}[section]
\newtheorem{cor}[thm]{Corollary}
\newtheorem{lem}[thm]{Lemma}

\theoremstyle{definition}

\newtheorem{ex}[thm]{Example}

\theoremstyle{remark}
\newtheorem{remark}[thm]{Remark}

\DeclareMathOperator{\vol}{Vol}
\DeclareMathOperator{\tr}{tr}
\DeclareMathOperator{\dist}{dist}
\title{The Geometry of Phase Transitions in Generative Dynamics via Projection Caustics}

\author{
Ryosuke Sakamoto\textsuperscript{1}
\quad
Kotaro Sakamoto\textsuperscript{2}
\\
\textsuperscript{1}Institute for the Advanced Study of Human Biology, \\Institute for Advanced Study, Kyoto University
\\
\textsuperscript{2}Graduate School of Engineering, The University of Tokyo
\\
\texttt{sakamoto.ryosuke.7x@kyoto-u.ac.jp}
\\
\texttt{kotaro.sakamoto@weblab.t.u-tokyo.ac.jp}
}

\begin{document}

\maketitle

\begin{abstract}
Continuous-state generative samplers, including diffusion and flow-matching models, evolve through continuous reverse-time dynamics, yet their samples often undergo abrupt qualitative changes: trajectories commit to modes, semantic alternatives collapse, and small perturbations in narrow time windows can produce large downstream effects.
This paper develops a geometric account of such phase-transition-like behaviour.
We view denoising as gradient descent on a free energy landscape and show that sharp transitions arise near projection caustics, where the nearest-point projection onto the data support ceases to be unique.
Motivated by this perspective, we introduce the Critical Boundary Detector (CBD), as practical diagnostics for score-direction instability.
Across toy models, standard diffusion models, and latent text-to-image diffusion models, CBD localises mode commitment, predicts intervention-sensitive windows, and supports targeted control in geometrically sensitive regions.
Our results connect geometry of data and dynamics of diffusion generation.
\end{abstract}

\section{Introduction}

\afterpage{
\begin{figure*}[t!]
    \centering
    \begin{minipage}[c]{0.70\textwidth}
        \centering

        \begin{subfigure}[b]{0.75\textwidth}
            \centering
            \includegraphics[width=\textwidth]{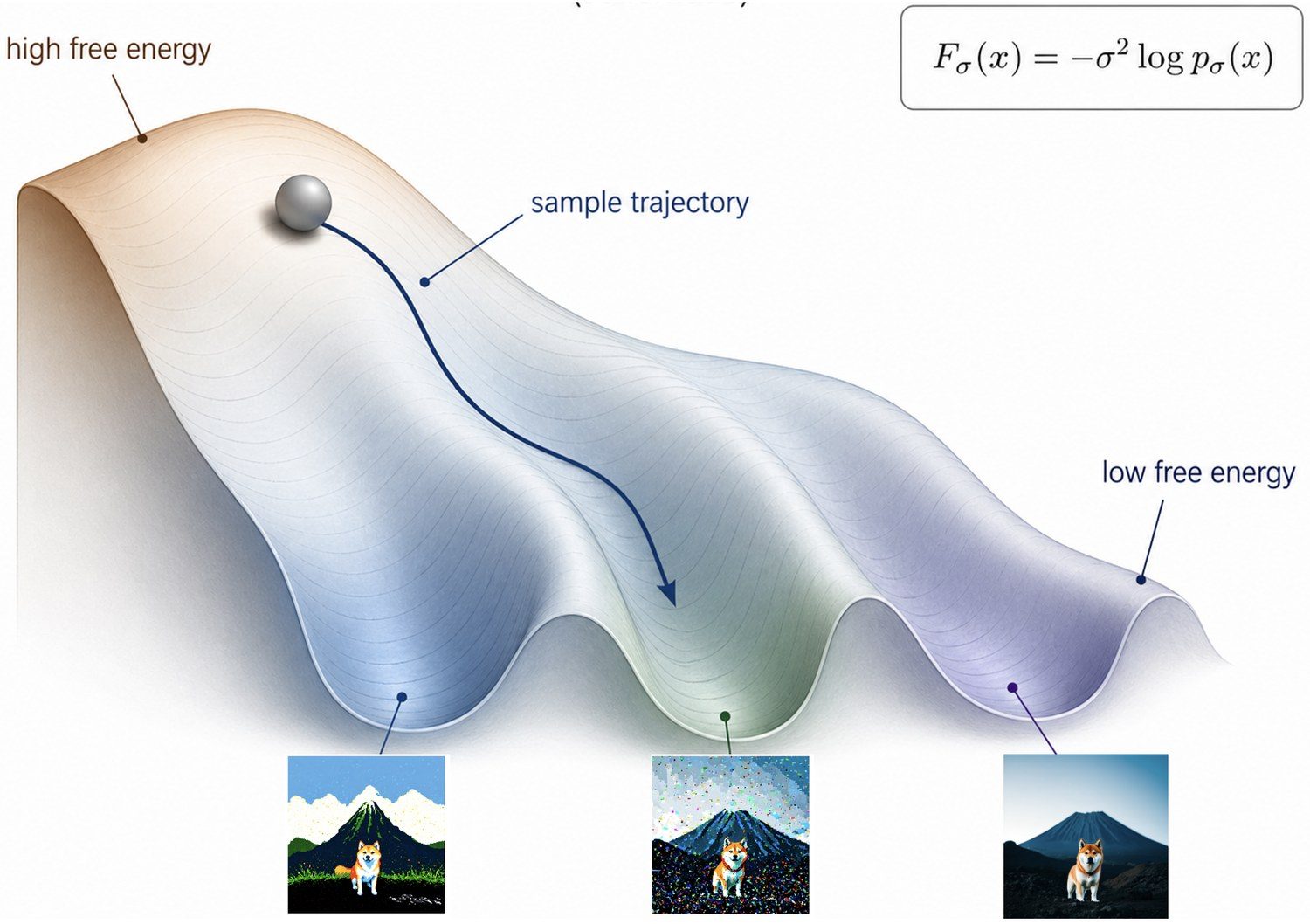}
            \caption{Denoising is interpreted as descent on a free energy landscape.
Near projection caustics, multiple branches compete and create free energy ridges.}
        \end{subfigure}

        \vspace{1em}

        \begin{subfigure}[b]{0.75\textwidth}
            \centering
            \includegraphics[width=\textwidth]{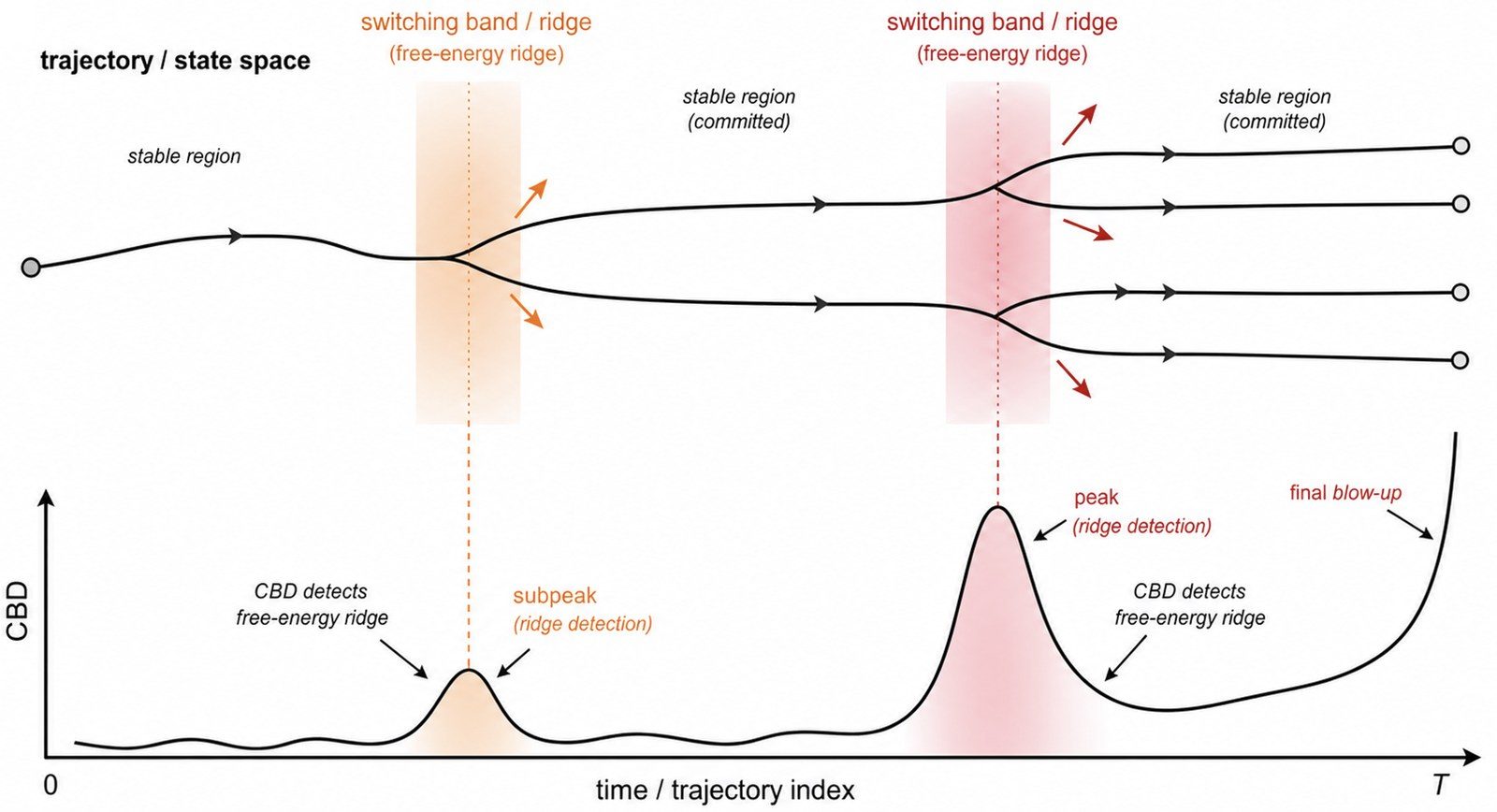}
            \caption{Trajectory segments located near these ridges exhibit rapid variation of the normalised score
direction, which is detected as CBD peaks or subpeaks along the trajectory.}
        \end{subfigure}

        \caption{}\label{fig:scheme}
        \label{fig:figure1}
    \end{minipage}%
    \hfill
    \begin{minipage}[c]{0.30\textwidth}
        \centering
        \includegraphics[width=\textwidth]{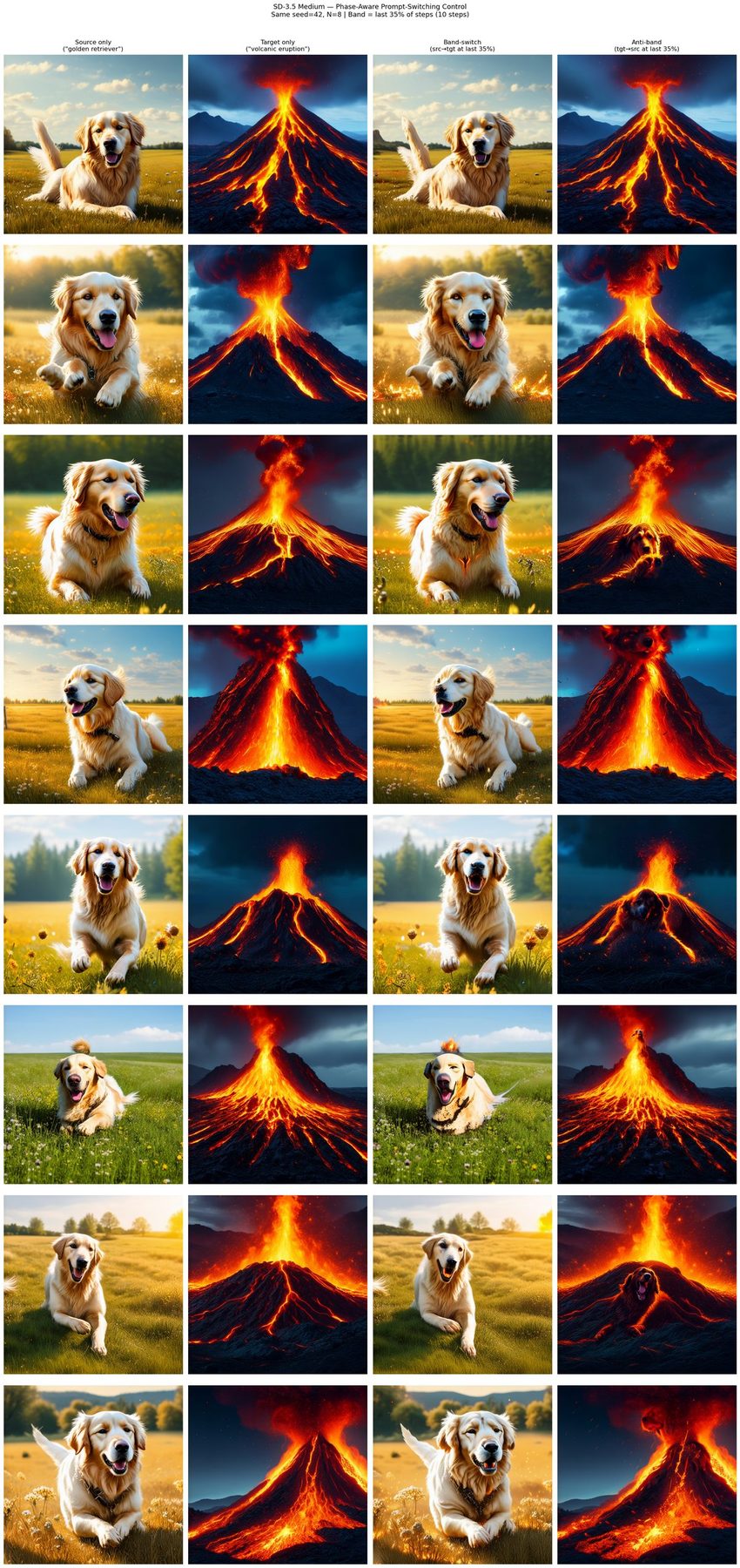}
        \caption{SD~3.5 CBD-triggered-window-intervention results. Dog- and volcano-like features are injected while largely preserving the composition and concept of each image.}
        \label{fig:figure2}
    \end{minipage}

\end{figure*}
}

Diffusion and score-based generative models build on the estimation of score fields, rooted in score matching, and sample by reversing a noise-to-data stochastic process or its associated probability-flow dynamics~\citep{hyvarinen2005score,vincent2011connection,song2019generative,anderson1982reverse,pmlr-v37-sohl-dickstein15,Ho2020denosing,song2021scorebased}.
Related continuous-state samplers, including flow matching, rectified flows, and stochastic interpolants, replace or complement stochastic reverse diffusion with learned transport or interpolation vector fields~\citep{lipman2023flow,liu2023rectified,albergo2025stochastic}.
Although these samplers evolve continuously in state space, individual trajectories often display apparently discrete events: a sample commits to one mode rather than another, semantic alternatives collapse, and small perturbations applied within a narrow time window produce macroscopic downstream effects.
Since the sampler state itself evolves continuously, these events are not literal discontinuities of the dynamics.
They instead indicate sharp reorganisations of the effective geometry governing denoising.
This paper asks what geometric event makes a continuous generative trajectory behave as if a discrete choice has been made.

Recent theory supports the view that diffusion generation is not a uniform relaxation process, but proceeds through multiple dynamical regimes. 
Statistical-mechanics and information-theoretic approaches describe these regimes in terms of speciation, condensation, symmetry breaking, critical instability, spectral gaps, and information flow~\citep{biroli2023generative,biroli2024dynamical,raya2023symmetry,ventura2025spectralgaps,ambrogioni2025thermo,stancevic2026information}. 
Related work further shows that diffusion models can resolve hierarchical structure over different noise scales, with coarse semantic features and fine details becoming recoverable at different stages of the reverse process~\citep{sclocchi2025phase}. 
More recent studies identify critical windows through complementary observables, including out-of-equilibrium pattern formation, correlation-length growth, low-frequency instabilities, and the concurrence of symmetry-breaking and nonlocality transitions in modern diffusion architectures~\citep{ambrogioni2026outofequilibrium,zhang2026concurrence}. 
Together, these results provide strong evidence for phase-transition-like behaviour in generative dynamics. 
However, they are primarily formulated at the population, thermodynamic, spectral, or architectural level, and therefore do not directly provide a lightweight geometric diagnostic for detecting branch-sensitive windows along a single trajectory of a large pretrained sampler. 
We seek such a local mechanism, and an observable trace of it, in the geometry of denoising itself.

A complementary line of work connects denoising to the local geometry of the data distribution. 
Denoising objectives and regularized autoencoders are known to encode score information~\citep{alain2014regularized,vincent2011connection}, and recent analyses show that score-based models can detect low-dimensional manifolds and that low-noise scores align with normal directions to the data support~\citep{pidstrigach2022score,pmlr-v235-stanczuk24a}. 
Projection-based and optimization-based interpretations make this geometry especially explicit: under a smooth manifold hypothesis, denoising behaves like approximate projection or descent toward the data support~\citep{permenter2024interpreting,yaguchi2025geometry}. 
This explains the projection-regular regime, where a single nearest branch of the data support dominates and the denoising direction is stable. 
The remaining question is what happens at the boundary of this uniqueness.

Our central claim is that sharp generative transitions arise precisely where projection regularity fails. 
Realistic data supports are unlikely to be globally smooth manifolds with a single-valued projection map; they may contain stratified, intersecting, or otherwise singular structures. 
Classical geometric measure theory, medial-axis theory, and singularity theory identify the failure of nearest-point uniqueness with medial-axis, cut-locus, and caustic-type singularities~\citep{federer1959curvature,blum1967medial,mather1983distance,chazal2005lambda,arnold1985singularities}. 
Motivated by this geometry, we study \emph{projection caustics}: regions where several nearest-point branches of the data support coexist and compete.

The mechanism is simple at an intuitive level. 
Away from a projection caustic, denoising is controlled by a single dominant branch, so the update direction changes smoothly. 
Near a projection caustic, several branch explanations can have comparable weight at finite noise. 
At the level of the Gaussian-smoothed free-energy landscape, this geometric ambiguity appears as a multi-branch log-sum-exp competition. 
As the trajectory moves or the noise level decreases, the dominant branch responsibility can switch rapidly from one branch to another. 
This responsibility switching reorients the normalised denoising direction, giving a local geometric mechanism for mode commitment, semantic branching, and symmetry breaking. 
Importantly, the zero-noise projection caustic is only the geometric skeleton of the transition; the sampler observes a finite-noise switching band around it, whose location can be shifted by branch density, curvature, and local amplitude factors.

This viewpoint suggests an operational diagnostic. 
Rather than reconstructing the unknown data support or explicitly estimating its caustic set, one can look directly for rapid changes in the normalised score direction or in the sampler's effective update direction. 
We introduce the \emph{Critical Boundary Detector} (CBD), a trajectory-level probe for this direction instability. 
CBD and its trajectory-adapted variants are designed to identify finite-noise switching bands: regions where competing branch explanations are locally balanced and small perturbations are likely to affect downstream commitment.

This also clarifies the relation between phase transitions and controllable generation. 
Existing control methods specify \emph{how} to steer a diffusion trajectory, for example through classifier guidance, classifier-free guidance, interval-restricted guidance, editing-by-noising, or attention-based prompt editing~\citep{dhariwal2021diffusion,ho2022classifierfree,kynkaanniemi2024interval,meng2022sdedit,hertz2023prompt}. 
Our question is complementary: \emph{when} along a trajectory is a small intervention most likely to change the eventual semantic branch? 
The projection-caustic perspective predicts that such intervention-sensitive windows should coincide with rapid reorientation of the denoising direction, and CBD is designed to detect exactly this signal.

\paragraph{Contributions.}
We make three contributions. 
First, we characterize the projection-regular regime and show that, when the nearest projection is unique, the leading geometry of the denoising landscape is governed by distance to the data support, while density and curvature appear only as lower-order corrections. 
Second, we analyse the non-regular regime near projection caustics and show that multiple nearest branches induce a finite branch-competition structure, leading to rapid switching of the dominant score direction. 
Third, we introduce CBD as the corresponding score-direction instability diagnostic and test its trajectory-level implications in toy geometries, pretrained DDPMs, transformer- and EDM-style diffusion models, and latent text-to-image samplers~\citep{Ho2020denosing,peebles2023scalable,karras2022edm,karras2024edm2,rombach2022latent,esser2024scaling}. 

The paper is primarily theoretical. 
The experiments do not aim to benchmark sample quality or to establish a universal raw instability profile across architectures. 
They test the operational implication of the theory: windows with strong score-direction reorientation should preferentially coincide with branch competition, mode commitment, and intervention-sensitive phases of generation.

\section{Asymptotic analysis of free energy in generative dynamics}\label{Sec:Theoretic}
Section~\ref{Sec:Theoretic} provides the asymptotic analysis underlying CBD. From the
free-energy gradient-flow perspective, we explain why transition windows are
expected to occur near peaks of score-direction instability.

\subsection{Free energy gradient-flow perspective and diffusion generative dynamics}
We adopt the following free energy gradient-flow perspective:
\begin{center}
\emph{diffusion dynamics is driven toward states of lower free energy.}
\end{center}
In other words, diffusion dynamics evolves as if it were descending along the free energy landscape.
In the present setting, this free energy $F_\sigma(x)$ is induced by Gaussian smoothing of the data distribution,
\[
p_\sigma(x)=\int_K\frac{1}{(2\pi\sigma^2)^{D/2}}
\exp\left(-\frac{\|x-y\|^2}{2\sigma^2}\right) p_0(y)d\mu_K,
\]
\[
F_\sigma(x):=-\sigma^2 \log p_\sigma(x),
\]
where \(K\subset \mathbb{R}^D\) denotes the support of the data distribution and \(p_0\) is the underlying density on the data support \(K\). Since the score satisfies
\[
s_\sigma(x)=\nabla_x \log p_\sigma(x)
=
-\frac{1}{\sigma^2}\nabla_x F_\sigma(x),
\]
the denoising direction is given by the steepest descent direction of the free energy landscape. Thus, at the level of idealised score dynamics, diffusion sampling may be viewed as a relaxation process on \(F_\sigma\).

This viewpoint suggests that phase transitions should be understood as changes in the geometry of the free energy landscape across noise scales. At high noise, the Gaussian smoothing washes out fine geometric structure, and the resulting landscape is dominated by coarse global statistics such as the data centroid (Figure \ref{fig:cross-free-energy-part1} (a), $\sigma=2.0$). At low noise, by contrast, the geometry of \(F_\sigma\) becomes increasingly sensitive to the local structure of the support \(K\), and in particular to the nearest-point projection onto \(K\) (Figure \ref{fig:cross-free-energy-part1} (a), $\sigma=0.1$). From this perspective, the key geometric distinction is whether a point \(x\) admits a unique nearest point on \(K\), or whether several competing nearest points coexist.


Recall that for \(x\in\mathbb{R}^D\), the nearest-point projection set is $\Pi_K(x):=\arg\min_{y\in K}\|x-y\|_2.$
We define the projection-regular region
\[
\mathcal{U}(K):=\{x\in\mathbb{R}^D:\ |\Pi_K(x)|=1\},
\]
and the projection caustic
\[
\mathcal{C}(K):=\{x\in\mathbb{R}^D:\ |\Pi_K(x)|>1\}.
\]
\begin{wrapfigure}{r}{0.45\textwidth}
    \centering
    \includegraphics[width=0.45\textwidth]{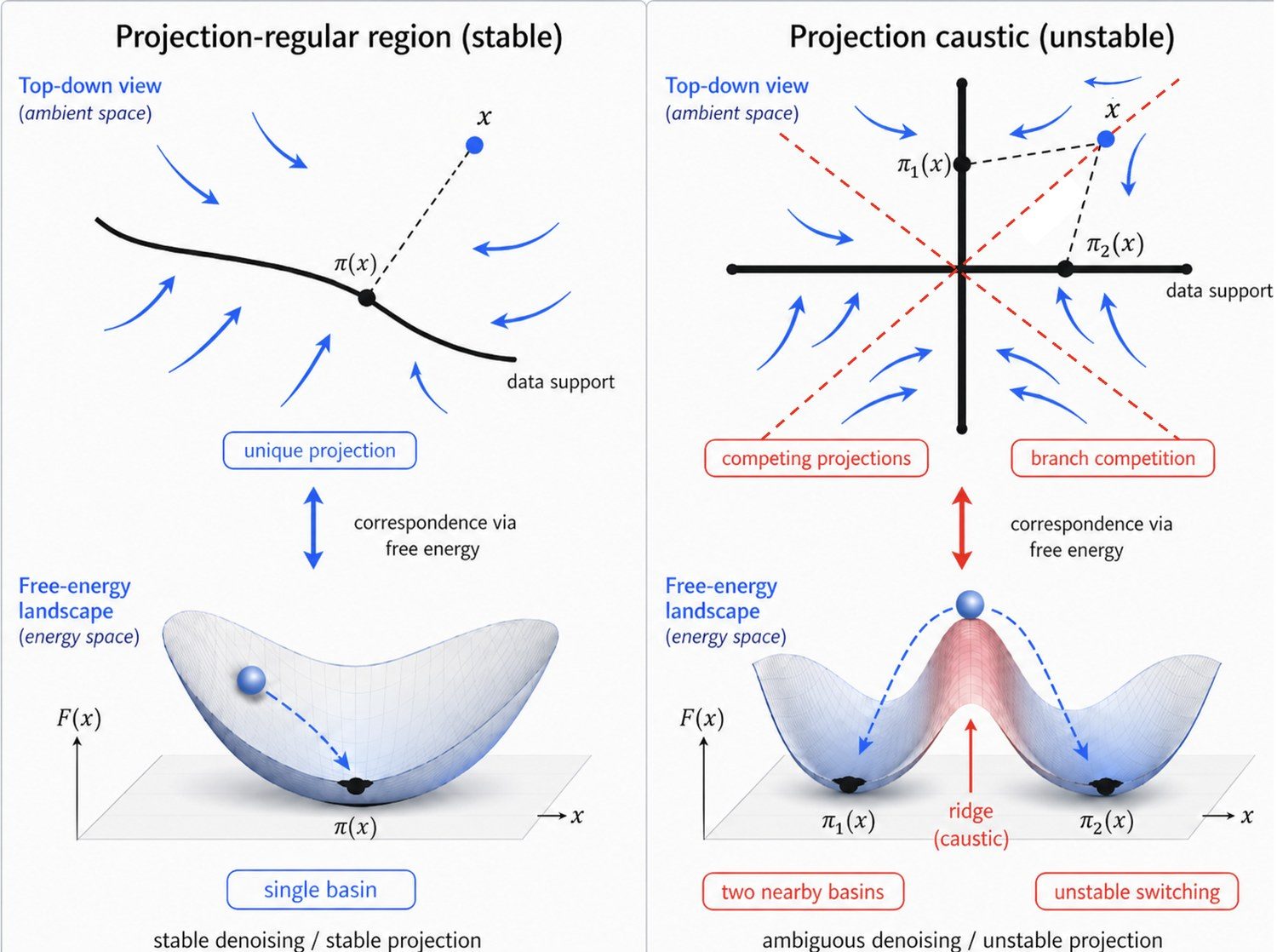}
    \caption{
    Projection-regular versus caustic regimes. 
    }
    \label{fig:projection_regular_vs_caustic}
\end{wrapfigure}

The set \(\mathcal{C}(K)\) is the locus where projection uniqueness fails. Geometrically, it plays the role of a medial-axis or cut-locus type singular set~\citep{blum1967medial,federer1959curvature}. Dynamically, it is the region where multiple nearest-point explanations compete, so that small perturbations of \(x\) may cause abrupt changes in the dominant descent direction of the free energy (See Figure~\ref{fig:projection_regular_vs_caustic}). In this sense, projection caustics provide a natural geometric mechanism for symmetry breaking and phase transitions in diffusion sampling.

When $K$ is a smooth embedded submanifold with positive reach, this framework is closely related to the tubular-neighbourhood picture of the previous work by \citet{yaguchi2025geometry}: the projection is single-valued in a neighbourhood of $K$, and the free energy landscape is locally governed by the normal geometry of the manifold.
The advantage of the present formulation is that it continues to make sense even when \(K\) is non-smooth. 
\begin{figure}[t]
    \centering
    \setlength{\abovecaptionskip}{4pt}
    \setlength{\belowcaptionskip}{0pt}

    \begin{subfigure}[t]{\linewidth}
        \centering
        \includegraphics[width=\linewidth,keepaspectratio]%
            {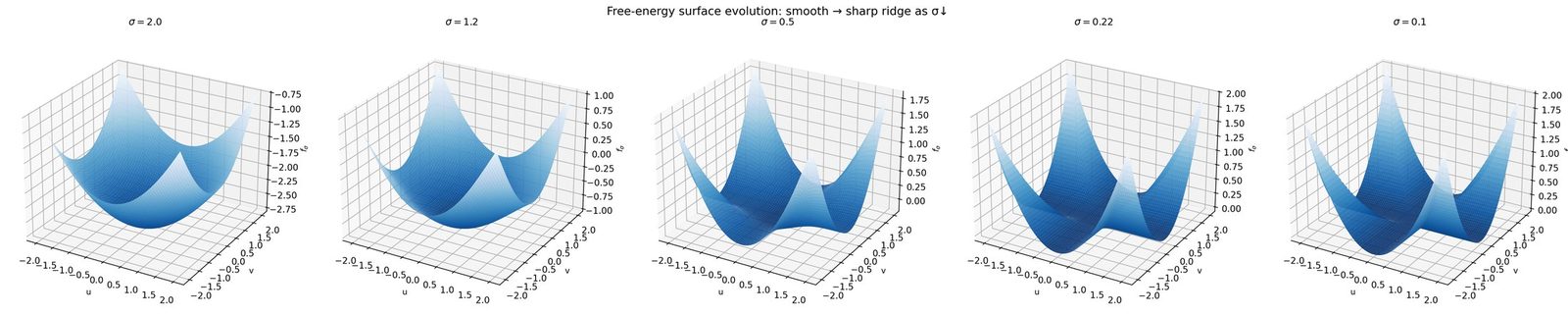}
        \caption{Free energy surfaces.}
        \label{fig:cross-fe-surfaces}
    \end{subfigure}

    \vspace{2ex}

    \begin{subfigure}[t]{\linewidth}
        \centering
        \includegraphics[width=0.95\linewidth,keepaspectratio]%
            {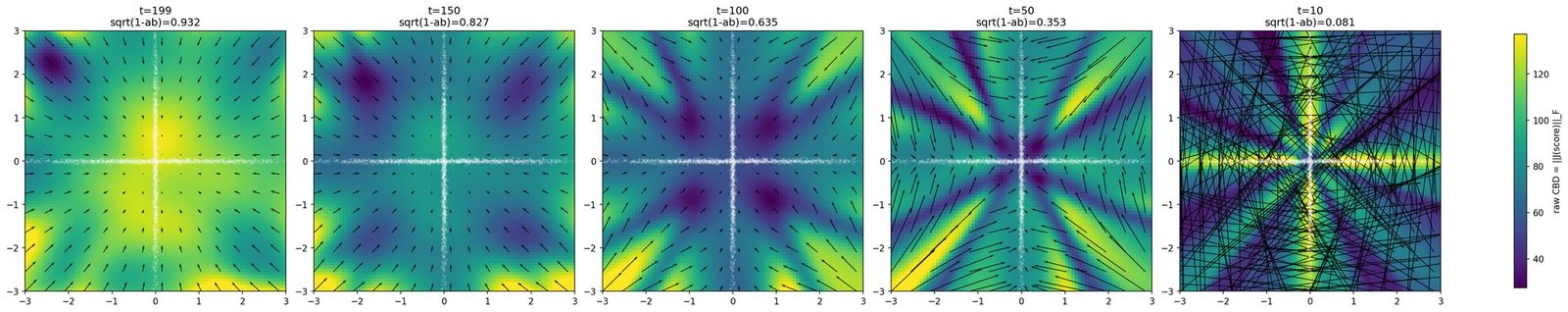}
        \caption{CBD and the score vector field over denoising time.}
        \label{fig:cross-cbd-raw}
    \end{subfigure}


    \caption{
    Transverse-cross reproduction for $K=\{(x_1,x_2)\in\mathbb R^2:x_1x_2=0\}$.
    As the noise level decreases, the free energy develops a ridge-like branch-competition geometry near the projection caustic.
    The raw CBD fields concentrate around the corresponding switching structure shown in the proxy.
    }
    \label{fig:cross-free-energy-part1}
    \vspace{-0.7em}
\end{figure}

\begin{figure}[t]
    \centering
    \setlength{\abovecaptionskip}{4pt}
    \setlength{\belowcaptionskip}{0pt}
    \begin{subfigure}[t]{0.64\linewidth}
        \centering
        \includegraphics[width=\linewidth,keepaspectratio]%
            {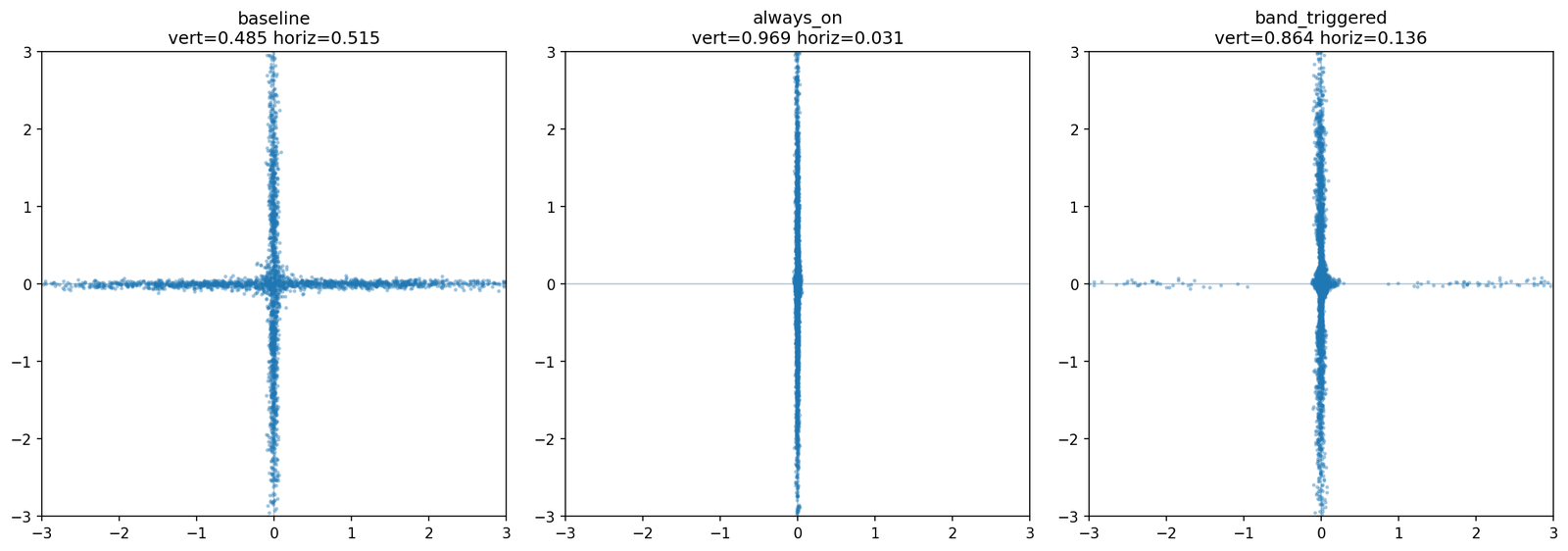}
        \caption{Branch-selection outcomes.}
        \label{fig:cross-control-comparison}
    \end{subfigure}\hfill
    \begin{subfigure}[t]{0.33\linewidth}
        \centering
        \includegraphics[width=0.8\linewidth,keepaspectratio]%
            {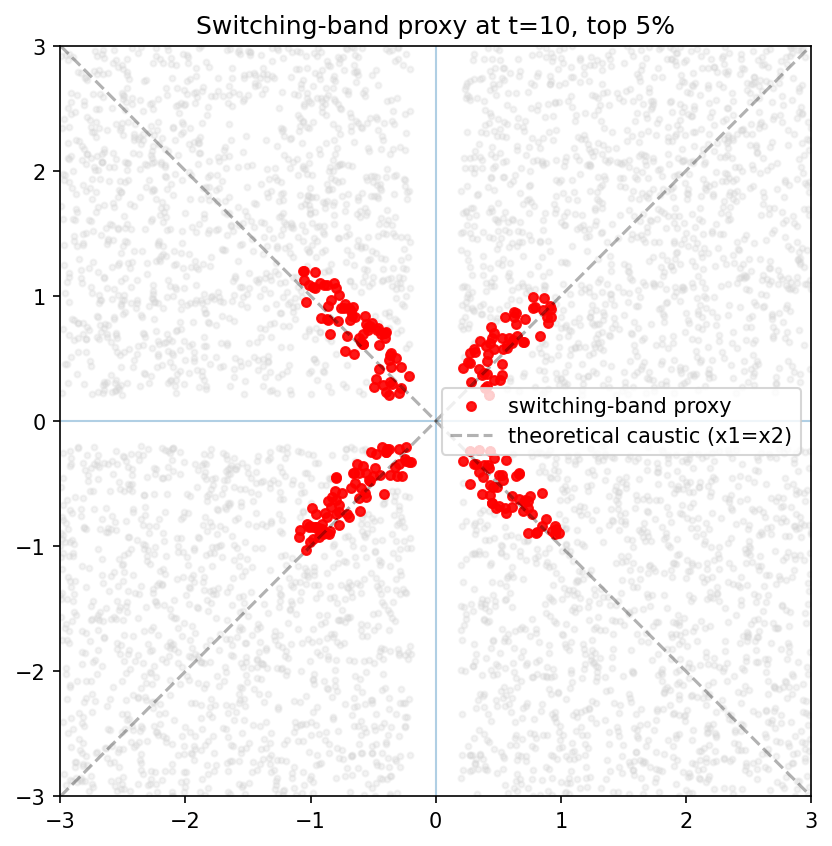}
        \caption{CBD switching-band proxy.}
        \label{fig:cross-switch-band}
    \end{subfigure}

    \vspace{2ex}

    \begin{subfigure}[t]{0.64\linewidth}
        \centering
        \includegraphics[width=\linewidth,keepaspectratio]%
            {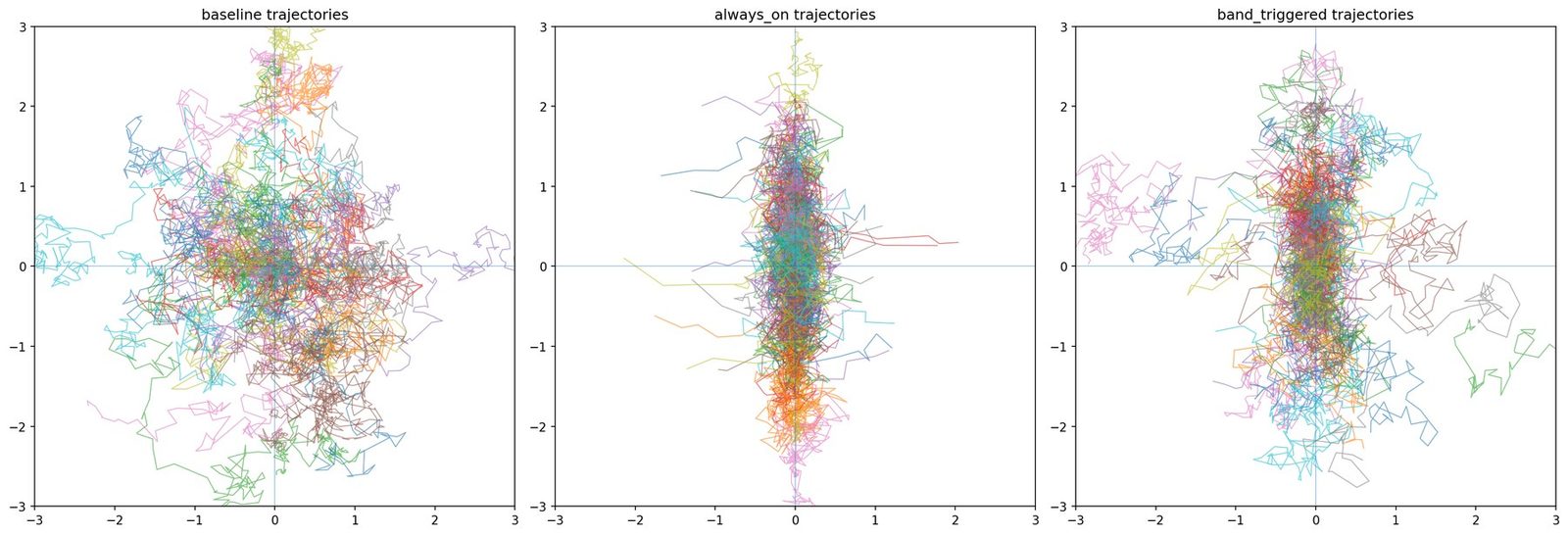}
        \caption{Reverse trajectories under baseline, always-on, and band-triggered control.}
        \label{fig:cross-trajectories}
    \end{subfigure}\hfill
    \begin{subfigure}[t]{0.34\linewidth}
        \centering
        \includegraphics[width=\linewidth,keepaspectratio]%
            {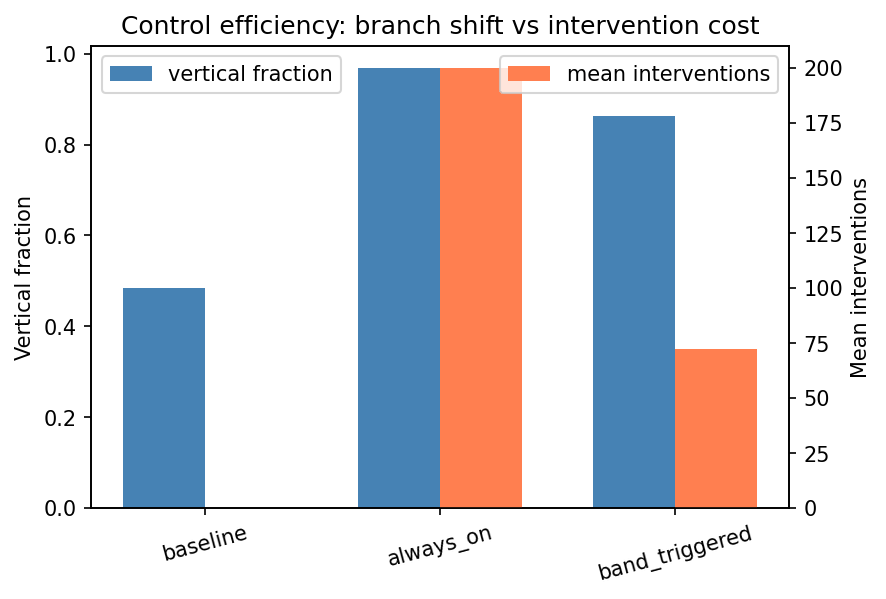}
        \caption{Control efficiency.}
        \label{fig:cross-control-efficiency}
    \end{subfigure}

    \caption{
    Transverse-cross reproduction.
    Using this high-CBD region as an intervention trigger recovers 86\% of the vertical-branch bias of always-on control while applying a vertical-branch-biased perturbation only near the CBD switching-band.
    }
    \label{fig:cross-free-energy-part2}
    \vspace{-0.7em}
\end{figure}



\subsection{The asymptotic expansion of free energy at generic point}
We first analyze the projection-regular regime, where the nearest-point map is single-valued.
 This serves as the baseline case: the free energy is controlled by a single minimizer, and no branch competition occurs.
 In particular, this regime clarifies which part of the score is stable and which effects only appear as lower-order corrections.
 In this regular regime, no sharp instability is expected, because the leading contribution comes from a unique nearest point.
 The effects of density variation and local geometry appear only in lower-order terms.

Assume that
$x$ lies in a projection-regular region, that is $x\in \mathcal{U}(K)$ where the nearest-point projection
$y_*=\Pi_K(x)$ is uniquely defined. Assume that $y^*$ contains in the following $k$-dimensional smooth submanifold $M$:
Let $M\subset\mathbb R^D$ be a $C^\infty$ embedded submanifold and let
the measure on $M$ be $d\mu_M=a\,d \vol_M$ with $a\in C^\infty(M)$, $a>0$.
Then we obtain the following asymptotic expansion of the free energy:
\begin{thm}\label{thm: regular}
The free energy admits the asymptotic expansion
\begin{equation}\label{generic_expansion}
F_\sigma(x)=
\frac12\|x-y^* \|^2
-\sigma^2\log a(y_*)
+\frac{\sigma^2}{2}\log\det(I-A_\nu)
+\frac{D-k}{2}\sigma^2\log(2\pi\sigma^2)
+O(\sigma^4),
\end{equation}
as $\sigma\to 0$.
Here, $A_\nu:T_{y_*}M\to T_{y_*}M$ denotes the shape operator in the normal direction
$\nu=x-y_*$ (see Appendix \ref{append: smth_sbmfld}).
In particular, $F_\sigma(x)=\frac12\operatorname{dist}(x,M)^2+O(\sigma^2),$
so at generic points the leading-order free energy landscape is governed solely by
the unique projection distance.
\end{thm}
\begin{proof}[Sketch of a proof]
By the Laplace method, the integral defining $p_\sigma(x)$ localises near the unique nearest point $y^*=\Pi_K(x)$. Expanding the phase around $y^*$ and evaluating the resulting Gaussian integral yields the leading term $\tfrac12\|x-y^*\|^2$, while all other contributions are either exponentially small or higher-order in $\sigma$. This gives the expansion.
\end{proof}
\begin{remark}
Differentiating the asymptotic expansion \eqref{generic_expansion}, we recover \cite[Theorem 4.1]{pmlr-v235-stanczuk24a}, which states that the score vector is asymptotically aligned with the normal direction to the data manifold. Theorem \ref{thm: regular} shows that, in the projection-regular regime, the free energy is controlled by a single branch and therefore has no intrinsic switching mechanism which will be discussed in Subsection \ref{subsec:near_caustic}.
The leading term of \eqref{generic_expansion} depends only on the squared distance to the support, while density and curvature enter only at order $\sigma^2$.
\end{remark}

\subsection{Asymptotic behaviour near the projection caustic: responsibility switching, and direction instability}\label{subsec:near_caustic}
The purpose of this subsection is to isolate the mechanism that is absent in the projection-regular regime: competition among several nearest-point branches.
The resulting log-sum-exp structure will explain both branch switching and the sharp change of score direction detected by CBD.

We study the asymptotic behaviour of the free energy near the projection caustic under some conditions. 
Our aim is to explain why score transitions become sharp, how density asymmetry shifts the switching location, and how competing branches generate direction instability. 
This also provides a theoretical interpretation of CBD as a geometric indicator of proximity to the switching set. 

Before stating the main theorem, we explain the settings (see \ref{app:settings} for details).
\paragraph{Settings: Local nondegenerate multi-branch setting near a projection caustic.}
Let $x_0\in\mathcal C(K)$ be a point of the projection caustic with $m$ nearest points $\Pi_K(x_0)=\{y_1,\ldots,y_m\}.$
We assume that, in a neighbourhood $U$ of $x_0$, these nearest-point contributions are represented by
$m$ smooth competing $k$-dimensional branches $M_j\subset K$.  More precisely, each branch admits a smooth local
parametrization
\[
\phi_j(\cdot;x):U_j\subset\mathbb R^k\to M_j,\qquad \phi_j(0;x)=y_j(x),
\]
with smooth positive amplitude $a_j(u,x)$, including the density and Jacobian factors.  We define
\[
\Phi_j(u;x)=\frac12\|x-\phi_j(u;x)\|^2,\qquad S_j(x)=\|x-y_j(x)\|^2 .
\]
We assume that $u=0$ is a uniformly nondegenerate minimum of
$\Phi_j(\cdot;x)$ for every branch, i.e. $H_j(x)=\nabla_u^2\Phi_j(0;x)$ is uniformly positive
definite.  We also assume a uniform separation gap: all points of $K$ outside the chosen branch
neighbourhoods have squared distance at least $\min_j S_j(x)+\delta$ for some $\delta>0$.
These assumptions isolate the ordinary multi-branch projection-caustic regime and exclude additional
focal degeneracies.
\begin{thm}[Local multi-branch log-sum-exp normal form near a projection caustic]\label{thm:log-sum}
Let $x$ be a point of $\mathcal{C}(K)$ that satisfies the above settings.
The free energy admits the asymptotic expansion
\[
F_\sigma(x)
=
-\sigma^2\log \left((2\pi \sigma^2)^{-(D-k)/2}
\sum_{j=1}^m B_j(x)\exp\!\left(-\frac{S_j(x)}{2\sigma^2}\right)
\right)
+ O(\sigma^4),
\]
as $\sigma\to 0$.  Where for each $j$,
\[B_j(x)=\frac{a_j(0, x)}{\sqrt{\det H_j(x)}}.\]
Here $a_j$ denotes the local amplitude factor coming from the density and the Jacobian of the branch parametrization.
In particular, near the projection caustic, the free energy is governed at leading order by a log-sum-exp competition among the local nearest-point branches.
\end{thm}

\begin{proof}[Sketch of a proof]
Split $p_\sigma(x)$ into the $m$ local branch integrals and the off-branch remainder. The latter is exponentially negligible by separation condition (iii), whereas each local term admits a uniform Laplace expansion; summing them and applying $-\sigma^2\log$ yields the claimed log-sum-exp formula for $F_\sigma(x)$.
A complete proof is given in Appendix \ref{app:caustic_free_energy}.
\end{proof}
By differentiating the above result we obtain:
\begin{cor}\label{cor:conv_sum}
Under the assumptions of Theorem~\ref{thm:log-sum}, define
\[
w_j(x,\sigma)
:=
\frac{B_j(x)\exp\!\left(-\frac{S_j(x)}{2\sigma^2}\right)}
{\sum_{\ell=1}^m B_\ell(x)\exp\!\left(-\frac{S_\ell(x)}{2\sigma^2}\right)}, \quad V_j(x,\sigma)
:=
-\frac{1}{2\sigma^2}\nabla_x S_j(x)+\nabla_x\log B_j(x).
\]
Then
\[
\sum_{j=1}^m w_j(x,\sigma)=1,
\qquad
0<w_j(x,\sigma)<1,
\]
and the score admits the expansion
\[
\nabla_x \log p_\sigma(x)
=
\sum_{j=1}^m
w_j(x,\sigma)
V_j(x,\sigma)
+O(1),
\]
uniformly on $U$ as $\sigma \to 0$. In other words, the score is a softmax-weighted convex combination of the branchwise directions, with the dominant branch switching sharply near the projection caustic.
\end{cor}
\begin{ex}[Two competing branches, \(m=2\)]
In the two-branch case, Corollary~1 gives
\begin{align*}
\nabla_x \log p_\sigma(x)
=
w_1(x,\sigma)
V_1
+
w_2(x,\sigma)
V_2
+O(1),
\end{align*}
where
$
w_1(x,\sigma)
=
\frac{1}{1+\exp\!\left(\frac{S_1(x)-S_2(x)}{2\sigma^2}-\log\!\frac{B_1(x)}{B_2(x)}\right)},
\qquad
w_2(x,\sigma)=1-w_1(x,\sigma).
$
If \(S_1(x)<S_2(x)\), then \(w_1(x,\sigma)\to 1\) and \(w_2(x,\sigma)\to 0\) as \(\sigma\to 0\);
if \(S_2(x)<S_1(x)\), the opposite holds. Thus the dominant branch switches within an
\(O(\sigma^2)\)-thin layer, which explains the sharpness of the score transition. 

Hence, whenever \(V_1\) and \(V_2\) are not collinear, a rapid variation of \(w_1\) produces a rapid
reorientation of the score direction.
The switching set is shifted from the geometric balance locus (the projection caustic) $S_1=S_2$
to
\[
S_1-S_2=2\sigma^2\log\!\left(\frac{B_1}{B_2}\right),
\]
showing that the amplitude asymmetry encoded in \(B_1,B_2\) affects the transition only at lower order.

\end{ex}
\begin{remark}
The statements of Theorem \ref{thm:log-sum} and Corollary \ref{cor:conv_sum} show that the ridge-like structure
along projection caustics observed in Figure~\ref{fig:cross-free-energy-part1} is not a low-dimensional artifact,
but persists in high dimensions. Indeed, the multi-branch log-sum-exp structure
implies that the free energy landscape develops a sharp competition ridge along
the switching set, where multiple nearest-point branches coexist.

In the low-noise regime, these weights concentrate on the branches minimizing the local action $S_j(x)$. Hence, away from the switching set where the minimizer is unique, the score is asymptotically governed by a single nearest-point branch. By contrast, near the projection caustic, two or more actions become comparable, and the soft weights $w_j$ undergo a rapid transition on the scale $S_i(x)-S_j(x)=O(\sigma^2)$. As a consequence, the dominant branch of the score switches sharply, which leads to strong instability of the normalised score direction. 

For a fixed noise level \(\sigma>0\), the actual branch-switching band is the finite-noise switching set $\mathcal S_\sigma
:=\{x:\ w_i(x,\sigma)=w_j(x,\sigma)
\text{ for some competing branches } i\neq j\}.$
Thus \(\mathcal S_\sigma\) generally differs from \(C(K)\) by an \(O(\sigma^2)\) displacement caused by branch amplitudes, density, and local Jacobian factors. 

This switching mechanism provides the geometric basis for CBD. Indeed, it is designed to detect rapid variation of the normalised score direction, and such variation is amplified near the projection caustic.
Therefore, CBD peaks especially in low-noise regime should be interpreted as signatures of proximity to the switching set.
\end{remark}

\section{Numerical analysis: from detection to control}\label{sec:CBD}

\subsection{CBD (Critical Boundary Detector) and its trajectory-adapted directional estimator (TAD)}
The theory predicts a directly testable signature: near a projection caustic, the dominant branch in the log-sum-exp free energy changes rapidly, and therefore the \emph{normalised} denoising direction should be locally unstable.
Let $v_t(x)$ denote the score or the scheduler's effective update field and set $u_t(x)=\frac{v_t(x)}{\|v_t(x)\|}.$
We measure this instability by
\begin{equation}
    \mathrm{CBD}(x,t)=\|\nabla_x u_t(x)\|_F .
    \label{eq:cbd_def}
\end{equation}
For large pretrained models we use random-direction finite differences,
\begin{equation}
  \widehat{\mathrm{CBD}}_{\rm FD}(x,t)^2
  =
  \frac1N\sum_{k=1}^N
  \left\|
  \frac{u_t(x+\delta_t z_k)-u_t(x)}{\delta_t}
  \right\|^2,
  \qquad
  z_k\sim\mathcal N(0,I),\quad
  \delta_t=c\sigma_t,
  \label{eq:cbd_fd}
\end{equation}
with $c\in[10^{-3},10^{-1}]$.
When a full Jacobian estimate is too expensive, we use the trajectory-adapted directional (TAD) proxy
\begin{equation}
    \mathrm{CBD}_{\rm TAD}(x_i,t)
    =
    \frac{\|u_t(x_i+h z)-u_t(x_i-h z)\|}{2h},
    \qquad
    z\sim\mathcal N(0,I),
    \label{eq:cbd_tad}
\end{equation}
where $x_i$ is a saved state on a baseline reverse trajectory and $h$ is proportional to the noise scale.
We report the normalised trajectory profile as $\mathrm{CBD}_{\mathrm{TAD}}$; high values mark windows in which the update direction is most sensitive to small state perturbations

\subsection{Pseudo-Online CBD}
As seen in Appendix~\ref{app:cusp}, the key difficulty is that the eventual branch outcome is not determined exactly when the trajectory crosses the projection-caustic-induced switching band, but only shortly after it has passed through that band. Moreover, the CBD activation associated with proximity to the projection caustic is most pronounced in the low-noise regime. 
A state that appears locally stable at its current time may nevertheless become close to a sharp
low-noise switching barrier when probed at a later, lower-noise scale. 
To capture this effect without modifying the sampler online, we introduce a pseudo-online CBD
ratio (Algorithm \ref{alg:pseudo_online_cbd}). We first fix a baseline reverse trajectory and evaluate the trajectory-adapted CBD at each
state using the actual sampler time. We then re-evaluate the same state using a low-noise probe
time, which asks whether a future low-noise free-energy wall is already latent at the current
location. The resulting ratio, or its inverted normalised form, gives a trajectory-level ranking of
states that are upstream of future score-direction instability.
\begin{algorithm}[t]
\caption{Pseudo-online $\mathrm{CBD}_{\mathrm{TAD}}$ signal along a fixed trajectory}
\label{alg:pseudo_online_cbd}
\begin{enumerate}
\item \textbf{Input:} Baseline reverse trajectory $\{(x_i,t_i)\}_{i=0}^{N}$,
low-noise probe time $t_{\rm low}$, perturbation scale $h$, probe direction $z$,
smoothing operator $\mathcal{S}$, and $\varepsilon>0$.

\item For each $i=0,\ldots,N$, compute the actual-time directional instability
\[
A_i =
\mathrm{CBD}_{\mathrm{TAD}}(x_i,t_i).
\]

\item Re-evaluate the same state at the low-noise probe time:
\[
B_i =
\mathrm{CBD}_{\mathrm{TAD}}(x_i,t_{\rm low}).
\]

\item Form the actual-to-low-noise CBD ratio
\[
R_i = \frac{A_i}{B_i+\varepsilon}.
\]
A small value of $R_i$ indicates that $x_i$ is relatively stable at its actual time
but becomes unstable when probed at the low-noise time.

\item Smooth the ratio along the trajectory:
\[
\bar R_i = \mathcal{S}\bigl(\{R_j\}_{j=0}^{N}\bigr)_i .
\]

\item Convert it into an inverted normalised signal:
\[
\widetilde R_i
=
\frac{\bar R_i-\min_{0\leq j\leq N}\bar R_j}
{\max_{0\leq j\leq N}\bar R_j-\min_{0\leq j\leq N}\bar R_j+\varepsilon},
\qquad
S_i = 1-\widetilde R_i .
\]

\item \textbf{Output:} Pseudo-online CBD signal $\{S_i\}_{i=0}^{N}$.
Large values of $S_i$ mark states that are upstream of future low-noise score-direction instability.
\end{enumerate}
\end{algorithm}

\subsection{CBD predicts intervention sensitivity in a pretrained DDPM}
\label{subsec:Cifar_psd}

We first evaluate a pretrained CIFAR-10 DDPM with a pseudo-online protocol.
Along each baseline trajectory we compute the $\mathrm{CBD}_{\mathrm{TAD}}$ profile, perturb the trajectory inside each candidate window by a fixed re-noising intervention, and measure the deviation of the final image from the unperturbed baseline using $\ell_2$ and LPIPS.
If CBD is detecting switching-band geometry, its high-instability windows should be exactly the windows whose perturbations survive to the final sample.

In Figure~\ref{fig:cifar_cbd_main},
the $1-\mathrm{CBD}_{\mathrm{TAD}}$ profile is stable across probe displacements $\Delta\in\{1,3,5\}$, showing that the signal is not a probe-scale artefact.
More importantly, the per-trajectory overlay of normalised CBD and LPIPS shows strong trajectory-level agreement:
across five seeds, Pearson correlations are
$\{0.953,0.908,0.966,0.849,0.964\}$, with mean $\rho=0.928$ and all $p$-value satisfy \(p<10^{-70}\).
Window-aggregated rank statistics are also stable, with mean Spearman $0.909$ for $\ell_2$ and $0.887$ for LPIPS; see Appendix~\ref{app:cifar_seed_robustness}.
Thus CBD does not merely locate a population-level average transition; it ranks sensitive windows on individual reverse trajectories without access to the data geometry.

\begin{figure}[t]
    \centering

    \begin{subfigure}[t]{0.70\linewidth}
        \centering
        \includegraphics[width=\linewidth]%
            {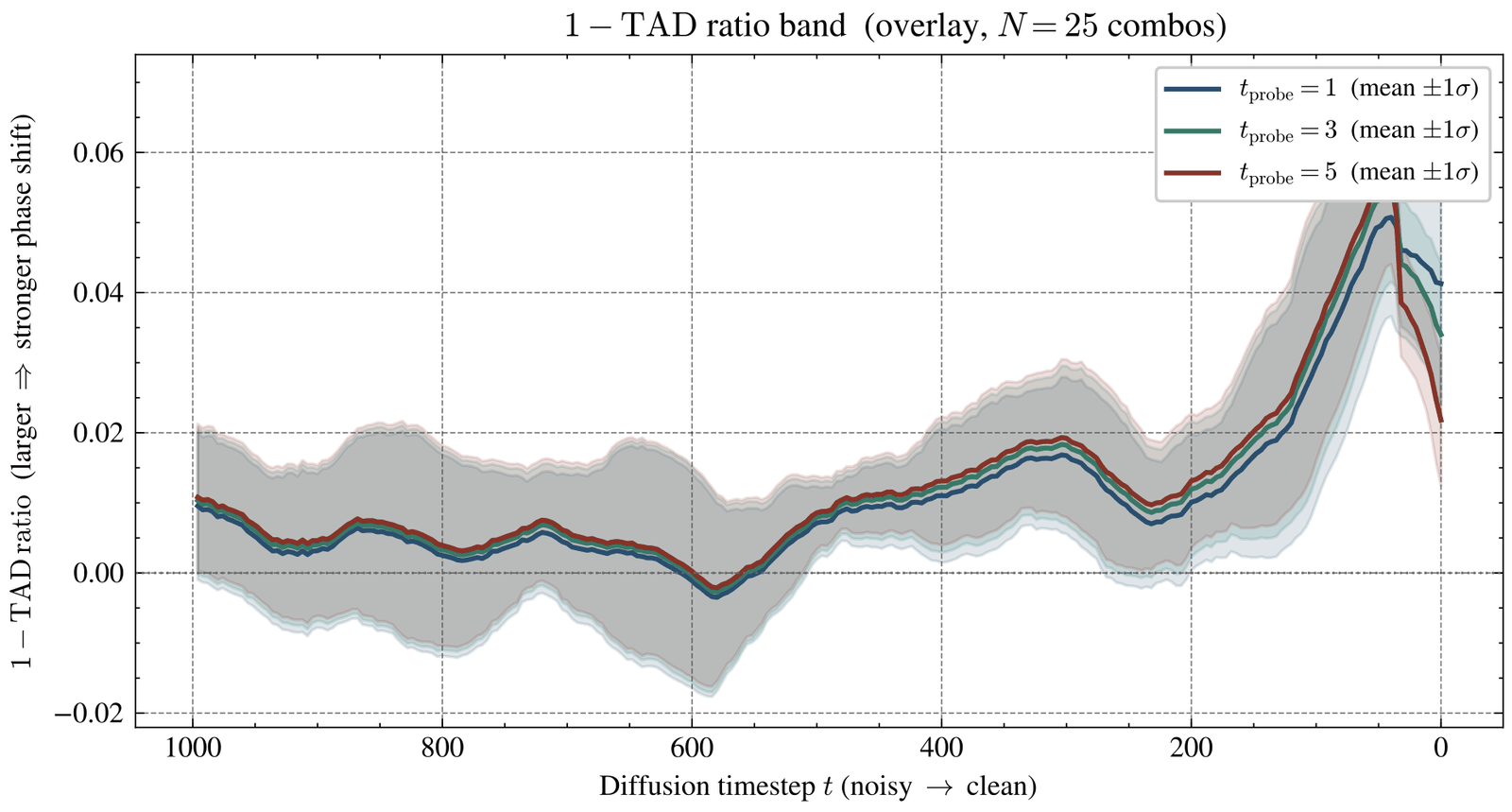}
        \caption{Probe-scale robustness of the $1-\mathrm{CBD}_{\mathrm{TAD}}$ profile.}
        \label{fig:cifar10_band_3probes}
    \end{subfigure}

    \vspace{3pt}

    \begin{subfigure}[t]{\linewidth}
        \centering
        \includegraphics[width=\linewidth]%
            {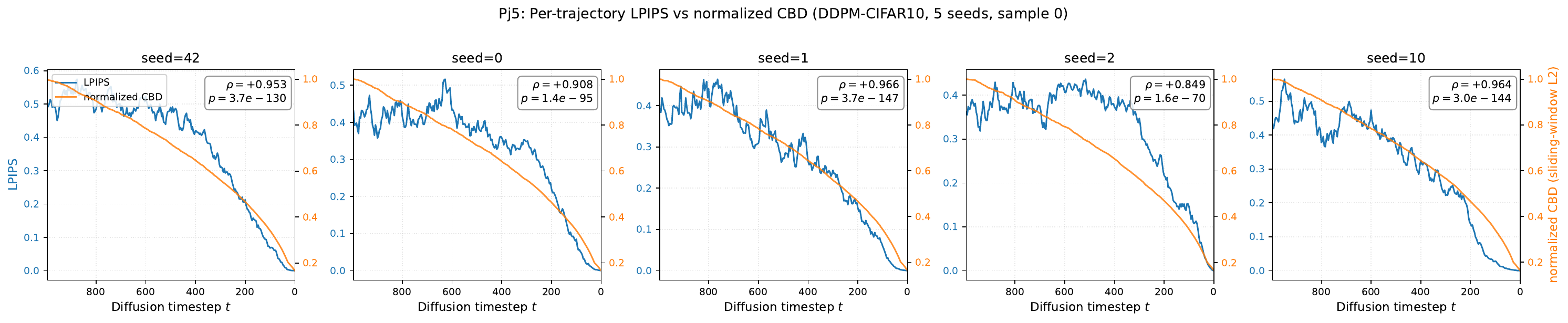}
        \caption{Per-trajectory CBD--LPIPS coupling.}
        \label{fig:cifar10_per_traj}
    \end{subfigure}

    \caption{
    CIFAR-10 DDPM: CBD predicts intervention-sensitive windows.
    (a) The $1-\mathrm{TAD}$ profile is stable across probe displacements
    $\Delta\in\{1,3,5\}$, showing that the detected instability band is not a probe-scale artefact.
    (b) Along five individual reverse trajectories, the normalised CBD proxy tracks downstream LPIPS sensitivity with mean Pearson correlation $\rho=0.928$.
    This is the main evidence that CBD ranks sensitive windows at the trajectory level, not only after population averaging.
    }
    \label{fig:cifar_cbd_main}
    \vspace{-0.8em}
\end{figure}
\subsection{$\mathrm{CBD_{TAD}}$-guided classifier guidance}

To further test whether the detected instability band is actionable, we applied
classifier guidance only inside the $\mathrm{CBD_{TAD}}$-high band on a pretrained
CIFAR-10 DDPM.  We used \texttt{google/ddpm-cifar10-32} with 250 inference
steps and target class 0 (airplane).  The $\mathrm{CBD_{TAD}}$ band was fixed to
$t\in[30,70]$, while a random comparison band of comparable length was placed
at $t\in[180,220]$.  Each condition was evaluated over 25 seed--sample
combinations and compared against both unguided sampling and full-step
classifier guidance.

Across guidance strengths $\gamma\in\{1,2,3,5\}$, full classifier guidance
achieved target accuracy $1.00$ using all 250 intervention steps.  In contrast,
$\mathrm{CBD_{TAD}}$-guided classifier guidance used only 10 intervention steps on average
and still achieved target accuracy $0.96$--$1.00$.  Thus, it recovers essentially
the same class-control effect with only $4\%$ of the intervention steps (Table \ref{tab:cbd_tad_classifier_guidance}).  The
matched random-band control, despite using 11 intervention steps, remained at
baseline accuracy $0.04$ for all $\gamma$.  This shows that the effect is not
explained by the number of guided steps alone, but by their placement in the
$\mathrm{CBD_{TAD}}$ instability band.

The LPIPS values of $\mathrm{CBD_{TAD}}$ and random-band interventions are comparable,
whereas their target accuracies differ dramatically.  This separation indicates
that perceptual displacement alone is not sufficient for semantic control:
intervening near the detected instability band changes the branch commitment,
whereas an equal-cost off-band intervention mainly induces non-semantic
perturbations.  In wall-clock time, $\mathrm{CBD_{TAD}}$ guidance remains close to baseline
sampling, while full classifier guidance is roughly two times slower.  These
results support the view that $\mathrm{CBD_{TAD}}$ identifies a narrow, causally effective
window for semantic control rather than merely a high-variance segment of the
trajectory.
\begin{table}[h]
\centering
\small
\caption{$\mathrm{CBD_{TAD}}$-guided classifier guidance on CIFAR-10 DDPM.
Classifier guidance is applied either on all reverse steps (full\_cg),
only inside the $\mathrm{CBD_{TAD}}$ instability band ($t\in[30,70]$),
or inside a random comparison band ($t\in[180,220]$).
$\mathrm{CBD_{TAD}}$ guidance achieves near-full classifier control using only
$\sim4\%$ of the intervention steps.}
\label{tab:cbd_tad_classifier_guidance}
\begin{tabular}{lcccc}
\toprule
Condition & NIS & Target Acc. & LPIPS & Sec/sample \\
\midrule
baseline         &   0 & 0.040 & 0.000 & 3.94 \\
full\_cg         & 250 & 1.000 & 0.637 & 8.80 \\
$\mathrm{CBD_{TAD}}$ guidance &  10 & 0.960 & 0.600 & 4.02 \\
random-band guidance & 11 & 0.040 & 0.592 & 4.05 \\
\bottomrule
\end{tabular}
\end{table}
\subsection{An instability signal across pretrained models}
\label{subsec:large_models}

The next test is portability.
We apply the same $\mathrm{CBD}_{\mathrm{TAD}}$ estimator, without architecture-specific tuning, to DiT-XL on ImageNet, EDM2-XS on CIFAR-10, and Stable Diffusion~3.5 Medium.
Figure~\ref{fig:3model_overlay} shows a common three-region structure:
a relatively stable high-noise regime, a concentrated intermediate instability band, and re-stabilisation near the clean endpoint.
The shared shape is the empirical signal of the projection-caustic mechanism; the differences are informative rather than nuisance variation.
DiT-XL exhibits the sharpest band, EDM2 a broader plateau, and SD~3.5 a compressed late trajectory consistent with latent flow matching.
Denoiser-derived free energy visualisations in Appendix~\ref{app:free_energy} show matching ridge/plateau structure, closing the loop between the trajectory-level CBD signal and the free energy geometry predicted by Theorem~\ref{thm:log-sum}.
Additional SD~3.5 prompts, the CelebaHQ DDPM, and the five-model diagnostic panels are collected in Appendices~\ref{app:sd35_lpips}, \ref{app:celebahq}, and~\ref{app:pj5_tad}.

\begin{figure}[t]
    \centering

    \begin{subfigure}[t]{0.45\linewidth}
        \centering
        \includegraphics[width=\linewidth]%
            {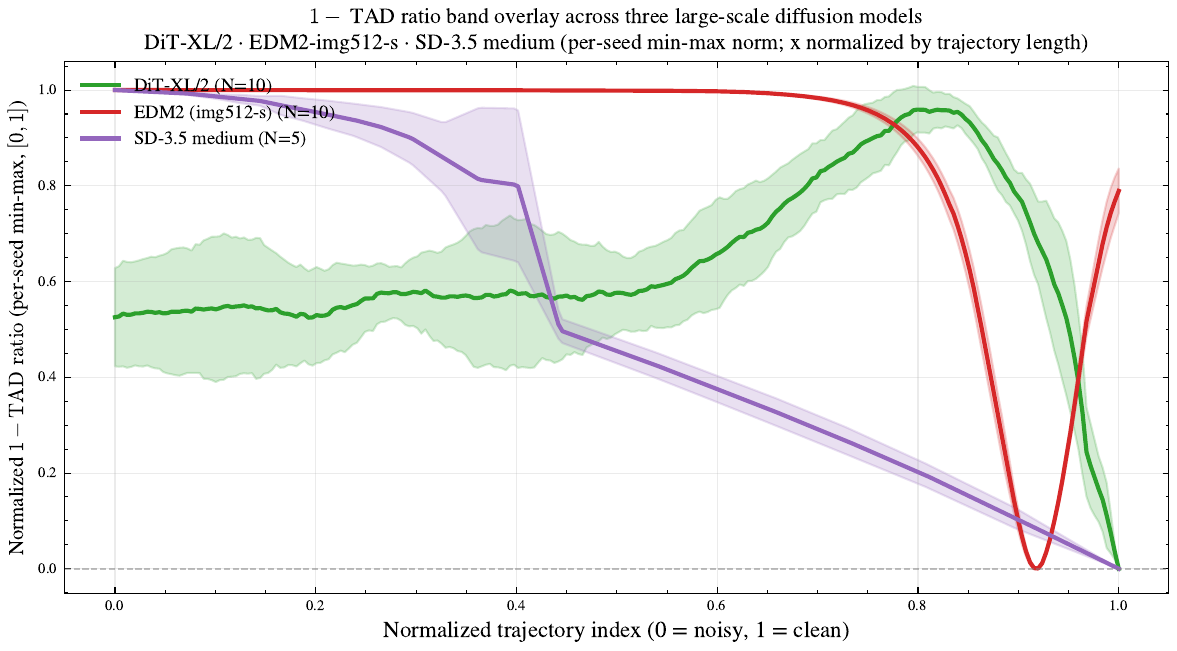}
        \caption{Shared normalised instability profile.}
        \label{fig:3model_overlay}
    \end{subfigure}\hfill
    \begin{subfigure}[t]{0.55\linewidth}
        \centering
        \includegraphics[width=\linewidth]%
            {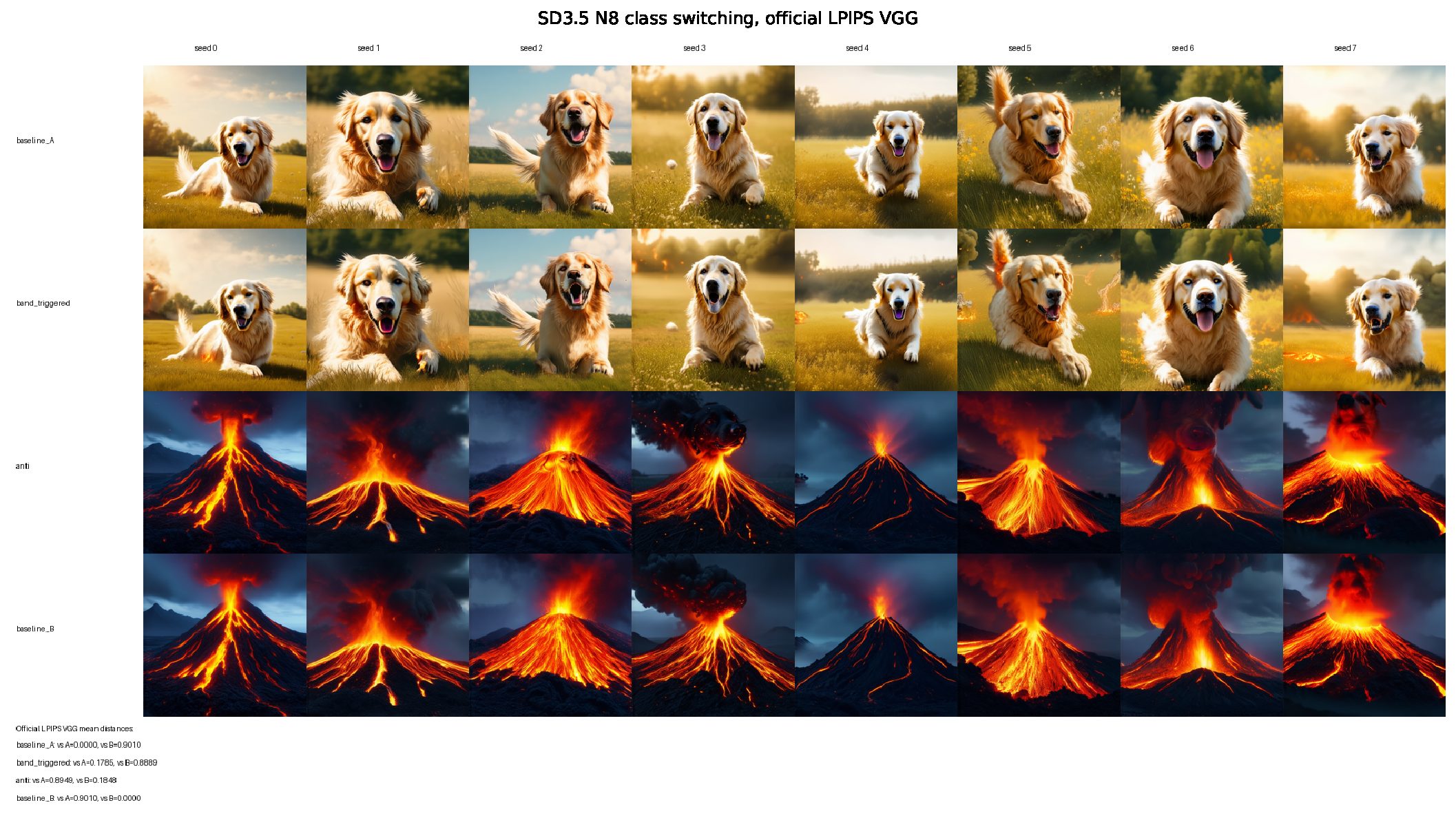}
        \caption{Concept insertion after semantic commitment.}
        \label{fig:sd35_control_compact}
    \end{subfigure}

    \caption{
    From shared diagnosis to phase-aware control.
    (a) After per-run min--max normalisation and trajectory-length normalisation, the same trajectory-adapted $1-\mathrm{CBD}_{\mathrm{TAD}}$ score reveals a shared three-regime pattern across DiT-XL, EDM2-XS, and SD~3.5 Medium: a relatively stable high-noise regime, a concentrated instability band, and re-stabilisation near the clean endpoint. We do not claim that the raw
    magnitudes or exact locations are universal; the band width and location remain model- and scheduler-dependent.
    (b) In SD~3.5 Medium, the $t^\star = \min \left\{ \arg\max_{t}\, (d/dt)[1 - \mathrm{CBD}_\mathrm{TAD}(t)]\right\}$ profile is used as a prompt-switching time, which we interpret as a semantic commitment boundary. The prompt active before this elbow determines the dominant semantic branch or class. A prompt introduced after commitment does not simply replace this class; rather, the new concept is incorporated as an element inside the already selected branch. Thus, when the trajectory is first conditioned on \(\mathcal P_A\) (``a golden retriever running in a meadow'') and then switched to \(\mathcal P_B\) (``a volcanic eruption with lava''), the output remains in the golden-retriever branch while acquiring volcanic or lava-like elements. Conversely, when the trajectory is first conditioned on \(\mathcal P_B\) and then switched to \(\mathcal P_A\), the volcano branch is preserved while dog-like elements are inserted. LPIPS distances to the source-only and target-only baselines quantify this branch preservation and late concept insertion.
}
    \label{fig:universal_signal_and_control}
    \vspace{-0.8em}
\end{figure}

\subsection{CBD predicts intervention sensitivity in SD~3.5 medium}\label{subsec:sd3.5_intervention_sensitivity}
\begin{table}[t]
\centering
\caption{Prompt-wise mean correlation statistics on SD~3.5 Medium, averaged over 5 seeds for each prompt. Here, $P$ and $S$ denote Pearson and Spearman correlation, respectively.}
\label{tab:sd35_promptwise_corr_maintext}
\small
\setlength{\tabcolsep}{4pt}
\begin{tabular}{clcccc}
\toprule
Idx & Prompt & P$(\mathrm{CBD}_{\mathrm{TAD}},\ell_2)$ & P$(T,\mathrm{LPIPS})$ & S$(\mathrm{CBD}_{\mathrm{TAD}},\ell_2)$ & S$(\mathrm{CBD}_{\mathrm{TAD}},\mathrm{LPIPS})$ \\
\midrule
0 & House $\leftrightarrow$ ship
  & $-0.446$ & $-0.247$ & $-0.864$ & $-0.671$ \\
1 & Shiba inu + volcano
  & $-0.396$ & $-0.299$ & $-0.921$ & $-0.714$ \\
2 & Lighthouse $\leftrightarrow$ seashell
  & $-0.489$ & $-0.391$ & $-0.943$ & $-0.786$ \\
3 & Mountain $\leftrightarrow$ lion
  & $-0.366$ & $\phantom{-}0.620$ & $-0.729$ & $\phantom{-}0.421$ \\
4 & Jellyfish $\leftrightarrow$ chandelier
  & $-0.623$ & $-0.587$ & $-0.993$ & $-0.993$ \\
5 & Red vintage car
  & $-0.499$ & $-0.477$ & $-0.871$ & $-0.721$ \\
\bottomrule
\end{tabular}
\end{table}
We conducted the corresponding SD~3.5 intervention sensitivity study as in Subsection \ref{subsec:Cifar_psd} over 30 prompt--seed runs: the window-averaged $\mathrm{CBD}_{\mathrm{TAD}}$ ratio has mean correlations
$\mathrm{Pearson}(\mathrm{CBD}_{\mathrm{TAD}},\ell_2)=-0.470$,
$\mathrm{Pearson}(\mathrm{CBD}_{\mathrm{TAD}},\mathrm{LPIPS})=-0.230$,
$\mathrm{Spearman}(\mathrm{CBD}_{\mathrm{TAD}},\ell_2)=-0.887$, and
$\mathrm{Spearman}(\mathrm{CBD}_{\mathrm{TAD}},\mathrm{LPIPS})=-0.577$.
Table \ref{tab:sd35_promptwise_corr_maintext} shows that we have consistent results if we exclude Mountain-Lion prompt case as an outlier.

We applied noise-injecting interventions over fixed-width time-index windows (0–1,1–2,…,25–26) and generated the corresponding output images (Figure \ref{fig:sliding_shiba_maintext}). Figure \ref{fig:sliding_shiba_cbd_26_maintext} shows the CBD curve along the baseline trajectory in the low-noise regime. It can be visually observed that the intervention effect becomes weaker after the CBD peak. See Appendix \ref{app:sd35_intervention_shiba_inu}.
\begin{figure}[t]
    \centering
    \includegraphics[width=\linewidth]{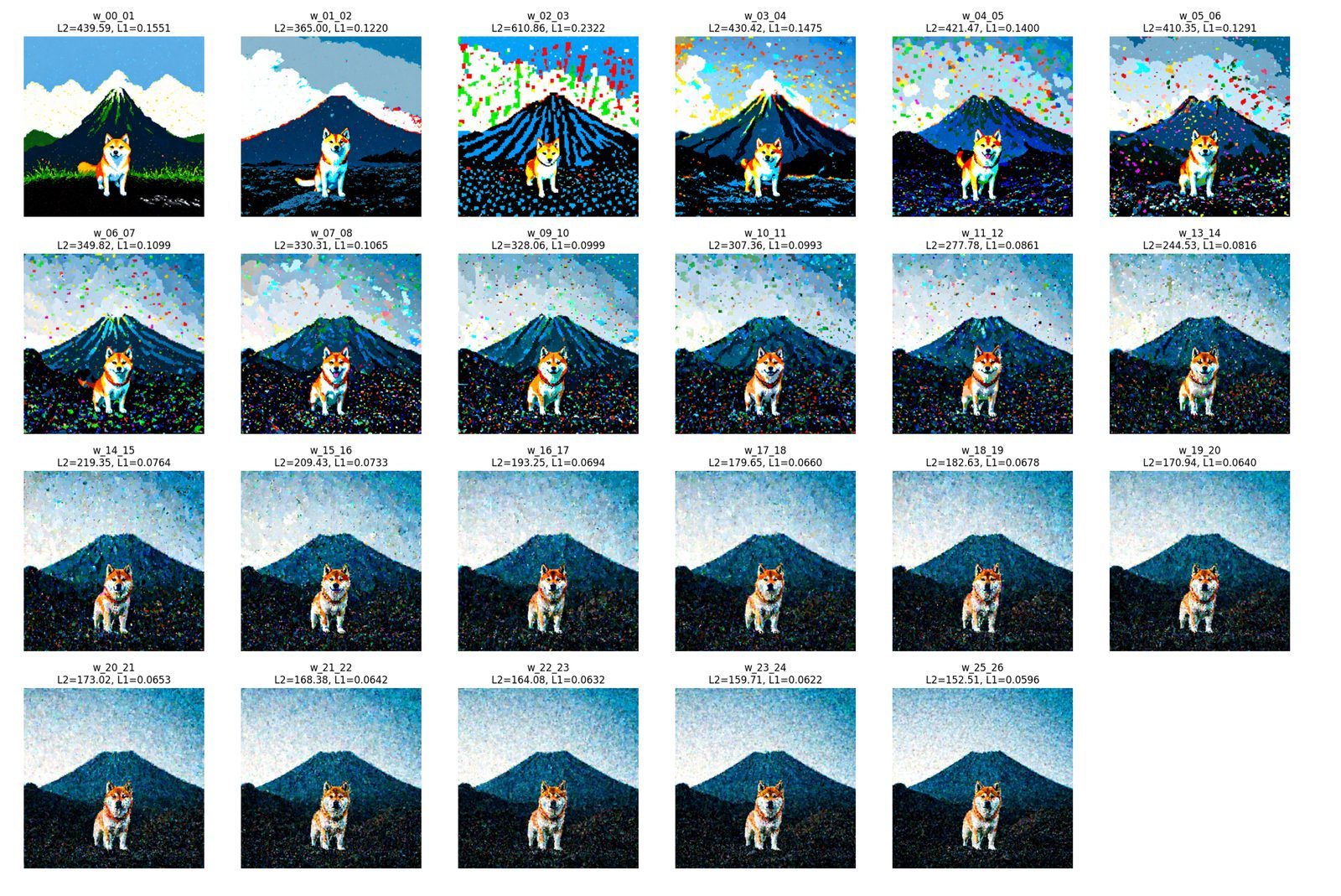}
   \caption{
   $w_{mn}$ denotes the case where intervention is performed over the time window (m,n). L2 and L1	indicate the L2 and L1 distances, respectively, between the baseline generation result and the intervention result. If we intervene within each time window in 0-11,  cartoon-like feature appears. The intervention result switches from cartoon-like features to cinematic features in the time window 12–15. This switching phenomenon can be described by the ridge window in the CBD plot (Figure \ref{fig:sliding_shiba_cbd_26}) within the time 12-15. This example suggests that the selection of whether the output becomes cinematic or not is made within the 12–15 time window. The interventions within late windows do not have effective consequences. Across all interventions, the overall composition of “Mount Fuji behind the dog” is preserved. The intervention results are highly unstable in the 0–5 time window, consistent with the ridge in the corresponding CBD plot in Figure \ref{fig:sliding_shiba_cbd_26_maintext}, which also appears during this interval. By contrast, the intervention outcomes around 7–11 appear relatively stable, which may be attributed to the corresponding CBD plot in Figure \ref{fig:sliding_shiba_cbd_26_maintext} being flat in that interval.
    }
    \label{fig:sliding_shiba_maintext}
\end{figure}
\begin{figure}[t]
    \centering
    \includegraphics[width=0.7\linewidth]{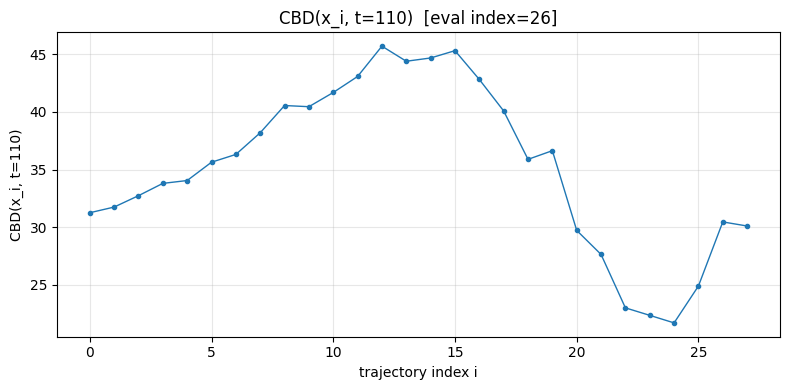}
   \caption{
CBD plots along the baseline trajectory at late time step $26$. Although the values remain very high for all indices, we observe a pronounced peak around 12–15.This is consistent with the intervention-sensitivity results in Figure~\ref{fig:sliding_shiba_maintext}, where the intervention outcome changes around the same window. Thus, the late-time CBD profile gives a possible explanation of the observed transition in terms of localised score-direction instability.
    }
    \label{fig:sliding_shiba_cbd_26_maintext}
\end{figure}
\subsection{Application: phase-aware generation control}
\label{subsec:phase_aware_control}

\paragraph{Concept insertion after semantic commitment.}
Finally, we use the SD~3.5 instability profile as a timing signal for prompt intervention.
Starting from the same initial noise, we run an $N=8$-step sampler and switch prompts at the elbow
$t^\star := \min \left\{ \arg\max_{t}\, (d/dt)[1 - \mathrm{CBD}_\mathrm{TAD}(t)]\right\}$ of the SD~3.5 $1-\mathrm{CBD}_{\mathrm{TAD}}$ profile in Figure~\ref{fig:3model_overlay}.
We interpret this elbow as a semantic commitment boundary. We test both prompt orders:
$\mathcal P_A\!\to\!\mathcal P_B$, where
$\mathcal P_A$ is ``a golden retriever running in a meadow'' and
$\mathcal P_B$ is ``a volcanic eruption with lava'', and the reverse order
$\mathcal P_B\!\to\!\mathcal P_A$. The all-$\mathcal P_A$ and all-$\mathcal P_B$ runs serve as baselines.

Figure~\ref{fig:sd35_control_compact} shows that the prompt active before $t^\star$ selects the
dominant semantic branch. The prompt introduced after $t^\star$ does not globally replace this
branch; instead, its concept is expressed within the committed branch. Thus,
$\mathcal P_A\!\to\!\mathcal P_B$ preserves the golden-retriever branch while adding volcanic or
lava-like elements, whereas $\mathcal P_B\!\to\!\mathcal P_A$ preserves the volcano branch while
inserting dog-like features. This supports a phase-aware control interpretation: the
$\mathrm{CBD}_{\mathrm{TAD}}$ elbow marks the transition from branch selection to within-branch
concept formation. 

\paragraph{Summary.}
Across the experiments, CBD and its trajectory-adapted proxies behave as operational counterparts
of the projection-caustic mechanism. They detect rapid score-direction reorientation in controlled
geometries, rank intervention-sensitive windows in pretrained DDPM trajectories, expose comparable
instability structure across several continuous-state image generators, and provide a practical signal
for phase-aware semantic control without retraining or modifying model weights.


\section{Discussion}

\paragraph{Mechanism and outlook.}
Transition-like behaviour does not require a discontinuous sampler: Gaussian smoothing can make finite-noise denoising multi-branch, so responsibilities switch near projection caustics and sharply reorient the normalised score, velocity, or effective update direction.
CBD measures this reorientation as a trajectory-level signal for branch-sensitive intervention windows, not as an exact zero-noise caustic estimator or universal order parameter; the observed band may shift with density, curvature, amplitudes, and scheduler effects.
This suggests online phase-aware control---adapting guidance, perturbation, solver steps, or compute during sampling---and future work on calibration across continuous samplers, hierarchical concept formation, and new geometry for discrete or absorbing-state models.

\paragraph{Limitations and degenerate regimes.}
Our asymptotic expansions rely on the classical Laplace method for nondegenerate critical points~\citep{wong2001asymptotic}. 
The multi-branch formula near a projection caustic is the finite-branch analogue of small-time heat-kernel asymptotics at a cut locus, where several minimizing geodesics contribute at the same exponential scale~\citep{barilari2012small}. 
If the branch Hessians degenerate, ordinary Laplace asymptotics must be replaced by uniform caustic asymptotics of Airy-- or Pearcey-type. 
Such singular asymptotic methods are uncommon in the diffusion-model literature but well developed in related settings~\citep{kamimoto-mizuno,kamimoto-nose2016,watanabe2009algebraic}. 
We leave this degenerate regime for future work.

\section{Conclusion}

We proposed a geometric account of phase-transition-like behaviour in generative dynamics based on projection caustics.
In projection-regular regions, denoising is controlled by a single nearest branch of the data support; near caustics, finite-noise branch competition produces rapid reorientation of the score, velocity, or effective update direction.
The Critical Boundary Detector operationalises this mechanism by probing direction instability along individual trajectories.
Together, the theory and experiments suggest that projection-caustic geometry is not only a post-hoc explanation of generative transitions, but also a useful organizing principle for phase-aware and intervention-efficient generation.

\newpage

\bibliographystyle{plainnat}
\bibliography{references}


\appendix
\section{Additional Theoretical Details}\label{app:settings}

\subsection{Settings in Subsection \ref{subsec:near_caustic}}
In this subsection, we explain the assumption for Theorem \ref{thm:log-sum}.
Let $x$ be a point of the projection caustic $\mathcal{C}(K)$.  Assume that $x$ has $m$ nearest points, $\Pi_K(x)=\{y_1, \dots, y_m\}$. Let \(U\subset \mathbb{R}^D\) be a neighbourhood of $x$ in which the set of nearest points
consists of \(m\) smooth competing branches
\[
y_1(x),\dots,y_m(x)\in K,
\]
depending smoothly on \(x\in U\).
After possibly shrinking \(U\), we assume that each branch is represented by a smooth
embedded submanifold \(M_j\subset K\) with a local parametrization
\[
\phi_j(\,\cdot\,;x):U_j\subset\mathbb{R}^k \to M_j,
\qquad \phi_j(0;x)=y_j(x),
\]
depending smoothly on \(x\in U\) and write
\[
d\mu_K|_{M_j}=a_j(u,x)\,du,
\]
where $a_j$ is a smooth positive function including the density and Jacobian factor.

Define the phase function
\[
\Phi_j(u;x):=\frac{1}{2}\|x-\phi_j(u;x)\|^2.
\]
We impose the following uniform nondegeneracy and separation conditions:

\begin{itemize}
\item[(i)] For each \(j\) and all \(x\in U\), \(u=0\) is a nondegenerate minimum of
\(\Phi_j(\cdot;x)\), i.e.
\[
\nabla_u \Phi_j(0;x)=0,
\qquad
H_j(x):=\nabla_u^2 \Phi_j(0;x) \text{ is positive definite}.
\]

\item[(ii)] The Hessians are uniformly positive definite: there exists \(c>0\) such that
\[
\langle H_j(x)\xi,\xi\rangle \ge c\|\xi\|^2
\qquad (x\in U,\ \xi\in\mathbb{R}^k,\ 1\le j\le m).
\]

\item[(iii)] There exists \(\delta>0\) such that for all \(x\in U\),
\[
\|x-y\|^2 \ge \min_{1\le j\le m} S_j(x) + \delta
\]
for all \(y\in K\) outside the chosen neighbourhoods of the branches,
where \(S_j(x):=\|x-y_j(x)\|^2\).
\end{itemize}
\section{Small-time score asymptotics near a smooth embedded submanifold}\label{append: smth_sbmfld}

In this subsection, we extend the curve-level local Laplace expansion to a general
smooth embedded submanifold. The main point is that, in the projection-regular region,
the leading term of the score is still governed by the normal distance, while the
$O(1)$ correction is given by the mean-curvature vector and the intrinsic density gradient.

\paragraph{Setting.}
Let $M^m \subset \mathbb R^D$ be a $C^\infty$ embedded submanifold of dimension $m$.
Let $\mu$ be a probability measure on $M$ of the form
\[
d\mu(y)=a(y)\,d \vol_M(y),
\]
where $a\in C^\infty(M)$ is strictly positive.
For $t>0$, define the Gaussian-smoothed density
\[
p_t(x)
:=
\int_M (2\pi t)^{-D/2}\exp\!\left(-\frac{\|x-y\|^2}{2t}\right)a(y)\,d\vol_M(y),
\qquad x\in\mathbb R^D.
\]
We are interested in the small-time asymptotics of
\[
\nabla_x \log p_t(x)
\qquad (t\downarrow 0).
\]

\paragraph{Tubular coordinates and notation.}
Fix $y_0\in M$.
For $\nu\in N_{y_0}M$ sufficiently small, write
\[
x=y_0+\nu.
\]
We use the nearest-point projection $\Pi_M$ in a tubular neighbourhood of $M$.
If $x$ lies in this tubular neighbourhood, then $x$ admits a unique decomposition
\[
x=y+\nu,
\qquad
y=\Pi_M(x)\in M,\quad \nu\in N_yM.
\]

Let $\mathrm{II}_y:T_yM\times T_yM\to N_yM$ be the second fundamental form,
and for $\nu\in N_yM$ define the symmetric endomorphism
\[
A_\nu:T_yM\to T_yM
\]
by
\[
\langle A_\nu v,w\rangle
=
\langle \mathrm{II}_y(v,w),\nu\rangle,
\qquad v,w\in T_yM.
\]
Equivalently, if $\{e_1,\dots,e_m\}$ is an orthonormal basis of $T_yM$, then
\[
(A_\nu)_{ij}
=
\langle \mathrm{II}_y(e_i,e_j),\nu\rangle.
\]
We also write
\[
\mathbf H_y:=\sum_{i=1}^m \mathrm{II}_y(e_i,e_i)\in N_yM
\]
for the mean-curvature vector.

\begin{thm}[Local Laplace asymptotics near a smooth stratum]
\label{thm:submanifold-laplace}
Let $M^m\subset\mathbb R^D$ be a $C^\infty$ embedded submanifold and let
$d\mu=a\,d\vol_M$ with $a\in C^\infty(M)$, $a>0$.
Fix a compact set $K$ contained in a sufficiently small tubular neighbourhood of $M$.
For each $x\in K$, write uniquely
\[
x=y+\nu,
\qquad
y=\Pi_M(x)\in M,\quad \nu\in N_yM.
\]
Then, as $t\downarrow 0$, uniformly for $x\in K$,
\[
p_t(x)
=
(2\pi t)^{-(D-m)/2}
\exp\!\left(-\frac{\|\nu\|^2}{2t}\right)
a(y)\,
\det(I-A_\nu)^{-1/2}
\bigl(1+O(t)\bigr).
\]
Consequently,
\begin{align}
\log p_t(x)
&=
-\frac{\|\nu\|^2}{2t}
-\frac{D-m}{2}\log(2\pi t)
+\log a(y)
-\frac12\log\det(I-A_\nu)
+O(t),
\label{eq:logpt-submanifold}
\\
F_t(x):=-t\log p_t(x)
&=
\frac12\|\nu\|^2
-t\log a(y)
+\frac t2\log\det(I-A_\nu)
+\frac{D-m}{2}t\log(2\pi t)
+O(t^2).
\label{eq:Ft-submanifold}
\end{align}
\end{thm}

\begin{proof}
Fix $x_0=y_0+\nu_0$ in the tubular neighbourhood, with $y_0=\Pi_M(x_0)$ and
$\nu_0\in N_{y_0}M$.
Since the statement is local and uniform on compact subsets, it is enough to prove
the formula in a neighbourhood of $x_0$ with constants depending smoothly on $x_0$.

Choose orthonormal coordinates in $\mathbb R^D$ centered at $y_0$ so that
\[
T_{y_0}M=\mathbb R^m\times\{0\},
\qquad
N_{y_0}M=\{0\}\times\mathbb R^{D-m},
\]
and write points as $(u,z)$ with $u\in\mathbb R^m$, $z\in\mathbb R^{D-m}$.
In these coordinates, after shrinking the neighbourhood if necessary, $M$ is represented
as a graph
\[
\Phi(u)=(u,h(u)),
\]
where
\[
h(0)=0,
\qquad
Dh(0)=0.
\]
The Taylor expansion of $h$ at $u=0$ is
\[
h(u)=\frac12\,\mathrm{II}_{y_0}(u,u)+O(|u|^3),
\]
where $\mathrm{II}_{y_0}(u,u)\in N_{y_0}M\simeq\mathbb R^{D-m}$.
Also, the induced volume form on $M$ satisfies
\[
d\vol_M(\Phi(u))
=
J(u)\,du,
\qquad
J(u)=1+O(|u|^2).
\]
Moreover,
\[
a(\Phi(u))=a(y_0)+O(|u|).
\]

Now write
\[
x_0=(0,\nu_0).
\]
Then
\[
p_t(x_0)
=
(2\pi t)^{-D/2}
\int_{\mathbb R^m}
\exp\!\left(-\frac{\|x_0-\Phi(u)\|^2}{2t}\right)
a(\Phi(u))J(u)\,du,
\]
up to restriction to a sufficiently small coordinate patch, since the contribution away
from the unique minimiser is exponentially small.

We expand the phase function.
Since
\[
x_0-\Phi(u)=(-u,\nu_0-h(u)),
\]
we have
\begin{align*}
\|x_0-\Phi(u)\|^2
&=
|u|^2+\|\nu_0-h(u)\|^2 \\
&=
|u|^2+\|\nu_0\|^2-2\langle \nu_0,h(u)\rangle+\|h(u)\|^2.
\end{align*}
Because $h(u)=O(|u|^2)$, we have $\|h(u)\|^2=O(|u|^4)$, hence
\[
\|x_0-\Phi(u)\|^2
=
\|\nu_0\|^2
+
|u|^2
-
\left\langle \nu_0,\mathrm{II}_{y_0}(u,u)\right\rangle
+
O(|u|^3).
\]
By definition of $A_{\nu_0}$,
\[
\left\langle \nu_0,\mathrm{II}_{y_0}(u,u)\right\rangle
=
\langle A_{\nu_0}u,u\rangle.
\]
Therefore,
\[
\|x_0-\Phi(u)\|^2
=
\|\nu_0\|^2
+
\langle (I-A_{\nu_0})u,u\rangle
+
O(|u|^3).
\]
Since $x_0$ lies in the tubular neighbourhood and $y_0=\Pi_M(x_0)$ is the unique
nearest point, the Hessian in the tangential variables is positive definite:
\[
I-A_{\nu_0}>0.
\]
Thus $u=0$ is a nondegenerate unique minimum of the phase.

We now apply the standard Laplace method.
Set
\[
u=\sqrt{t}\,w.
\]
Then
\begin{align*}
p_t(x_0)
&=
(2\pi t)^{-D/2}
e^{-\|\nu_0\|^2/(2t)}
\int_{\mathbb R^m}
\exp\!\left(
-\frac{1}{2}\langle (I-A_{\nu_0})w,w\rangle + O(\sqrt{t}|w|^3)
\right)
\\
&\qquad\qquad\qquad\qquad\times
\bigl(a(y_0)+O(\sqrt{t}|w|)\bigr)
\bigl(1+O(t|w|^2)\bigr)
t^{m/2}\,dw.
\end{align*}
Since $I-A_{\nu_0}$ is positive definite, the integrand is dominated by a Gaussian,
and therefore
\[
p_t(x_0)
=
(2\pi t)^{-D/2}
e^{-\|\nu_0\|^2/(2t)}
t^{m/2}
\left(
a(y_0)\int_{\mathbb R^m}
e^{-\frac12\langle (I-A_{\nu_0})w,w\rangle}\,dw
+O(t)
\right).
\]
Using the Gaussian integral
\[
\int_{\mathbb R^m}
e^{-\frac12\langle Bw,w\rangle}\,dw
=
(2\pi)^{m/2}(\det B)^{-1/2}
\qquad (B>0),
\]
with $B=I-A_{\nu_0}$, we obtain
\[
p_t(x_0)
=
(2\pi t)^{-(D-m)/2}
e^{-\|\nu_0\|^2/(2t)}
a(y_0)\det(I-A_{\nu_0})^{-1/2}
\bigl(1+O(t)\bigr).
\]
Since the argument is uniform on compact subsets of the tubular neighbourhood,
the stated asymptotic expansion follows. Taking logarithms gives
\eqref{eq:logpt-submanifold}, and multiplying by $-t$ yields
\eqref{eq:Ft-submanifold}.
\end{proof}

\begin{cor}[Score asymptotics in the projection-regular region]
\label{cor:score-submanifold}
Under the assumptions of Theorem~\ref{thm:submanifold-laplace}, one has
\[
\nabla_x\log p_t(x)
=
-\frac{\nu}{t}
+
\nabla_x\!\left(
\log a(y)-\frac12\log\det(I-A_\nu)
\right)
+O(t),
\]
uniformly on compact subsets of the tubular neighbourhood, where $x=y+\nu$,
$y=\Pi_M(x)$, $\nu\in N_yM$.
In particular, for $\|\nu\|$ small,
\begin{equation}
\label{eq:score-submanifold-expanded}
\nabla_x\log p_t(x)
=
-\frac{\nu}{t}
+
\nabla_M\log a(y)
+\frac12\,\mathbf H_y
+O(\|\nu\|)+O(t).
\end{equation}
\end{cor}

\begin{proof}
Differentiate \eqref{eq:logpt-submanifold} with respect to $x$.
Since
\[
\nabla_x\!\left(-\frac{\|\nu\|^2}{2t}\right)=-\frac{\nu}{t},
\]
we get
\[
\nabla_x\log p_t(x)
=
-\frac{\nu}{t}
+
\nabla_x\!\left(
\log a(y)-\frac12\log\det(I-A_\nu)
\right)
+O(t).
\]
This proves the first formula.

It remains to expand the $O(1)$ term for small $\nu$.
First, since $y=\Pi_M(x)$ and $a$ is defined intrinsically on $M$, we have
\[
\nabla_x(\log a(y))
=
\nabla_M\log a(y)+O(\|\nu\|),
\]
because the nearest-point projection has differential equal to the tangential projection
at $\nu=0$, and varies smoothly for small $\nu$.

Next, expand the determinant term at $\nu=0$.
Since $A_\nu$ depends linearly on $\nu$,
\[
\log\det(I-A_\nu)
=
-\tr(A_\nu)+O(\|\nu\|^2).
\]
Therefore
\[
-\frac12\log\det(I-A_\nu)
=
\frac12\tr(A_\nu)+O(\|\nu\|^2).
\]
Now let $\eta\in N_yM$.
Using an orthonormal basis $\{e_i\}_{i=1}^m$ of $T_yM$,
\[
\tr(A_\eta)
=
\sum_{i=1}^m \langle A_\eta e_i,e_i\rangle
=
\sum_{i=1}^m \langle \mathrm{II}_y(e_i,e_i),\eta\rangle
=
\langle \mathbf H_y,\eta\rangle.
\]
Hence the normal gradient of $\frac12\tr(A_\nu)$ at $\nu=0$ is
\[
\frac12\,\mathbf H_y.
\]
Since the tangential derivative of $\tr(A_\nu)$ is $O(\|\nu\|)$, we conclude that
\[
\nabla_x\!\left(-\frac12\log\det(I-A_\nu)\right)
=
\frac12\,\mathbf H_y+O(\|\nu\|).
\]
Combining the two expansions gives
\[
\nabla_x\log p_t(x)
=
-\frac{\nu}{t}
+
\nabla_M\log a(y)
+\frac12\,\mathbf H_y
+O(\|\nu\|)+O(t),
\]
as claimed.
\end{proof}

\begin{cor}[On-manifold score]
\label{cor:on-manifold-score}
For $x=y\in M$ (that is, $\nu=0$), one has
\[
\nabla_x\log p_t(y)
=
\nabla_M\log a(y)
+\frac12\,\mathbf H_y
+O(t).
\]
In particular, if $a\equiv \mathrm{const}$ on $M$, then
\[
\nabla_x\log p_t(y)
=
\frac12\,\mathbf H_y+O(t).
\]
\end{cor}

\begin{proof}
Set $\nu=0$ in \eqref{eq:score-submanifold-expanded}.
\end{proof}

\begin{remark}[Reduction to the curve case]
When $m=1$, write the unit tangent and unit normal by $\tau,n$, and let
$\kappa$ be the scalar curvature with
\[
\mathrm{II}(\tau,\tau)=\kappa n.
\]
Then
\[
\mathbf H=\kappa n,
\]
so Corollary~\ref{cor:on-manifold-score} becomes
\[
\nabla_x\log p_t|_M
=
(\partial_\tau\log a)\,\tau+\frac{\kappa}{2}n+O(t),
\]
which is exactly the curve-level formula.
\end{remark}

\begin{remark}[free energy interpretation]
Theorem~\ref{thm:submanifold-laplace} implies
\[
F_t(x)
=
\frac12\dist(x,M)^2
-t\log a(\Pi_Mx)
+\frac t2\log\det(I-A_\nu)
+\frac{D-m}{2}t\log(2\pi t)
+O(t^2).
\]
Thus, in the projection-regular region, the zero-temperature leading term is purely
the squared distance to the submanifold, while geometry enters at order $t$ through
the Jacobian factor $\det(I-A_\nu)$.
The singular behaviour arises precisely when $I-A_\nu$ degenerates, namely when one
approaches a focal/cut-locus type boundary of projection regularity.
\end{remark}

\section{Free energy asymptotics near projection caustics}
\label{app:caustic_free_energy}

In this appendix, we derive the small-noise asymptotics of the Gaussian-smoothed density
and the associated free energy near a projection caustic. The main point is that, when several
nearest-point branches coexist, the single Laplace contribution from the projection-regular
regime is replaced by a sum of branchwise contributions. As a result, the free energy acquires
a local log-sum-exp normal form, and the score field becomes a softmax-weighted combination
of branchwise gradients. This mechanism explains the rapid directional instability of the score
near the switching set. 

\subsection{Setting and objective}

Let $K \subset \mathbb{R}^D$ be a data support set. We assume that, locally, $K$ is given by
a finite union of smooth embedded submanifolds, each equipped with a smooth positive density.
We consider the Gaussian-smoothed density
\begin{equation}
p_\sigma(x)
:=
\int_K
\rho(y)
\exp\left(
-\frac{\|x-y\|^2}{2\sigma^2}
\right)
d\mu_K(y),
\label{eq:app-psigma}
\end{equation}
and the associated free energy
\begin{equation}
F_\sigma(x):=-\sigma^2 \log p_\sigma(x).
\label{eq:app-free-energy}
\end{equation}
Here $\rho$ is the density of the data distribution on $K$ with respect to
$\mu_K$; it corresponds to $p_0$ in the main text.  We omit the Gaussian
normalisation factor $(2\pi\sigma^2)^{-D/2}$ to simplify the
notation.  Restoring it only adds
\[
\frac{D}{2}\sigma^2\log(2\pi\sigma^2)
\]
to $F_\sigma(x)$, and therefore has no effect on the multi-branch weights,
the switching scale $S_i(x)-S_j(x)=O(\sigma^2)$, or the CBD interpretation.

In the projection-regular region, the nearest-point map $\Pi_K(x)$ is single-valued, and the
usual Laplace method yields an expansion governed by a single minimizer. The phenomenon of
interest here is what happens when $\Pi_K(x)$ becomes multi-valued, namely on the projection
caustic
\begin{equation}
\mathrm{Cau}(K):=\{x\in\mathbb{R}^D:\Pi_K(x)\ \text{is not single-valued}\}.
\end{equation}
We focus on the case where a finite number of nearest-point branches compete.

\subsection{Nondegenerate local multi-branch model with the same dimension}

Fix a point $x\in\mathbb{R}^D$ near the projection caustic. Assume that the relevant competing
nearest points are
\[
y_1,\dots,y_m \in K,
\]
and that each $y_j$ lies on a smooth $k$-dimensional branch $M_j \subset K$, with
$M_i\cap M_j=\varnothing$ for $i\neq j$ in the chosen local model. Let
\[
\phi_j:U_j\subset\mathbb{R}^k \to M_j,\qquad \phi_j(0)=y_j,
\]
be a local coordinate chart, and write
\[
d\mu_K|_{M_j}=a_j(u)\,du,
\]
where $a_j$ is a smooth positive function including the density and Jacobian factor.

For each branch, define the phase function
\begin{equation}
\Phi_j(u;x):=\frac12\|x-\phi_j(u)\|^2.
\label{eq:app-phase}
\end{equation}
By the nearest-point condition, $u=0$ is a critical point of $\Phi_j(\cdot;x)$. We say that
$y_j$ is a nondegenerate nearest point on $M_j$ if the Hessian
\[
H_j(x):=\nabla_u^2 \Phi_j(0;x)
\]
is positive definite. Geometrically, this means that $x-y_j$ is orthogonal to $T_{y_j}M_j$
and the second variation along tangent directions is strictly positive. In terms of the second
fundamental form, one may write
\[
H_j(x)=g_j-\langle x-y_j,\mathrm{II}_{y_j}\rangle,
\]
so degeneration of $\det H_j(x)$ corresponds to a focal-type phenomenon. In the present
appendix we exclude such focal degeneration and assume $H_j(x)>0$ for all competing branches. 

\subsection{Single-branch Laplace expansion}

We first recall the standard local Laplace expansion in a form suited to our application.

\begin{lem}
Let $f\in C_c^\infty(U)$ and $\Phi\in C^\infty(U)$, where $U\subset\mathbb{R}^k$ is a
neighborhood of the origin. Assume that
\[
\Phi(0)=\frac{d^2}{2},\qquad \nabla \Phi(0)=0,\qquad H:=\nabla^2\Phi(0)\text{ is positive definite},
\]
for some $d>0$.
Then
\[
I_\sigma:=\int_U f(u)e^{-\Phi(u)/\sigma^2}\,du
\]
admits the asymptotic expansion
\begin{equation}
I_\sigma
=
(2\pi\sigma^2)^{k/2}
e^{-d^2/(2\sigma^2)}
\frac{f(0)}{\sqrt{\det H}}
\bigl(1+O(\sigma^2)\bigr)
\qquad (\sigma\to0).
\label{eq:app-single-laplace}
\end{equation}
\end{lem}

\begin{proof}
By Taylor expansion at the origin,
\[
\Phi(u)=\frac{d^{2}}{2}+\frac12 \langle Hu,u\rangle + R_3(u),
\qquad
R_3(u)=O(\|u\|^3).
\]
Since \(\nabla \Phi(0)=0\) and \(H>0\), the point \(u=0\) is a nondegenerate local minimum.
After shrinking \(U\) if necessary, there exist \(c,C>0\) such that
\[
\Phi(u)\ge \frac{d^2}{2}+c\|u\|^2
\qquad (u\in U),
\]
and
\[
|R_3(u)|\le C\|u\|^3 .
\]

Set \(u=\sigma v\). Then
\[
I_\sigma
=
e^{-d^2/(2\sigma^2)}\sigma^k
\int_{U/\sigma}
f(\sigma v)
\exp\!\left(
-\frac12\langle Hv,v\rangle
-\frac{R_3(\sigma v)}{\sigma^2}
\right)\,dv.
\]
Since
\[
\frac{R_3(\sigma v)}{\sigma^2}=O(\sigma \|v\|^3),
\]
the integrand is dominated by an integrable Gaussian on each fixed bounded \(v\)-region, and
the contribution from \(|v|\) is exponentially small because \(H>0\). Thus the leading term is
\[
e^{-d^2/(2\sigma^2)}\sigma^k
f(0)\int_{\mathbb R^k} e^{-\langle Hv,v\rangle/2}\,dv.
\]

To obtain the sharper relative error \(O(\sigma^2)\), we expand one order further.
Write
\[
f(\sigma v)=f(0)+\sigma \nabla f(0)\cdot v + O(\sigma^2 \|v\|^2),
\]
and
\[
R_3(u)=T_3(u)+O(\|u\|^4),
\]
where \(T_3\) is the cubic Taylor polynomial of \(\Phi\) at the origin. Hence
\[
\frac{R_3(\sigma v)}{\sigma^2}
=
\sigma T_3(v)+O(\sigma^2\|v\|^4),
\]
so
\[
\exp\!\left(-\frac{R_3(\sigma v)}{\sigma^2}\right)
=
1-\sigma T_3(v)+O\!\bigl(\sigma^2(1+\|v\|^6)\bigr).
\]
Therefore,
\[
f(\sigma v)\exp\!\left(-\frac{R_3(\sigma v)}{\sigma^2}\right)
=
f(0)
+\sigma\bigl(\nabla f(0)\cdot v - f(0)T_3(v)\bigr)
+O\!\bigl(\sigma^2(1+\|v\|^6)\bigr).
\]

Multiplying by the Gaussian factor \(e^{-\langle Hv,v\rangle/2}\), the coefficient of \(\sigma\) is an odd
function of \(v\), since \(\nabla f(0)\cdot v\) is linear and \(T_3(v)\) is cubic. Hence its integral over
\(\mathbb R^k\) vanishes:
\[
\int_{\mathbb R^k}
\bigl(\nabla f(0)\cdot v - f(0)T_3(v)\bigr)
e^{-\langle Hv,v\rangle/2}\,dv
=0.
\]
It follows that
\[
I_\sigma
=
e^{-d^2/(2\sigma^2)}\sigma^k
\left(
f(0)\int_{\mathbb R^k} e^{-\langle Hv,v\rangle/2}\,dv
+O(\sigma^2)
\right).
\]
Finally, using
\[
\int_{\mathbb R^k} e^{-\langle Hv,v\rangle/2}\,dv
=
(2\pi)^{k/2}(\det H)^{-1/2},
\]
we obtain
\[
I_\sigma
=
(2\pi\sigma^2)^{k/2}e^{-d^2/(2\sigma^2)}
\frac{f(0)}{\sqrt{\det H}}
\bigl(1+O(\sigma^2)\bigr).
\]
This proves the claim.
\end{proof}

\subsection{Multi-branch Laplace expansion on the caustic}

We now fix $x\in\mathrm{Cau}(K)$ and assume that
\[
\Pi_K(x)=\{y_1,\dots,y_m\}
\]
is a finite set of nondegenerate nearest points. Let
\[
d:=d(x,K).
\]
Assume moreover that there exists $\delta>0$ such that all points outside small neighborhoods
$V_j$ of the $y_j$ satisfy
\[
\|x-y\|^2 \ge d(x,K)^2+\delta
\qquad
\left(y\in K\setminus \bigcup_{j=1}^m V_j\right).
\]

\begin{thm}[Multi-branch Laplace expansion on the caustic]
Under the assumptions above, one has
\begin{equation}
p_\sigma(x)
=
(2\pi\sigma^2)^{k/2}
e^{-d^2/(2\sigma^2)}
\left(
\sum_{j=1}^m B_j(x)+O(\sigma^2)
\right)
+
O\!\left(e^{-(d^2+\delta)/(2\sigma^2)}\right),
\label{eq:app-caustic-expansion}
\end{equation}
where
\begin{equation}
B_j(x):=\frac{\rho(y_j)a_j(0)}{\sqrt{\det H_j(x)}}>0.
\label{eq:app-Bj}
\end{equation}
Consequently,
\begin{equation}
F_\sigma(x)
=
\frac12 d^2
-\frac{k}{2}\sigma^2\log(2\pi\sigma^2)
-\sigma^2 \log\!\left(\sum_{j=1}^m B_j(x)\right)
+O(\sigma^4).
\label{eq:app-caustic-free-energy}
\end{equation}
\end{thm}

\begin{proof}
We first decompose the integral defining \(p_\sigma(x)\) into the contributions from small neighbourhoods of the competing nearest points and the remainder away from them.
Choose pairwise disjoint neighbourhoods \(V_j\subset K\) of \(y_j\) \((j=1,\dots,m)\) so small that
\[
K=\Bigl(\bigcup_{j=1}^m V_j\Bigr)\cup \Bigl(K\setminus \bigcup_{j=1}^m V_j\Bigr),
\]
and, by the isolation assumption on the nearest points, there exists \(\delta>0\) such that
\[
\|x-y\|^2\ge d(x,K)^2+\delta
\qquad
\text{for all }y\in K\setminus \bigcup_{j=1}^m V_j.
\]
Hence
\[
p_\sigma(x)=\sum_{j=1}^m I_{j,\sigma}(x)+R_\sigma(x),
\]
where
\[
I_{j,\sigma}(x):=\int_{V_j}\rho(y)\exp\!\left(-\frac{\|x-y\|^2}{2\sigma^2}\right)\,d\mu_K(y),
\]
and
\[
R_\sigma(x):=\int_{K\setminus \cup_j V_j}\rho(y)\exp\!\left(-\frac{\|x-y\|^2}{2\sigma^2}\right)\,d\mu_K(y).
\]
The remainder is exponentially small by the distance gap assumption:
\[
|R_\sigma(x)|\le C e^{-(d^2+\delta)/(2\sigma^2)}.
\]
On the other hand, in local coordinates,
\[
I_{j,\sigma}(x)
=
\int_{U_j} f_j(u)e^{-\Phi_j(u;x)/\sigma^2}\,du,
\qquad
f_j(u):=\rho(\phi_j(u))a_j(u),
\]
and the single-branch Laplace lemma applies since
\[
\Phi_j(0;x)=\frac{d^2}{2},\qquad \nabla_u\Phi_j(0;x)=0,\qquad H_j(x)>0.
\]
Thus,
\[
I_{j,\sigma}(x)
=
(2\pi\sigma^2)^{k/2}
e^{-d^2/(2\sigma^2)}
\bigl(B_j(x)+O(\sigma^2)\bigr).
\]
Summing over $j$ yields \eqref{eq:app-caustic-expansion}. Taking $-\sigma^2\log$ of the result
gives \eqref{eq:app-caustic-free-energy}.
\end{proof}

\begin{remark}
At a projection-regular point, there is only one minimizer, so the sum $\sum_j B_j(x)$ collapses
to a single term. On the caustic, several minimizers coexist at the same exponential scale, and
their amplitudes add. This is the origin of the switching behaviour that appears after differentiating
the free energy.
\end{remark}
\subsection{Uniform multi-branch Laplace expansion near the caustic}

We now extend the fixed-point Laplace expansion to a neighbourhood of the caustic,
making the dependence on \(x\) uniform.

\medskip

Let \(U\subset \mathbb{R}^d\) be a neighbourhood in which the set of nearest points
consists of \(m\) smooth competing branches
\[
y_1(x),\dots,y_m(x)\in K,
\]
depending smoothly on \(x\in U\).

After possibly shrinking \(U\), we assume that each branch is represented by a smooth
embedded submanifold \(M_j\subset K\) with a local parametrization
\[
\phi_j(\,\cdot\,;x):U_j\subset\mathbb{R}^k \to M_j,
\qquad \phi_j(0;x)=y_j(x),
\]
depending smoothly on \(x\in U\).
Define the phase function
\[
\Phi_j(u;x):=\frac{1}{2}\|x-\phi_j(u;x)\|^2.
\]
We impose the following uniform nondegeneracy and separation conditions:

\begin{itemize}
\item[(i)] For each \(j\) and all \(x\in U\), \(u=0\) is a nondegenerate minimum of
\(\Phi_j(\cdot;x)\), i.e.
\[
\nabla_u \Phi_j(0;x)=0,
\qquad
H_j(x):=\nabla_u^2 \Phi_j(0;x) \text{ is positive definite}.
\]

\item[(ii)] The Hessians are uniformly positive definite: there exists \(c>0\) such that
\[
\langle H_j(x)\xi,\xi\rangle \ge c\|\xi\|^2
\qquad (x\in U,\ \xi\in\mathbb{R}^k,\ 1\le j\le m).
\]

\item[(iii)] There exists \(\delta>0\) such that for all \(x\in U\),
\[
\|x-y\|^2 \ge \min_{1\le j\le m} S_j(x) + \delta
\]
for all \(y\in K\) outside the chosen neighbourhoods of the branches,
where \(S_j(x):=\|x-y_j(x)\|^2\).
\end{itemize}

\medskip

\begin{thm}[Uniform multi-branch Laplace expansion]
Under the above assumptions, there exist smooth positive functions
\[B_1,\dots,B_j=\frac{\rho(y_j(x)) a(0; x))}{\sqrt{\det H_j(x)}},\dots,B_m,\] 
on \(U\) such that,
\begin{equation}
p_\sigma(x)
=
(2\pi\sigma^2)^{k/2}
\sum_{j=1}^m
B_j(x)\,e^{-S_j(x)/(2\sigma^2)}
\bigl(1+O(\sigma^2)\bigr),
\qquad \sigma\to 0, 
\tag{13}
\end{equation}
uniformly for \(x\in U\),
Consequently,
\begin{equation}\label{eq:app-logsumexp-free-energy}
F_\sigma(x)
=
-\sigma^2 \log\!\left(
(2\pi\sigma^2)^{k/2}
\sum_{j=1}^m
B_j(x)\,e^{-S_j(x)/(2\sigma^2)}
\right)
+O(\sigma^4),
\tag{14}
\end{equation}
uniformly for \(x\in U\).
\end{thm}

\begin{proof}
We decompose the integral defining \(p_\sigma(x)\) into contributions from the
branch neighbourhoods and a remainder:
\[
p_\sigma(x)
=
\sum_{j=1}^m I_{j,\sigma}(x)
+
R_\sigma(x).
\]

By the uniform separation assumption (iii), there exist constants \(C,\delta>0\)
such that
\[
|R_\sigma(x)|
\le
C\exp\!\left(
-\frac{\min_{1\le \ell\le m} S_\ell(x)+\delta}{2\sigma^2}
\right),
\qquad x\in U.
\]
Hence the remainder is uniformly exponentially smaller than the leading terms.

For each branch \(j\), introduce local coordinates:
\[
I_{j,\sigma}(x)
=
\int_{U_j}
f_j(u;x)\,e^{-\Phi_j(u;x)/\sigma^2}\,du,
\]
where \(f_j(u;x)\) is a smooth positive amplitude incorporating the density
and Jacobian factor.

Since \(y_j(x)\) is a nearest point,
\[
\Phi_j(0;x)=\frac{1}{2}S_j(x),
\qquad
\nabla_u \Phi_j(0;x)=0.
\]
By the uniform nondegeneracy assumption (ii),
we have the Taylor expansion
\[
\Phi_j(u;x)
=
\frac{1}{2}S_j(x)
+
\frac{1}{2}\langle H_j(x)u,u\rangle
+
R_{j,3}(u;x),
\]
with
\[
R_{j,3}(u;x)=O(\|u\|^3),
\]
uniformly in \(x\in U\).

Rescaling \(u=\sigma v\), we obtain
\[
I_{j,\sigma}(x)
=
e^{-S_j(x)/(2\sigma^2)}\sigma^k
\int_{U_j/\sigma}
f_j(\sigma v;x)
\exp\!\left(
-\frac{1}{2}\langle H_j(x)v,v\rangle
-
\frac{R_{j,3}(\sigma v;x)}{\sigma^2}
\right)\,dv.
\]
Since
\[
\frac{R_{j,3}(\sigma v;x)}{\sigma^2}
=
O(\sigma\|v\|^3)
\]
uniformly in \(x\), and the condition (ii) \(H_j(x)\) is uniformly positively definite for each $j$, the integrand is dominated by a
Gaussian independent of \(x\).
Hence the dominated convergence theorem applies uniformly, yielding
\[
I_{j,\sigma}(x)
=
(2\pi\sigma^2)^{k/2}
e^{-S_j(x)/(2\sigma^2)}
\left(
B_j(x)+O(\sigma^2)
\right),
\]
uniformly for \(x\in U\), where
\[
B_j(x)
=
\frac{f_j(0;x)}{\sqrt{\det H_j(x)}}.
\]

Summing over \(j\) and absorbing the exponentially small remainder into the error,
we obtain \((13)\).

Finally, writing
\[
p_\sigma(x)
=
A_\sigma(x)\bigl(1+O(\sigma^2)\bigr),
\quad
A_\sigma(x)
:=
(2\pi\sigma^2)^{k/2}
\sum_{j=1}^m
B_j(x)e^{-S_j(x)/(2\sigma^2)},
\]
we have
\[
F_\sigma(x)
=
-\sigma^2 \log A_\sigma(x)
-
\sigma^2 \log\bigl(1+O(\sigma^2)\bigr).
\]
Since \(\log(1+O(\sigma^2))=O(\sigma^2)\) uniformly, the second term is \(O(\sigma^4)\),
which proves \((14)\).
\end{proof}


\subsection{Score as a softmax mixture and multi-branch switching}

Differentiating \eqref{eq:app-logsumexp-free-energy}, we obtain
\begin{equation}
\nabla \log p_\sigma(x)
=
\sum_{j=1}^m
w_j(x,\sigma)
\left(
-\frac{1}{2\sigma^2}\nabla S_j(x)+\nabla \log B_j(x)
\right)
+O(1),
\label{eq:app-score-switching}
\end{equation}
where the weights are
\begin{equation}
w_j(x,\sigma)
:=
\frac{B_j(x)e^{-S_j(x)/(2\sigma^2)}}
{\sum_{\ell=1}^m B_\ell(x)e^{-S_\ell(x)/(2\sigma^2)}}.
\label{eq:app-softmax-weight}
\end{equation}
These weights satisfy
\[
0<w_j(x,\sigma)<1,
\qquad
\sum_{j=1}^m w_j(x,\sigma)=1.
\]
Hence the score is a convex combination of the branchwise gradients, with coefficients given by
a softmax over the local actions $S_j(x)$.

To describe the switching mechanism, define
\[
S_*(x):=\min_{1\le j\le m} S_j(x),
\qquad
I_*(x):=\{j: S_j(x)=S_*(x)\}.
\]
Then, as $\sigma\to0$, the softmax concentrates on the minimizing branches:
\[
w_j(x,\sigma)\to 0
\qquad\text{for } j\notin I_*(x).
\]
If the minimizer is unique, say $I_*(x)=\{j_*\}$, then
\[
w_{j_*}(x,\sigma)\to 1,
\qquad
w_j(x,\sigma)\to 0 \quad (j\neq j_*).
\]
If several branches tie, then only those branches survive at leading order, and the limiting
weights are proportional to their amplitudes:
\[
w_j(x,\sigma)
=
\frac{B_j(x)}{\sum_{\ell\in I_*(x)} B_\ell(x)}
+o(1),
\qquad j\in I_*(x).
\]
The transition between dominant branches occurs when several actions become comparable,
namely in the region where
\[
S_i(x)-S_j(x)=O(\sigma^2)
\]
for competing branches $i,j$. Thus, away from the switching set, the score is asymptotically
governed by a single branch, whereas near the projection caustic it rapidly changes between
competing minimizers. This is the multi-branch replacement of the two-branch logistic switching
formula in the original note. 

\begin{remark}
This switching mechanism is precisely what produces strong instability of the normalised score
direction near the projection caustic. From the perspective of the main text, this provides the
theoretical basis for CBD: peaks of CBD should be interpreted as signatures of rapid variation
in the softmax-selected score direction caused by competition among nearest-point branches.
\end{remark}

\subsection{Explicit example: the transverse cross}

Consider the transverse cross
\[
K=\{(s,0):s\in\mathbb{R}\}\cup\{(0,t):t\in\mathbb{R}\}\subset\mathbb{R}^2,
\]
with branchwise densities $\rho_x$ and $\rho_y$ on the two coordinate axes. For a point
$(u,v)\in\mathbb{R}^2$, one has
\begin{align}
p_\sigma(u,v)
&=
\int_{\mathbb{R}}
\rho_x(s)\exp\!\left(
-\frac{(u-s)^2+v^2}{2\sigma^2}
\right)\,ds
\nonumber\\
&\quad
+
\int_{\mathbb{R}}
\rho_y(t)\exp\!\left(
-\frac{u^2+(v-t)^2}{2\sigma^2}
\right)\,dt.
\end{align}
Factoring out the branch-independent exponentials gives
\begin{equation}
p_\sigma(u,v)
=
a_x(u,\sigma)e^{-v^2/(2\sigma^2)}
+
a_y(v,\sigma)e^{-u^2/(2\sigma^2)},
\label{eq:app-cross-density}
\end{equation}
where
\[
a_x(u,\sigma):=\int_{\mathbb{R}}\rho_x(s)e^{-(u-s)^2/(2\sigma^2)}\,ds,
\qquad
a_y(v,\sigma):=\int_{\mathbb{R}}\rho_y(t)e^{-(v-t)^2/(2\sigma^2)}\,dt.
\]
Hence the free energy is exactly
\begin{equation}
F_\sigma(u,v)
=
-\sigma^2
\log\!\left(
a_x(u,\sigma)e^{-v^2/(2\sigma^2)}
+
a_y(v,\sigma)e^{-u^2/(2\sigma^2)}
\right).
\label{eq:app-cross-free-energy}
\end{equation}

The projection caustic is given by $|u|=|v|$, since the distance to the $x$-axis is $|v|$ and the
distance to the $y$-axis is $|u|$. On the caustic, the two branch contributions survive at the
same exponential order, so \eqref{eq:app-cross-free-energy} reduces to the concrete two-branch
instance of the general theorem above.

\subsection{Generalisations and caveats}

\paragraph{Branches of different dimensions.}
If the branch dimensions are $k_j$ rather than a common $k$, then
\[
I_{j,\sigma}(x)\sim (2\pi\sigma^2)^{k_j/2} B_j(x)e^{-d^2/(2\sigma^2)}.
\]
Thus the competition involves not only the exponential terms but also powers of $\sigma$.
At the amplitude level, lower-dimensional branches may become dominant. Therefore, if one
wants a completely symmetric competition, it is natural to compare branches of the same
dimension.

\paragraph{Focal degeneration.}
If $H_j(x)$ degenerates, the standard Laplace method breaks down and one expects uniform
asymptotics of Airy or Pearcey type. Such points are more singular than the ordinary
projection-caustic regime treated here.

\paragraph{Projection caustic versus singular support.}
The singularity discussed here concerns the multi-valuedness of the projection on the $x$-side.
The set $K$ itself need not be globally smooth. Even a union of smooth branches may produce
a projection caustic when viewed from the ambient space.

\begin{remark}[Relation to classical and singular asymptotics]
The expansions above use the classical Laplace method for integrals with
nondegenerate critical points; see, e.g., \cite{wong2001asymptotic}.
The multi-branch formula should be viewed as the finite-branch analogue
of small-time heat-kernel asymptotics at a cut locus, where several minimizing
geodesics contribute at the same exponential scale; see \cite{barilari2012small}.
In the present work we exclude focal degeneracies by assuming that the branch Hessians
are uniformly positive definite. When this nondegeneracy fails, the ordinary Laplace
expansion is no longer sufficient, and one expects singular asymptotic methods based on
Newton polyhedra, toric resolution, or uniform caustic normal forms to become relevant;
see, for example, \cite{cho-kamimoto-nose2013}
and \cite{kamimoto-nose2016, kamimoto-mizuno}.
\end{remark}

\section{Numerical Analysis}

In this section, we provide empirical evidence that phase transitions occur within projection-caustic switching bands.

\subsection{Cusp experiment: commitment and switching-band geometry}\label{app:cusp}

In this subsection, we study a two-dimensional toy diffusion model whose data support is concentrated near a cusp. The purpose of this experiment is to examine how the branch-selection dynamics of reverse diffusion relate to the projection-induced switching geometry discussed in the main text.

\paragraph{Dataset.}
We consider samples generated from the planar cusp
\[
\gamma(u) = (u^2, u^3), \qquad u \in [-u_{\max},u_{\max}],
\]
with a small additional observational Gaussian noise. In the implementation used here, the parameter $u$ is sampled uniformly, the resulting points are perturbed by small noise, and the dataset is normalised before training. The corresponding code constructs a dataset of $30{,}000$ samples on the cusp manifold with small noise.

\begin{figure}[t]
    \centering
    \includegraphics[scale=0.5]{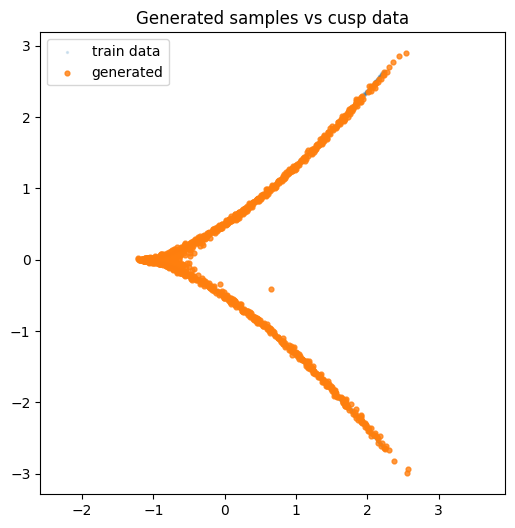}
    \caption{}
    \label{fig:cusp}
\end{figure}

\paragraph{Diffusion model and training.}
We train a standard two-dimensional DDPM with a linear variance schedule over $T=300$ diffusion steps. The denoiser is a multilayer perceptron with sinusoidal time embedding, trained by the usual noise-prediction objective using Adam. In the reported run, the training loss decreases from approximately $0.586$ at the first epoch to approximately $0.317$ after $100$ epochs, indicating that the model learns a stable approximation of the reverse denoising dynamics on this toy distribution.

\paragraph{Generated samples and branch structure.}
After training, we generate reverse trajectories from Gaussian initial conditions and store the full sampling paths. In the reported experiment, the stored trajectory tensor has shape
\[
(301,1500,2),
\]
corresponding to $1500$ sampled trajectories over $301$ saved reverse-time states. The final generated samples clearly concentrate near the two visible branches of the cusp, namely the upper branch $(y>0)$ and the lower branch $(y<0)$.

To analyse branch selection, we assign to each final sample a binary label:
\[
\mathrm{label} =
\begin{cases}
1, & y_{\mathrm{final}} > 0,\\
0, & y_{\mathrm{final}} \le 0.
\end{cases}
\]
In the reported run, this yields $770$ trajectories ending on the upper branch and $730$ trajectories ending on the lower branch.

\paragraph{Intermediate decision field.}
To visualise how branch information emerges during reverse diffusion, we estimate the conditional probability
\[
P(\mathrm{upper}\mid x_k \approx x)
\]
at intermediate reverse-time states using a local $k$-nearest-neighbour estimator on the trajectory cloud. This produces a probabilistic decision field over the ambient plane, together with an empirical decision boundary defined by the level set
\[
P(\mathrm{upper}\mid x_k \approx x)=\tfrac12.
\]
The resulting plots show that the branch decision is initially diffuse, but progressively sharpens into a structured separation between the upper and lower branches as the reverse process approaches the data support (Figure \ref{fig:evolution}).
\begin{figure}[htbp]
    \centering

    \begin{subfigure}{0.45\textwidth}
        \centering
        \includegraphics[width=\linewidth]{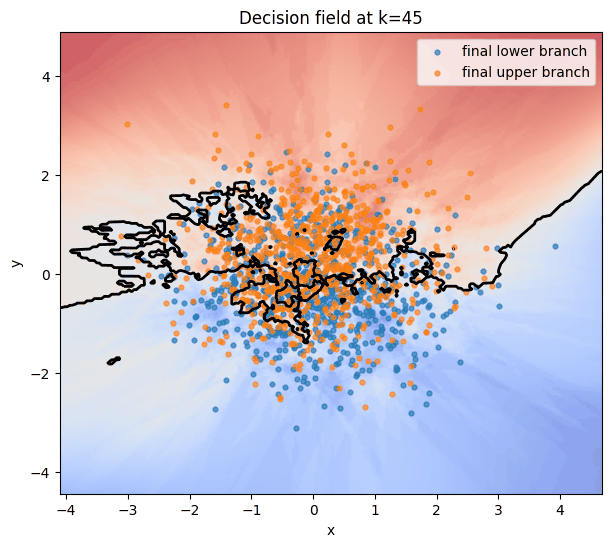}
        \caption{$k=45$}
    \end{subfigure}
    \hfill
    \begin{subfigure}{0.45\textwidth}
        \centering
        \includegraphics[width=\linewidth]{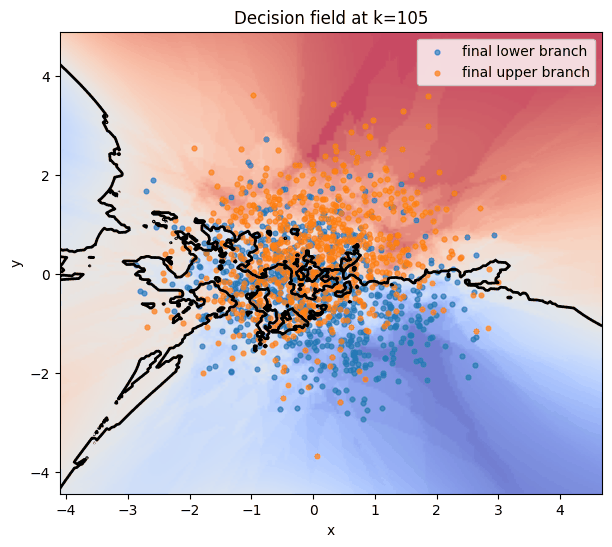}
        \caption{$k=105$}
    \end{subfigure}

    \vspace{0.5em}

    \begin{subfigure}{0.45\textwidth}
        \centering
        \includegraphics[width=\linewidth]{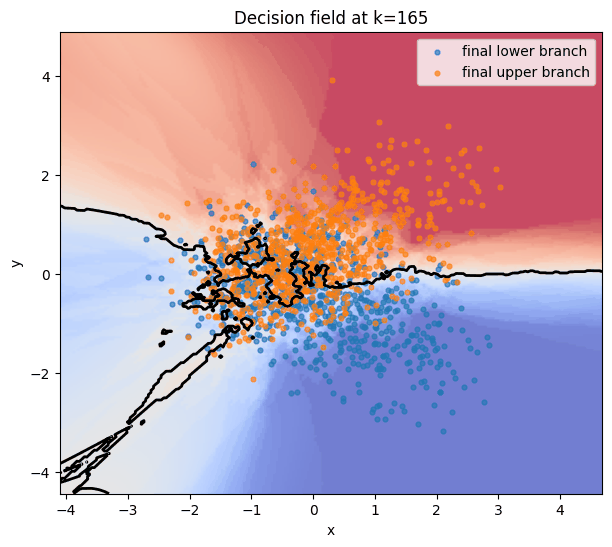}
        \caption{$k=165$}
    \end{subfigure}
    \hfill
    \begin{subfigure}{0.45\textwidth}
        \centering
        \includegraphics[width=\linewidth]{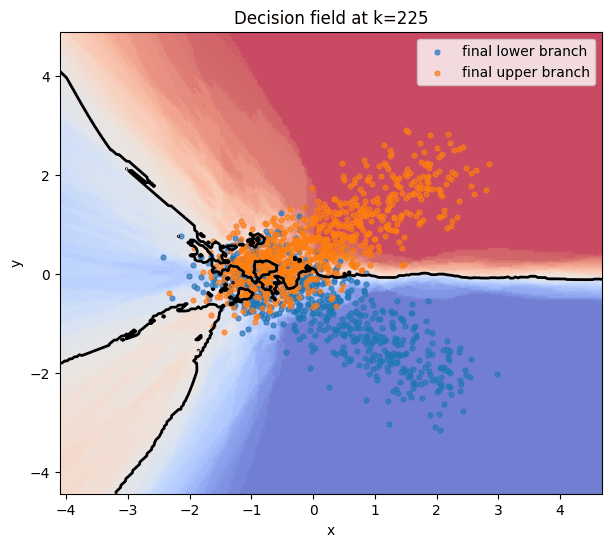}
        \caption{$k=225$}
    \end{subfigure}

    \vspace{0.5em}

    \begin{subfigure}{0.45\textwidth}
        \centering
        \includegraphics[width=\linewidth]{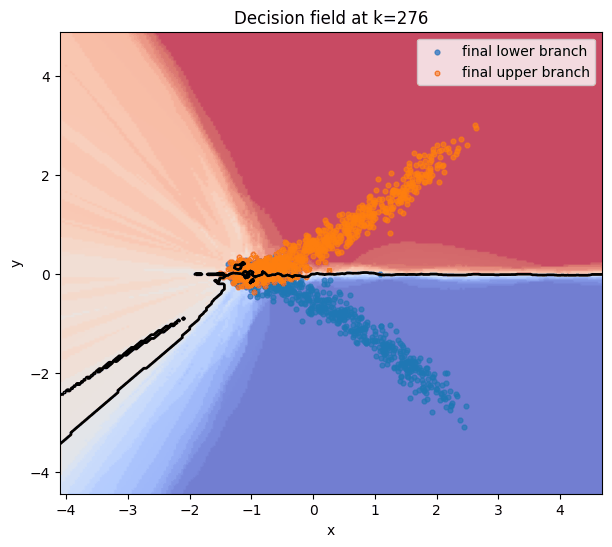}
        \caption{$k=276$}
    \end{subfigure}

    \caption{Evolution of the empirical branch-decision field during reverse diffusion on the cusp dataset.
    At each saved reverse-time index $k$, the background color shows the estimated conditional probability
    $P(\mathrm{upper}\mid x_k \approx x)$, while the black contour denotes the empirical decision boundary
    $P=1/2$. Points are colored by their final branch labels.
    The figure shows that the upper/lower branch decision is initially diffuse and gradually sharpens
    into a structured separation as reverse diffusion approaches the data support.}
    \label{fig:evolution}
\end{figure}
\paragraph{Commitment time.}
To quantify when a trajectory commits to one branch, we define the branch confidence at saved reverse step $k$ by
\[
\mathrm{conf}_i(k) = \left| P_i(k) - \tfrac12 \right|,
\]
where $P_i(k)$ denotes the estimated probability that trajectory $i$ will end on the upper branch, conditioned on its state at time $k$.

Let $x_k^{(i)} \in \mathbb{R}^2$ denote the state of trajectory $i$ at saved reverse-time index $k$, and let
\[
y_j =
\begin{cases}
1, & \text{if trajectory } j \text{ ends on the upper branch},\\
0, & \text{if trajectory } j \text{ ends on the lower branch}.
\end{cases}
\]
To estimate how strongly the intermediate state $x_k^{(i)}$ predicts the final branch outcome, we define
\[
P_i(k) := \widehat{P}(\mathrm{upper}\mid x_k \approx x_k^{(i)}),
\]
where the conditional probability is estimated by a weighted $k$-nearest-neighbor rule on the cloud of points
\[
\{x_k^{(1)},\dots,x_k^{(N)}\}.
\]
More precisely, let $\mathcal{N}_k(x_k^{(i)})$ be the set of the $K$ nearest neighbors of $x_k^{(i)}$ among the time-$k$ trajectory points, and let
\[
w_{ij}^{(k)} = \exp\!\left(-\frac{\|x_k^{(i)}-x_k^{(j)}\|^2}{\tau_k}\right),
\]
where $\tau_k$ is an adaptive temperature parameter taken from the local neighbor-distance scale. In the experiment, we use $K=25$. 

Then we define
\[
P_i(k)
=
\frac{\sum_{j \in \mathcal{N}_k(x_k^{(i)})} w_{ij}^{(k)}\, y_j}
{\sum_{j \in \mathcal{N}_k(x_k^{(i)})} w_{ij}^{(k)}}.
\]

We then define the commitment index of trajectory $i$ as the first saved step at which this confidence exceeds a fixed threshold:
\[
k_i^{\ast} = \min \left\{ k : \left| P_i(k)-\tfrac12 \right| > \tau \right\},
\]
with threshold $\tau = 0.40$ in the present experiment. Thus, commitment means that the future branch outcome has become highly predictable from the current state.


\paragraph{Geometric switching band from the cusp projection problem.}
We next compare the empirical commitment points with the projection-induced switching geometry of the cusp. For a point $(x,y)$ in the ambient plane, we consider the squared distance to the cusp parameterized by $u$:
\[
S(u;x,y) = (x-u^2)^2 + (y-u^3)^2.
\]
For each grid point in the plane, we numerically identify the two smallest local minima of $S(u;x,y)$, denoted by $S_1(x,y)$ and $S_2(x,y)$. The quantity
\[
\Delta S(x,y)=S_2(x,y)-S_1(x,y)
\]
measures the degree of competition between the two leading projection branches.

Motivated by the multi-branch log-sum-exp asymptotics derived in the theoretical part of the paper, the relevant switching regime is expected to occur where competing branch actions differ by $O(\sigma^2)$. In the main text, this appears through the softmax weights and the statement that switching occurs on the scale
\[
S_i(x)-S_j(x)=O(\sigma^2).
\]

Accordingly, we define an approximate switching band by the criterion
\[
\frac{S_2(x,y)-S_1(x,y)}{\sigma_{\mathrm{ref}}^2} \le c,
\]
where $\sigma_{\mathrm{ref}}$ is the effective noise scale associated with a representative reverse-time slice and $c$ is a fixed threshold. In the experiment reported here, the representative slice is chosen at saved-step index $k_{\mathrm{ref}}=165$, corresponding to $t_{\mathrm{ref}}=135$ and
\[
\sigma_{\mathrm{ref}} \approx 0.6823.
\]
Using this scale, the approximate switching-band mask occupies about $14.6\%$ of the plotting grid.

\paragraph{Distance from commitment points to the switching band.}
We then measure the Euclidean distance from each empirical commitment point to the approximate switching band. The resulting histogram (Figure \ref{fig:histo}) shows a strong concentration near small distances. Quantitatively, the mean distance from a commitment point to the switching band is approximately $0.501$, while the median distance is approximately $0.322$. Thus, commitment points tend to lie near the switching ridge generated by competing projections, although the alignment is not exact.

\begin{figure}[h]
    \centering
    \includegraphics[scale=0.5]{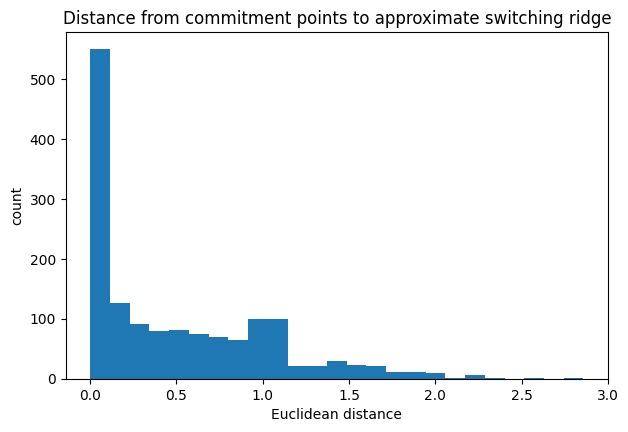}
    \caption{
    }
    \label{fig:histo}
\end{figure}

\paragraph{Residence intervals inside the switching band.}
To examine the temporal relation more directly, we compute for each trajectory the time intervals during which the trajectory remains inside the switching band. Using a geometric tolerance based on the grid spacing, the experiment finds that approximately $56.0\%$ of all trajectory-time points lie inside the band. At the trajectory level, approximately $98.3\%$ of trajectories intersect the switching band at least once. More importantly, approximately $47.1\%$ of trajectories have their commitment index inside some switching-band residence interval, while approximately $12.9\%$ commit inside the first such interval. Conversely, approximately $52.9\%$ commit outside all detected switching-band intervals. 
For those trajectories, the mean temporal gap from commitment to the nearest switching-band segment is about $15.75$ in saved-step index, with median $1.0$. Here, a saved step refers to one stored reverse-time state along the sampled trajectory. The small median indicates that, for many trajectories, commitment occurs very close in time to a switching-band residence interval, while the larger mean suggests the presence of a smaller number of trajectories with substantially larger gaps.
\paragraph{Interpretation.}
These results support the geometric picture developed in the main text. The switching band extracted from the projection competition of the cusp is strongly correlated with the region where branch selection becomes unstable and where commitment tends to emerge. At the same time, commitment is not identical to the geometric switching set itself. 

Overall, this experiment provides empirical evidence that phase-transition-like branch commitment occurs preferentially near projection-caustic switching bands, thereby supporting the interpretation of reverse diffusion transitions as a manifestation of multi-branch geometric competition.

\subsection{CBD-guided branch control on the cross dataset}\label{app:cross}

We further test whether the projection-caustic switching band can be used as an intervention target.
For this purpose, we train a two-dimensional DDPM on the transverse cross
\(
K=\{(x_1,x_2)\in\mathbb{R}^2:x_1x_2=0\}
\)
and compute the normalised CBD field along the ambient plane. The resulting high-CBD ridge is
concentrated near the branch-switching set, as predicted by the two-branch competition picture.

\begin{figure}[h]
    \centering
    \includegraphics[scale=0.5]{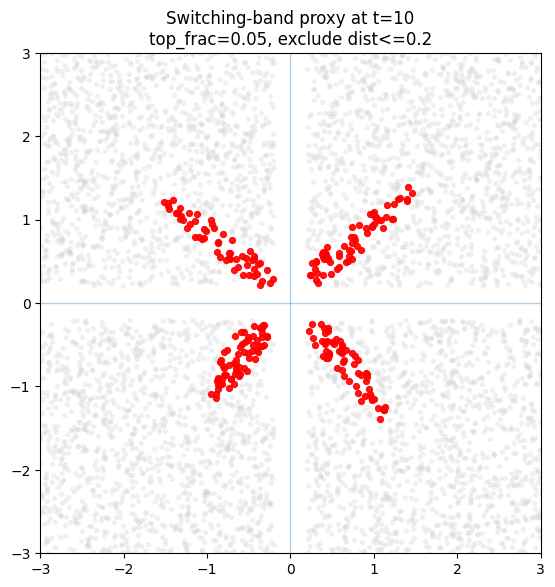}
    \caption{Estimated switching-band proxy for the transverse cross at $t=10$. 
Red points denote the top $5\%$ of the normalised CBD values among points satisfying $\mathrm{dist}>0.2$ from the support. 
The resulting four-lobe structure reflects the symmetry of the two-branch competition and localises the region where the normalised score direction is most unstable.
    }
    \label{fig:switch_band_cross}
\end{figure}

\paragraph{Intervention rule.}
Let $x_t$ denote the current reverse-time state, and let $x_{t-1}^{\mathrm{base}}$ be the standard DDPM update obtained from the pretrained denoiser without any modification.
In the control experiment, we add a small deterministic guidance term after this baseline reverse step:
\[
x_{t-1}^{\mathrm{ctrl}}
=
x_{t-1}^{\mathrm{base}}+\eta\, m_t(x_t)\, g(x_t),
\]
where $\eta>0$ is the control strength, $m_t(x_t)\in\{0,1\}$ is an activation mask, and
\[
g(x_t)=(-x_{t,1},\,0).
\]
Thus, the intervention shrinks the first coordinate toward zero while leaving the second coordinate unchanged.
Geometrically, this biases the reverse trajectory toward the vertical branch of the cross.

In the \emph{always-on} condition, we set $m_t(x_t)\equiv 1$ throughout the whole reverse process.
In the \emph{band-triggered} condition, we activate the same control only when the current state lies sufficiently close to the estimated switching-band proxy:
\[
m_t(x_t)=
\mathbf{1}\bigl\{\mathrm{dist}(x_t,\mathcal B)\le r_{\mathrm{band}}\bigr\},
\]
where $\mathcal B$ is the set of high normalised-CBD points used as the switching-band proxy and $r_{\mathrm{band}}>0$ is a fixed radius threshold.
Outside this region, no intervention is applied and the sampler follows the baseline DDPM dynamics.

We then compare three sampling strategies: baseline (no intervention), always-on intervention at all
reverse steps, and band-triggered intervention applied only when the trajectory enters the switching
band. Using $2000$ generated samples for each condition, the baseline sampler remains nearly balanced
between the two branches (vertical fraction $0.5225$, horizontal fraction $0.4775$). Always-on
intervention strongly biases generation toward the vertical branch (vertical fraction $0.9640$),
while band-triggered intervention also yields a substantial bias (vertical fraction $0.8635$).

The main difference lies in intervention frequency. Always-on control perturbs all $200$ reverse
steps, whereas the band-triggered rule intervenes only $65.5$ times on average (median $71$).
Hence, most of the branch-selection effect can be recovered by acting only near the switching band.
This supports the interpretation that the CBD ridge identifies the geometrically sensitive region
where small perturbations are most effective at altering the eventual branch choice.

\begin{figure}[h]
    \centering
    \includegraphics[width=\linewidth]{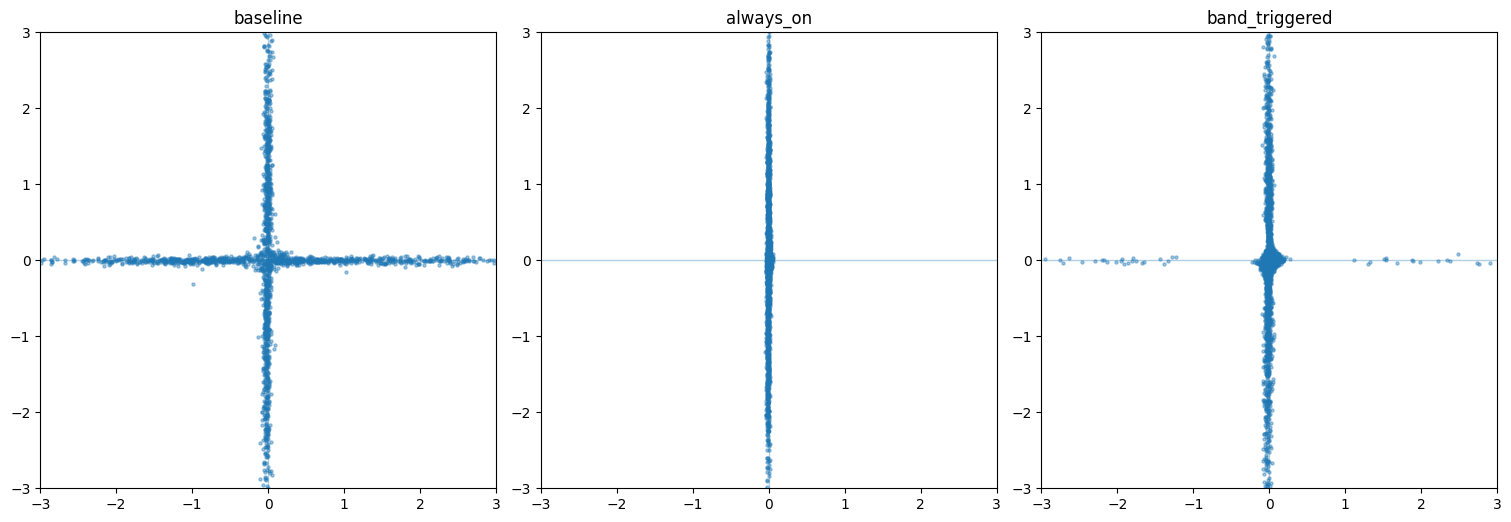}
   \caption{Branch-selection control on the transverse cross. 
From left to right: baseline sampling, always-on intervention, and intervention triggered only inside the estimated switching band. 
The intervention is implemented by adding, after each baseline DDPM reverse step, a small extra displacement in the direction $(-x_1,0)$, which pushes samples toward the vertical branch $x_1=0$. 
The baseline distribution is nearly symmetric between the horizontal and vertical branches, whereas always-on control drives samples almost entirely to the vertical branch. 
The switching-band-triggered rule also induces a clear vertical-branch bias, showing that a substantial part of the control effect is concentrated near the geometrically sensitive switching region.}
    \label{fig:band_triggered_control}
\end{figure}

\subsection{Future-count-based trigger windows and pseudo-online intervention}
\label{app:future-count-trigger-window}

To translate the geometric CBD signal into an actionable intervention rule, we introduce a
future-count statistic along a single baseline reverse trajectory.
The goal is not merely to detect a local CBD peak at the current time, but to identify states
that are upstream of future score-direction instability and are therefore intervention-relevant.

\paragraph{Baseline trajectory.}
We first generate a single baseline reverse trajectory
\(
\{x_t\}_{t=T}^{0}
\)
for a fixed class-conditional MNIST sample using a pretrained DDPM.
At each saved state \(x_t\), we evaluate a finite-difference approximation of CBD based on the
normalised predicted noise direction.
To suppress points that are either too close to or too far from the data manifold proxy,
we apply a band-pass mask determined from nearest-neighbor distances to a reference subset of
MNIST training images.
In the implementation, this band is chosen automatically by quantiles of the trajectory-wise
nearest-neighbor distance distribution.

\paragraph{Future count.}
For each state \(x_{t_i}\) on the baseline trajectory, we do not only inspect the CBD value at the
current time, but also ask whether the same state would exhibit large instability when evaluated at
future reverse times.
This yields a future-count statistic of the form
\[
\mathrm{FC}(i)
:=
\sum_{j>i}
\mathbf{1}\!\left\{
\widehat{\mathrm{CBD}}(x_{t_i}; t_j)
\ge \tau
\right\},
\]
possibly after residualisation and terminal-step suppression.
Intuitively, \(\mathrm{FC}(i)\) measures how strongly the current state is connected to future
instability events in the reverse process.

The future-count statistic was computed using an adaptive residual threshold
\[
\tau_{\mathrm{res}}=\mathrm{median}(r)+3.0 \times 1.4826 \,\mathrm{MAD}(r),
\]
where \(r\) denotes the detrended residual CBD values over non-terminal steps (\(t>10\)).
For each trajectory index \(i\), we counted how many of the next \(w=12\) future steps satisfied
\(r_j \ge \tau_{\mathrm{res}}\).
In the final trigger extraction stage, we further required \(\mathrm{FC}(i)\ge 2\) and
\(r_i\) to lie above the 80th percentile of positive residual values.

\paragraph{Prototype construction and intervention rule.}
Given a selected trigger window \(W\subset\{0,\dots,T-1\}\), we extract from the baseline reverse
trajectory the predicted clean images
\[
\mathcal P_W := \{\hat x_{0,t}^{\mathrm{base}} : t\in W\},
\]
where \(\hat x_{0,t}^{\mathrm{base}}\) denotes the DDPM estimate of the clean image at time \(t\)
along the baseline trajectory.
These baseline predictions are used as trigger prototypes.

During a new reverse sampling run, at each time \(t\) we first compute the current clean-image
prediction \(\hat x_{0,t}\).
Intervention is considered only when \(t\in W\).
We then compare \(\hat x_{0,t}\) with the prototype set \(\mathcal P_W\) in flattened image space
(after rescaling to \([0,1]\)).
More precisely, if
\[
d_t := \min_{p\in \mathcal P_W} \|\hat x_{0,t}-p\|_2
\]
is smaller than a prescribed radius \(r\), then we regard the current state as lying inside the
prototype region and activate intervention.
In the experiments reported here, we used \(r=2.0\).

When intervention is activated, we choose the nearest prototype
\[
p_t^\star := \arg\min_{p\in \mathcal P_W} \|\hat x_{0,t}-p\|_2
\]
and push the current prediction away from that prototype:
\[
\tilde x_{0,t}
=
\hat x_{0,t}
+
\lambda
\frac{\hat x_{0,t}-p_t^\star}{\|\hat x_{0,t}-p_t^\star\|_2+\varepsilon},
\]
where \(\lambda>0\) is the intervention strength and \(\varepsilon\) is a small numerical constant.
Thus, the intervention is not a global perturbation but a localised repulsive update in
\(\hat x_0\)-space relative to the baseline trigger prototypes.
After this modification, the DDPM noise prediction is recomputed consistently from
\(\tilde x_{0,t}\), and the reverse step is then taken using this modified denoising target.
In our default MNIST experiment, we used \(\lambda=0.5\).

This construction makes the intervention both temporally localised and geometry-aware.
Temporal localisation comes from restricting intervention to the trigger window \(W\), while
geometric localisation comes from requiring the current prediction to lie near the baseline
prototype set.
Therefore, the sampler is perturbed only when it revisits a region that is empirically associated
with downstream branch instability, rather than being modified uniformly over the entire reverse
process.

\paragraph{Intervention experiments within various time windows}
First, from the future-count plot (Figure \ref{fig:future-count-profile}), one can visually estimate that the singular time windows are roughly distributed over 
400–300, 300–250, 200–110, and 70–40. We then performed intervention experiments separately within each of these time windows. 
Compared with the baseline sample (Figure \ref{fig:baseline-mnist}), the intervention applied in the 200–110 time window led to a particularly dramatic change in the generated sample (Figure \ref{fig:110-200-intervention}). By contrast, the interventions in the 400–300 and 300–250 windows appear much less effective, judging from the resulting outputs (Figure \ref{fig:250-350-intervention}, \ref{fig:300-400-intervention}). We further conducted intervention experiments using the narrower time windows 50–75 and 85–110. As expected, the intervention in the non-singular window 85–110 (Figure \ref{fig:85-110_final}) appears less effective than that in 50–75 (Figure \ref{fig:50-75_final}). We also observe that, in this model, fine-grained details are generated during the final stage of the reverse process.

\paragraph{Interpretation.}
The role of future count is operational rather than purely descriptive.
CBD itself is a local instability signal, whereas future count aggregates whether a currently visited
state is likely to participate in downstream instability at later reverse times.
From this perspective, the trigger window should be viewed as an empirical proxy for an
intervention-relevant critical interval.
The empirical purpose of the appendix experiment is therefore to show that
window-restricted intervention is substantially more effective than intervention outside the
detected window or over broad non-selective time ranges.
This supports the view that phase commitment in diffusion sampling is concentrated in a narrow
temporal region rather than being uniformly distributed across the entire reverse process.
\begin{figure}[h]
    \centering
    \includegraphics[width=.35\linewidth]{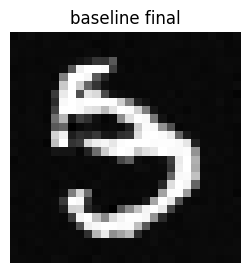}
    \caption{
    baseline
    }
    \label{fig:baseline-mnist}
\end{figure}

\begin{figure}[h]
    \centering
    \includegraphics[width=\linewidth]{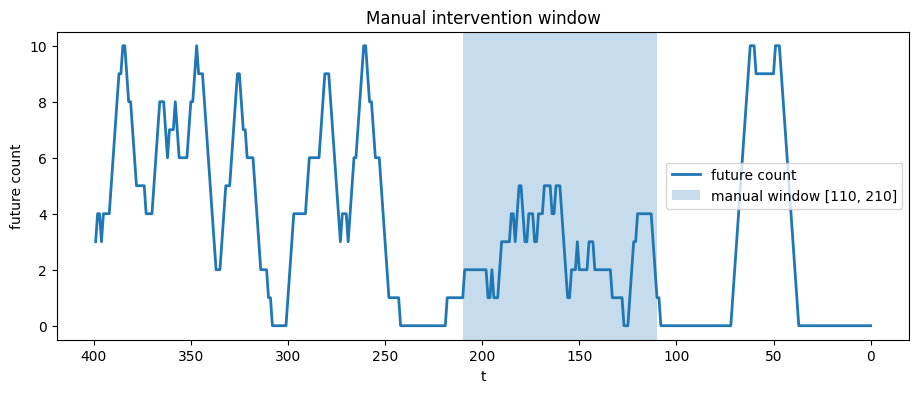}
    \caption{
    Future-count profile along a baseline reverse trajectory.
    High-value contiguous segments are interpreted as candidate trigger windows.
    These windows indicate states that are not only locally unstable, but are also upstream of
    future score-direction switching events.
    }
    \label{fig:future-count-profile}
\end{figure}

\begin{figure}[h]
    \centering
    \includegraphics[width=.55\linewidth]{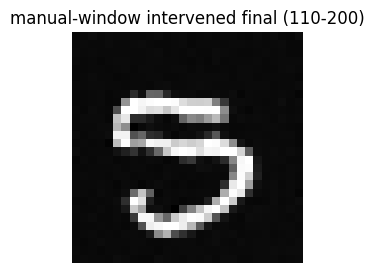}
    \caption{
    We intervene within the time window 110-200
    }
    \label{fig:110-200-intervention}
\end{figure}
\begin{figure}[h]
    \centering
    \includegraphics[width=.55\linewidth]{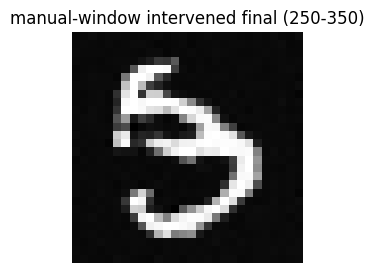}
    \caption{
    We intervene within the time window 250-350.
    }
    \label{fig:250-350-intervention}
\end{figure}
\begin{figure}[h]
    \centering
    \includegraphics[width=.55\linewidth]{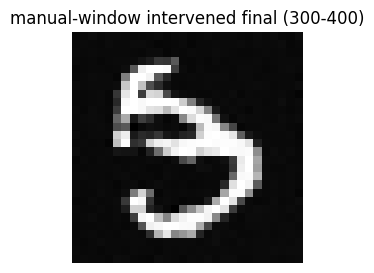}
    \caption{
    We intervene within the time window 300-400.
    }
    \label{fig:300-400-intervention}
\end{figure}

\begin{figure}[h]
    \centering
    \includegraphics[width=\linewidth]{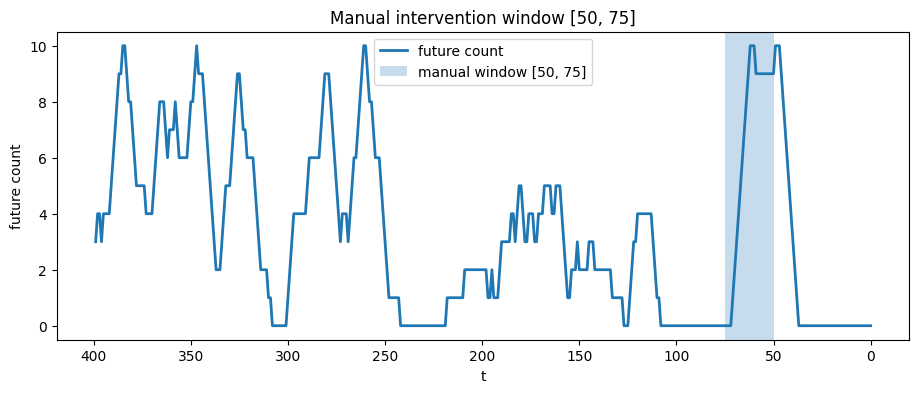}
    \caption{
    we intervene within the time window 50-75
    }
    \label{fig:50-75_CW}
\end{figure}

\begin{figure}[h]
    \centering
    \includegraphics[width=\linewidth]{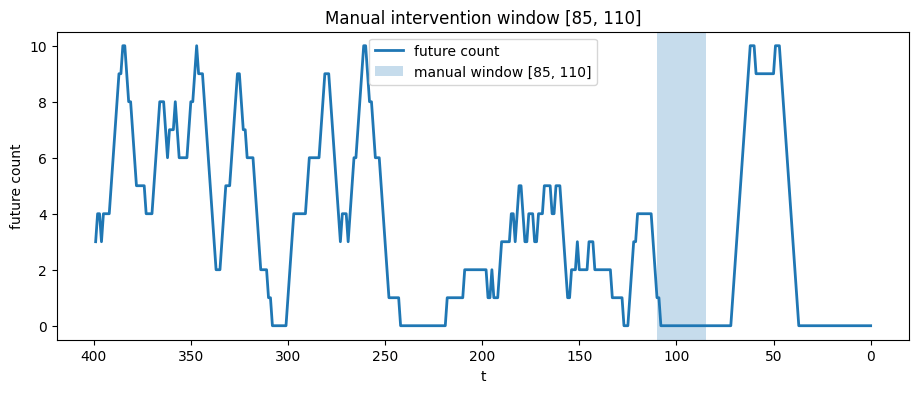}
    \caption{
    we intervene within the time window 85-110
    }
    \label{fig:85-110_CW}
\end{figure}

\begin{figure}[h]
\centering
\begin{minipage}[b]{0.49\columnwidth}
    \centering
    \includegraphics[width=0.9\columnwidth]{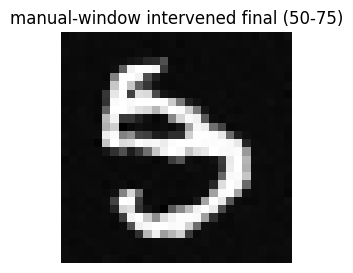}
    \caption{}
    \label{fig:50-75_final}
\end{minipage}
\begin{minipage}[b]{0.49\columnwidth}
    \centering
    \includegraphics[width=0.9\columnwidth]{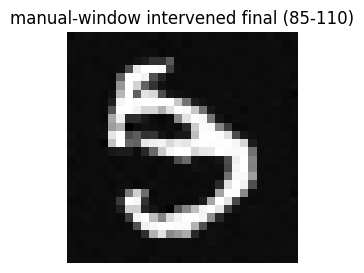}
    \caption{}
    \label{fig:85-110_final}
\end{minipage}
\end{figure}

\FloatBarrier
\subsection{Pseudo-online sensitivity diagnostic and intervention on CIFAR10}\label{appendix:CIFAR10}
\label{app:cifar10_pseudo_online}

To examine whether the proposed direction-instability perspective remains informative beyond toy models, we conducted a pseudo-online intervention experiment on a pretrained CIFAR10 DDPM.
We used the public checkpoint \texttt{google/ddpm-cifar10-32} and sampled a single baseline reverse trajectory with \(250\) inference steps.
Along this fixed trajectory, we evaluated a score directional instability probe (CBD-like estimator) at each saved state and then used selected trajectory windows as candidate intervention regions.

More precisely, let \(x_i\) denote the \(i\)-th saved state on the baseline reverse trajectory, and let \(t_i\) be its corresponding actual diffusion time step.
For each \(x_i\), we computed a trajectory-adapted directional instability quantity
\[
\mathrm{TAD}(x_i,t)
\approx
\frac{\|u_t(x_i+h z)-u_t(x_i-h z)\|}{2h},
\qquad
u_t(x)=\frac{v_t(x)}{\|v_t(x)\|},
\]
where \(v_t(x)\) is the predicted noise direction, \(z\) is a shared random perturbation direction, and \(h\) is chosen proportionally to the noise scale.
To mimic an online diagnostic, we compared the actual-time instability \(\mathrm{TAD}(x_i,t_i)\) with final-stage probe instabilities \(\mathrm{TAD}(x_i,t_{\mathrm{probe}})\) for \(t_{\mathrm{probe}}\in\{1,3,5\}\), and formed the ratio
\[
R_i(t_{\mathrm{probe}})
=
\frac{\mathrm{TAD}(x_i,t_i)}{\mathrm{TAD}(x_i,t_{\mathrm{probe}})}.
\]
We then smoothed this ratio along the trajectory and used it only as a heuristic guide for selecting candidate windows.

Intervention was performed by re-noising a saved state \(x_t\) inside a chosen window:
\[
x_t^{\mathrm{int}} = x_t + \gamma \sigma_t \xi,
\qquad
\xi \sim \mathcal N(0,I),
\]
with intervention strength \(\gamma = 0.35\).
From \(x_t^{\mathrm{int}}\), we continued the reverse process with the same scheduler to obtain a final sample \(x_0^{\mathrm{int}}\).
To quantify the effect of intervention, we measured the deviation from the baseline final sample \(x_0^{\mathrm{base}}\) using
\[
\|x_0^{\mathrm{int}}-x_0^{\mathrm{base}}\|_2
\quad\text{and}\quad
\frac{1}{d}\|x_0^{\mathrm{int}}-x_0^{\mathrm{base}}\|_1.
\]
For each candidate window, we repeated this procedure \(24\) times by sampling intervention indices within the window with replacement.

\paragraph{CIFAR10 pseudo-online diagnostic: localisation of an intervention-sensitive window.}
Figure~\ref{fig:cifar10_tad_ratio} shows the normalised TAD-ratio diagnostic along a single baseline reverse trajectory of the pretrained CIFAR10 DDPM (the seed $= 1000$).
Although the ratio does not exhibit a perfectly isolated sharp spike, it displays a noticeable early-time depression over approximately the first \(75\) saved trajectory steps.
This qualitative signal is consistent with the intervention experiment: among the candidate windows we tested, the early window \([0,75]\) is precisely the one that produces by far the largest final-time deviation from the baseline sample.
More specifically, intervention in the early window yields mean final \(\ell_2\) deviation \(43.32\) and median \(40.00\), with mean per-pixel \(\ell_1\) deviation \(0.657\).
By contrast, a late window \([180,249]\) produces a much smaller effect, with mean final \(\ell_2\) deviation \(6.64\), median \(6.36\), and mean per-pixel \(\ell_1\) deviation \(0.0756\).
An intermediate manually chosen window \([175,225]\) remains somewhat sensitive but is still clearly weaker than the early one, yielding mean final \(\ell_2\) deviation \(9.18\), median \(9.02\), mean per-pixel \(\ell_1\) deviation \(0.101\), and median \(0.101\).
Therefore, even though the normalised TAD ratio is not yet sharp enough to define a unique critical time step in a fully automatic manner, it successfully identifies the most intervention-sensitive \emph{temporal window} in this CIFAR10 experiment.
In this sense, the early dip should be interpreted not as a precise pointwise detector, but as a practically useful ranking signal for selecting candidate critical windows.

\begin{figure}[h]
    \centering
    \includegraphics[width=\linewidth]{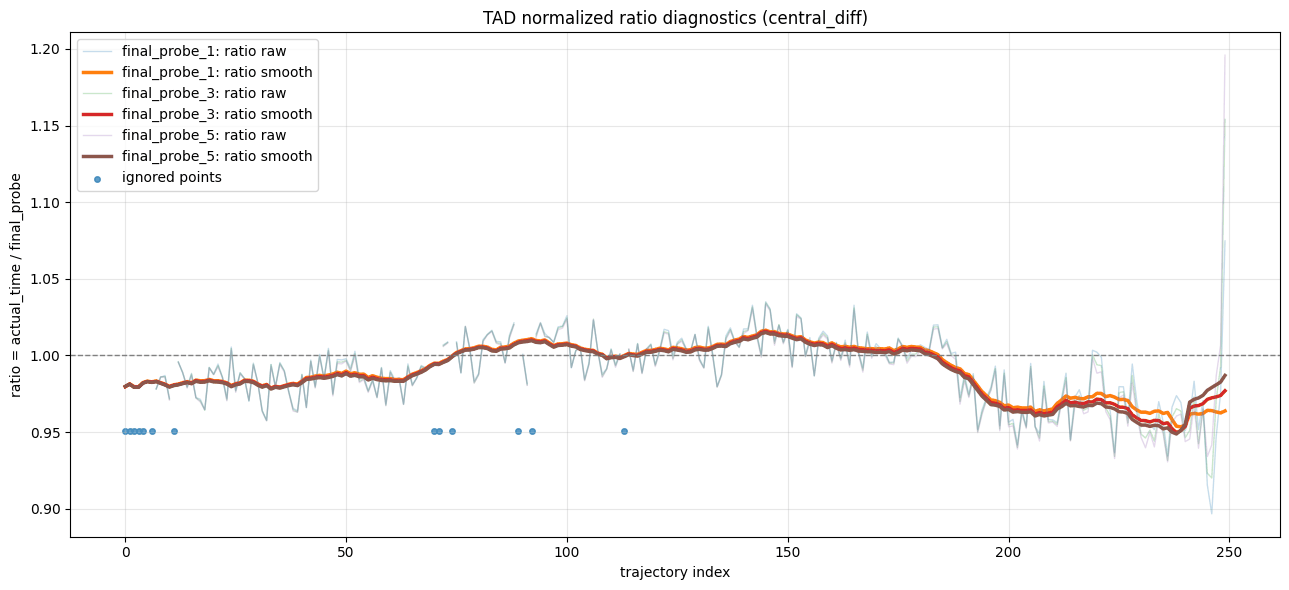}
    \caption{
    Normalised TAD-ratio diagnostic along a baseline reverse trajectory on CIFAR10.
    The ratio does not form a perfectly sharp isolated spike, but it exhibits a clear early-time dip over roughly the first \(75\) saved trajectory steps.
    This dip is consistent with the intervention study: the corresponding early window is the most sensitive one, producing substantially larger final-sample deviation than the late and manual comparison windows.
    }
    \label{fig:cifar10_tad_ratio}
\end{figure}

\paragraph{Additional window-wise sensitivity analysis on the CIFAR10 trajectory.}
To further examine whether the pseudo-online TAD signal tracks intervention sensitivity beyond the three windows reported above, we performed an additional sliding-window analysis on another baseline reverse trajectory. We partitioned the saved trajectory into non-overlapping windows of width $15$ in trajectory index, and for each window repeated the same re-noising intervention experiment $24$ times, sampling intervention indices within the window with replacement. For each trial, we measured both the final $\ell_2$ deviation from the baseline sample and the perceptual LPIPS distance.

The resulting window-wise profile shows a clear global trend (Figure \ref{fig:l2lpips_42}): intervention sensitivity is strongest in the early part of the reverse process and gradually decreases toward the final stage. Concretely, the earliest window $[0,14]$ yields mean final $\ell_2$ deviation $35.68$ and mean LPIPS $0.165$, while the last window $[240,249]$ yields much smaller values, mean final $\ell_2$ deviation $3.33$ and mean LPIPS $0.0053$. Intermediate windows interpolate between these two regimes, with the effect size decreasing overall as the reverse trajectory approaches the terminal denoising stage. This supports the interpretation that the reverse process contains an extended early-to-mid interval in which perturbations are much more consequential for the eventual sample outcome.

\begin{figure}[h]
    \centering
    \includegraphics[width=\linewidth]{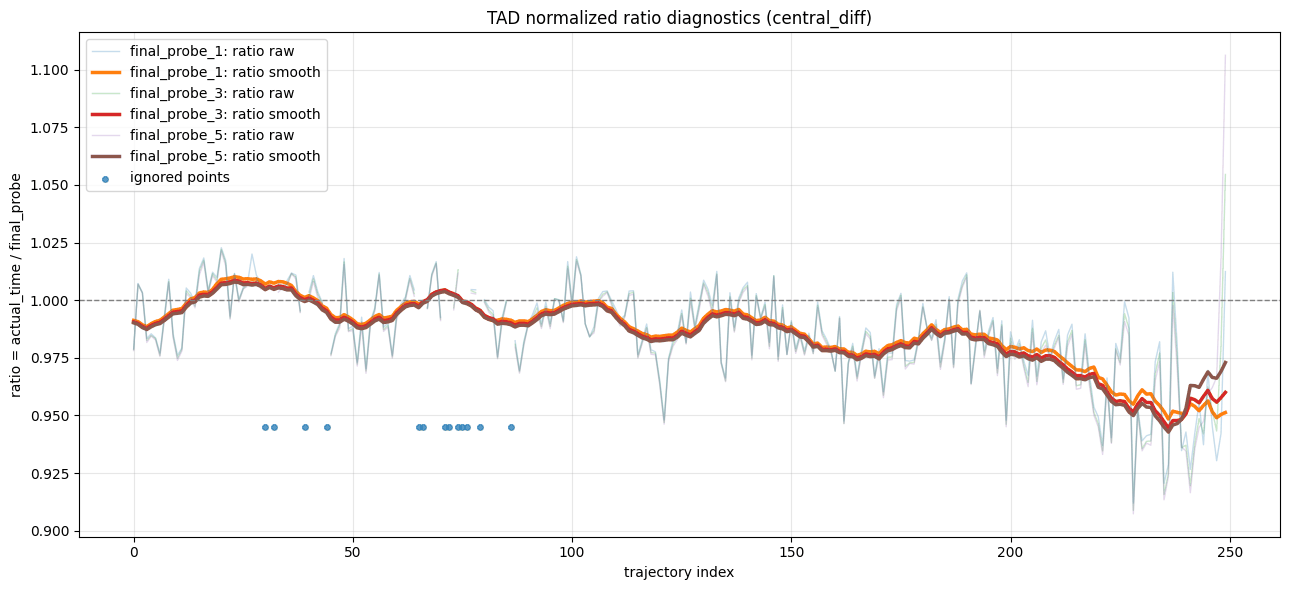}
    \caption{
   TAD for another trajectory (seed = 42)
    }
    \label{fig:cifar10_tad_ratio_42}
\end{figure}

\begin{figure}[h]
    \centering
    \includegraphics[width=\linewidth]{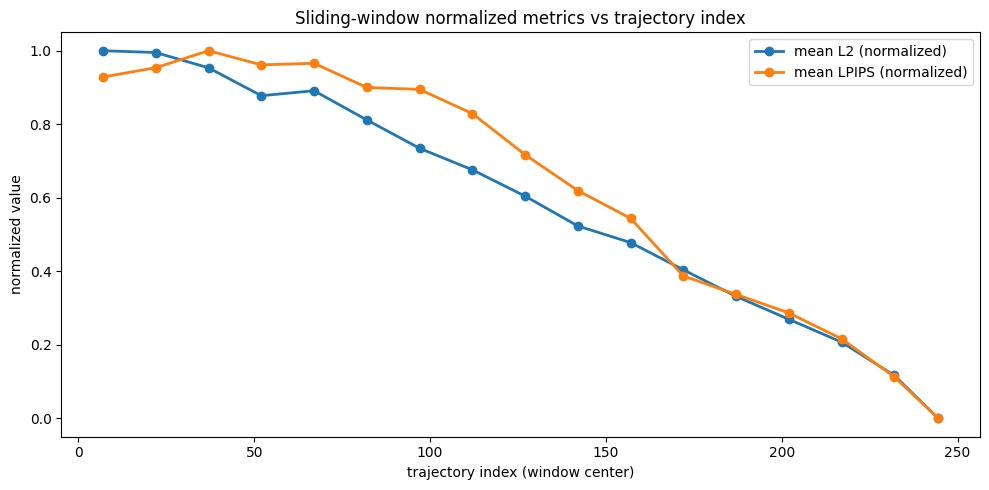}
    \caption{
   Normalised mean $\ell_2$ distance and mean LPIPS distance over trajectory indices.
    }
    \label{fig:l2lpips_42}
\end{figure}

\paragraph{Relation between the TAD-ratio profile and downstream effect size.}
We then compared the smoothed TAD-ratio curve (Figure \ref{fig:cifar10_tad_ratio_42}) with the window-averaged intervention metrics. More precisely, we averaged the smoothed TAD-ratio over each width-$15$ window and compared it against the corresponding window means of final $\ell_2$ deviation and LPIPS. At this coarse scale, the agreement is strong: the Pearson correlation between the window-averaged TAD-ratio and normalised mean final $\ell_2$ is $0.913$, and the correlation with normalised mean LPIPS is $0.919$. Thus, even though the normalised TAD-ratio in Figure~\ref{fig:cifar10_tad_ratio_42} does not yet isolate a unique sharply peaked critical time step, its smoothed trajectory-level profile remains strongly aligned with the actual intervention sensitivity measured by both pixel-space and perceptual final-sample deviation.

Taken together, these additional results sharpen the interpretation of the CIFAR10 experiment. The pseudo-online TAD diagnostic should not be viewed as a pointwise exact detector, but rather as a coarse sensitivity profile that ranks trajectory regions by how strongly perturbations there propagate to the final sample. In this sense, the diagnostic is already useful operationally: it identifies temporally extended intervention-relevant regions, and its agreement with both $\ell_2$ and LPIPS suggests that the detected sensitivity is not merely a pixel-space artifact.

\subsection{Robustness across random seeds on CIFAR-10}
\label{app:cifar_seed_robustness}

To assess whether the pseudo-online signal is stable beyond a single baseline trajectory, we repeated the CIFAR-10 window-wise correlation analysis over five random seeds. For each seed, we generated one baseline reverse trajectory with 250 inference steps, computed the normalised TAD-ratio diagnostic along the trajectory, and evaluated window-wise intervention sensitivity using the same re-noising protocol as in Appendix~\ref{app:cifar10_pseudo_online}. As in the main CIFAR-10 experiment, each candidate window had width 15 in trajectory index, and intervention sensitivity was measured by the mean final deviation from the baseline sample over repeated interventions within the window, using both $\ell_2$ distance and LPIPS.

Across all five seeds, we observed the same qualitative pattern: the \emph{raw} normalised TAD-ratio was consistently \emph{negatively} correlated with the window-wise intervention effect. In other words, the intervention-sensitive region is expressed not as a peak of the raw ratio, but rather as a \emph{depression}. Accordingly, after inverting the normalised ratio, the resulting diagnostic showed stable positive correlation with both mean final $\ell_2$ deviation and mean LPIPS.

Quantitatively, averaged over five seeds, the inverted normalised ratio achieved Pearson correlation $0.6939 \pm 0.0212$ with mean final $\ell_2$ and $0.7005 \pm 0.0561$ with mean LPIPS. The corresponding Spearman correlations were $0.9088 \pm 0.0631$ for mean final $\ell_2$ and $0.8873 \pm 0.0787$ for mean LPIPS (Table \ref{tab:cifar_seed_robustness}, Table \ref{tab:cifar_seed_corr_inv}). Thus, the signal is especially strong at the level of \emph{rank correlation}, suggesting that it is reliable as a \emph{window-ranking diagnostic} for intervention-sensitive regions.

Overall, these multi-seed results support the interpretation that, on CIFAR-10, the practically relevant critical interval is robustly associated with a dip of the normalised TAD-ratio. This strengthens the claim that the pseudo-online diagnostic captures a meaningful intervention-sensitive temporal structure rather than an artifact of a single trajectory realisation.

\begin{table}[h]
\centering
\caption{
Multi-seed robustness of the CIFAR-10 pseudo-online diagnostic over five random seeds.
Here $R$ denotes the normalised TAD-ratio and $1-R$ its inverted version.
Across all seeds, the raw ratio is consistently negatively correlated with window-wise intervention effect, while the inverted ratio is positively correlated with both mean final $\ell_2$ deviation and mean LPIPS.
The particularly high Spearman correlations indicate that the signal is especially effective as a coarse ranking criterion for intervention-sensitive windows.
}
\label{tab:cifar_seed_robustness}
\begin{tabular}{lccc}
\toprule
Metric & Mean & Std. & Median \\
\midrule
Pearson$(R,\ \mathrm{L2})$             & $-0.6939$ & $0.0212$ & $-0.6925$ \\
Pearson$(R,\ \mathrm{LPIPS})$          & $-0.7005$ & $0.0561$ & $-0.6714$ \\
Pearson$(1-R,\ \mathrm{L2})$           & $ 0.6939$ & $0.0212$ & $ 0.6925$ \\
Pearson$(1-R,\ \mathrm{LPIPS})$        & $ 0.7005$ & $0.0561$ & $ 0.6714$ \\
\midrule
Spearman$(R,\ \mathrm{L2})$            & $-0.9088$ & $0.0631$ & $-0.8995$ \\
Spearman$(R,\ \mathrm{LPIPS})$         & $-0.8873$ & $0.0787$ & $-0.9314$ \\
Spearman$(1-R,\ \mathrm{L2})$          & $ 0.9088$ & $0.0631$ & $ 0.8995$ \\
Spearman$(1-R,\ \mathrm{LPIPS})$       & $ 0.8873$ & $0.0787$ & $ 0.9314$ \\
\bottomrule
\end{tabular}
\end{table}

\begin{table}[t]
\centering
\caption{Seed-wise correlation between the inverted TAD-ratio diagnostic and downstream intervention effect size on CIFAR-10.}
\label{tab:cifar_seed_corr_inv}
\begin{tabular}{ccccc}
\toprule
Seed & Pearson ($R$, $\ell_2$) & Pearson ($R$, LPIPS) & Spearman ($\rho$, $\ell_2$) & Spearman ($\rho$, LPIPS) \\
\midrule
0 & 0.661 & 0.651 & 0.833 & 0.777 \\
1 & 0.716 & 0.743 & 0.900 & 0.831 \\
2 & 0.692 & 0.777 & 0.973 & 0.931 \\
3 & 0.691 & 0.671 & 0.973 & 0.949 \\
4 & 0.708 & 0.661 & 0.865 & 0.949 \\
\midrule
Mean   & 0.694 & 0.700 & 0.909 & 0.887 \\
Std.   & 0.021 & 0.056 & 0.062 & 0.079 \\
Median & 0.692 & 0.671 & 0.900 & 0.931 \\
\bottomrule
\end{tabular}
\end{table}

\subsection{Pseudo-online TAD diagnosis and windowed intervention on SD3.5 Medium}
\label{app:sd35_lpips_shiba_psd}

We additionally tested the pseudo-online TAD diagnostic on the foundation text-to-image model \emph{Stable Diffusion 3.5 Medium} under a fixed prompt and seed.
As in the CIFAR-10 study, we first recorded a baseline reverse trajectory in latent space and computed a finite-difference proxy of CBD along the trajectory.
We then defined the pseudo-online TAD ratio by normalising the finite-difference CBD at each trajectory index by its value at the final finite step.
This diagnostic exhibited a clear late rise: the smoothed TAD ratio remained low over most of the trajectory and increased sharply only near the terminal denoising phase.
A sliding-window selection based on this signal identified the late band \((23\text{--}26)\) as the highest-TAD window, while \((15\text{--}18)\) and \((0\text{--}3)\) were used as mid- and low-TAD controls, respectively.

We then performed latent-space interventions restricted to these windows.
A naive comparison using pixel-space distances showed a qualitatively different behaviour from the CIFAR-10 case: the earliest window produced the largest final-image deviation in $\ell_2$ and $\ell_1$ distance, whereas the late high-TAD window did not maximize these raw pixel metrics.
This indicates that, in SD3.5 Medium, the earliest denoising steps have very strong \emph{compositional sensitivity}: perturbations in this regime strongly affect global layout, object placement, and coarse scene geometry, which dominate pixel-space deviations.

To better separate compositional sensitivity from perceptual sensitivity, we additionally evaluated the same interventions with LPIPS.
Here the ordering reversed in a meaningful way.
For the three aggregated windows, the LPIPS scores were
\[
\mathrm{critical}\ (23\text{--}26): 0.708,\qquad
\mathrm{mid}\ (15\text{--}18): 0.698,\qquad
\mathrm{low}\ (0\text{--}3): 0.669,
\]
so that the LPIPS ranking matched the pseudo-online TAD ranking:
\[
\mathrm{critical} > \mathrm{mid} > \mathrm{low}.
\]
Thus, while the earliest window maximizes raw pixel displacement, the late high-TAD band is the one that produces the strongest \emph{perceptual} change.

To further localise this effect, we ran a quick scan over several two-step windows.
In pixel-space $\ell_2$, the strongest effect again appeared at the very beginning of sampling, e.g.\ \((0,1)\), and then decayed monotonically over early windows.
By contrast, in LPIPS the late window \((25,26)\) achieved the highest score among the tested two-step windows, exceeding even the earliest window \((0,1)\).
This supports the view that, in large text-to-image models, there are at least two distinct notions of intervention sensitivity:
an \emph{early compositional sensitivity}, which dominates pixel-space deviations, and a \emph{late directional/perceptual sensitivity}, which is more faithfully tracked by the pseudo-online TAD signal.

Overall, these SD3.5 Medium experiments suggest that the pseudo-online TAD diagnostic should not be interpreted as a detector of the window that maximizes final pixel-space displacement.
Rather, for large text-to-image models, it appears to identify a late regime of heightened directional instability that is more closely aligned with perceptual sensitivity than with raw $\ell_2$ change.
This distinction is consistent with the qualitative observation that early interventions mainly alter global scene composition, whereas late interventions more selectively affect perceptual appearance and local structure.

\begin{table}[t]
\centering
\caption{
Selected two-step window interventions on SD3.5 Medium.
The late window $(25,26)$ has the largest pseudo-online TAD score and also the largest LPIPS, while the earliest window $(0,1)$ maximizes final pixel-space $\ell_2$ deviation.
}
\label{tab:sd35_quickscan_lpips_timestep}
\begin{tabular}{lccccc}
\toprule
Window & Traj. idx. & Time steps & Mean TAD ratio (smooth) & Final $\ell_2$ & LPIPS \\
\midrule
$w_{25,26}$ & $(25,26)$ & $(199,110)$ & 0.503 & 40.15  & \textbf{0.196} \\
$w_{0,1}$   & $(0,1)$   & $(1000,987)$ & 0.051 & \textbf{207.72} & 0.187 \\
$w_{23,24}$ & $(23,24)$ & $(347,278)$ & 0.305 & 34.31  & 0.117 \\
$w_{2,3}$   & $(2,3)$   & $(974,960)$ & 0.053 & 137.96 & 0.113 \\
$w_{4,5}$   & $(4,5)$   & $(945,929)$ & 0.054 & 90.99  & 0.086 \\
$w_{3,4}$   & $(3,4)$   & $(960,945)$ & 0.053 & 73.05  & 0.084 \\
$w_{9,10}$  & $(9,10)$  & $(857,836)$ & 0.062 & 43.32  & 0.075 \\
$w_{10,11}$ & $(10,11)$ & $(836,814)$ & 0.064 & 39.49  & 0.070 \\
\bottomrule
\end{tabular}
\end{table}
\begin{figure}[h]
    \centering
    \includegraphics[width=\linewidth]{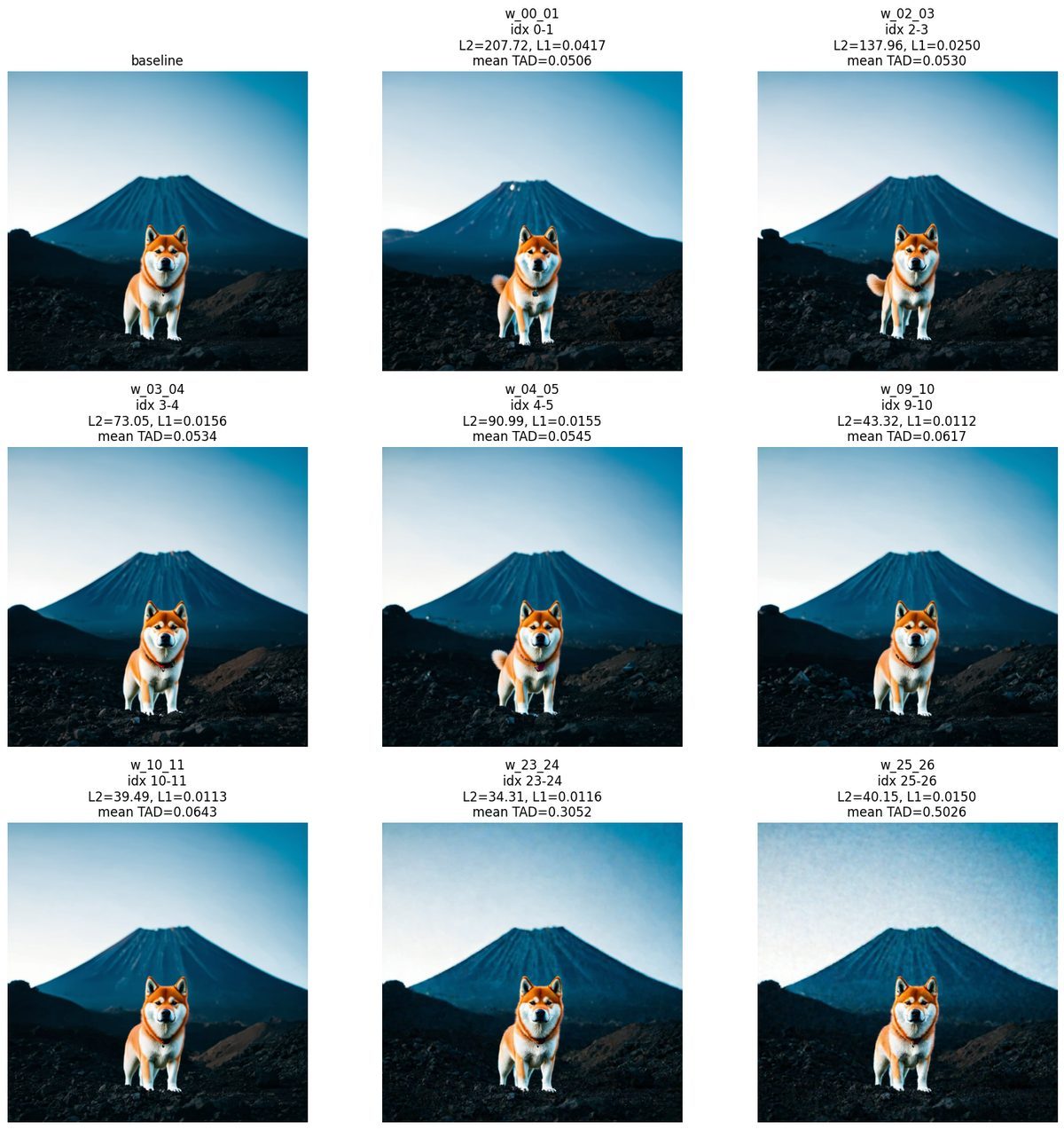}
   \caption{
Qualitative intervention study on SD3.5 for the prompt ``a cinematic photo of a shiba inu standing in front of a volcano, highly detailed.'' 
The leftmost panel is the baseline sample, and each remaining panel shows the final image obtained by applying noise only in the indicated trajectory window. 
Titles report the corresponding window, trajectory indices, final $\ell_2$ / $\ell_1$ distance from the baseline, and the window-averaged TAD value. 
Windows with small mean TAD produce much larger output deviations, while windows with larger mean TAD are comparatively insensitive, supporting the interpretation of TAD as a practical ranking signal for intervention sensitivity.
    }
    \label{fig:shiba}
\end{figure}

\paragraph{Robustness study on SD3.5 Medium.}
We further evaluated the pseudo-online TAD diagnostic on \texttt{stabilityai/stable-diffusion-3.5-medium} using a small but diverse prompt set designed to include both relatively stable prompts and prompts with concept fusion or ambiguous morphology. Concretely, we used six prompts:
\emph{(i)} ``a house that looks like a ship, realistic architectural photography,''
\emph{(ii)} ``a cinematic photo of a shiba inu standing in front of a volcano, highly detailed,''
\emph{(iii)} ``a lighthouse fused with a giant seashell on the coast, cinematic realism,''
\emph{(iv)} ``a mountain shaped like a sleeping lion, realistic landscape photo,''
\emph{(v)} ``a jellyfish that resembles a chandelier, dark background, cinematic photo,''
and \emph{(vi)} ``a realistic photo of a red vintage car parked on a rainy street at night.''
For each prompt, we used five random seeds $\{1234,1235,1236,1237,1238\}$, yielding $30$ runs in total.

All runs were performed at resolution $768 \times 768$ with $16$ denoising steps, classifier-free guidance scale $7.0$, and maximum sequence length $256$. We enabled skip-layer guidance with layers $[7,8,9]$, scale $2.8$, start ratio $0.01$, and stop ratio $0.2$. Along each baseline trajectory, we recorded the latent state at every denoising step via the step-end callback. We then computed a finite-difference CBD estimate on the normalised guided velocity field, using $3$ random probe directions, perturbation scale $\delta = \max(5\times 10^{-3}\sigma_t, 10^{-4})$, and evaluating every $2$ steps; missing steps were filled by linear interpolation. The resulting fdCBD curve was smoothed by a moving average with window size $3$, and converted into the normalised TAD ratio used in the main text.

To assess downstream intervention sensitivity, we scanned contiguous time windows of length $2$ with stride $2$, excluding the last denoising step. For each active window, we injected Gaussian noise into the latent variable only within that window, with amplitude proportional to the current noise scale $\sigma_t$. More precisely, if $W$ is the active window and $\{\sigma_t\}_{t\in W}$ are the corresponding scheduler noise levels, we used a budget-normalised intervention strength
\[
\gamma_W
=
\frac{\gamma_0}{\left(\sum_{t\in W}\sigma_t^p\right)^{1/p}},
\qquad
\gamma_0 = 0.20,\quad p=2,
\]
and perturbed the latent at each step $t\in W$ by $\gamma_W \sigma_t \varepsilon_t$ with fresh Gaussian noise $\varepsilon_t$. For each window, we measured the final deviation from the baseline sample by the pixel-space $\ell_2$ distance, mean $\ell_1$ distance, and LPIPS. We then compared the window-averaged TAD ratio against these downstream effect sizes within each run using both Pearson and Spearman correlations, and finally aggregated these statistics across all $30$ runs.

\begin{table}[t]
\centering
\caption{Aggregated correlation statistics between the window-averaged TAD ratio and downstream intervention effect on SD3.5 Medium, over 6 prompts and 5 seeds ($30$ runs total).}
\label{tab:sd35_aggregated_corr}
\begin{tabular}{lcccc}
\toprule
Metric & Mean & Std & Median & $n$ \\
\midrule
Pearson$(\mathrm{TAD}, \ell_2)$      & $-0.470$ & $0.117$ & $-0.444$ & $30$ \\
Pearson$(\mathrm{TAD}, \mathrm{LPIPS})$ & $-0.230$ & $0.430$ & $-0.359$ & $30$ \\
Spearman$(\mathrm{TAD}, \ell_2)$     & $-0.887$ & $0.109$ & $-0.929$ & $30$ \\
Spearman$(\mathrm{TAD}, \mathrm{LPIPS})$ & $-0.577$ & $0.529$ & $-0.875$ & $30$ \\
\bottomrule
\end{tabular}
\end{table}
Table~\ref{tab:sd35_aggregated_corr} shows that the TAD ratio has a consistently strong negative rank correlation with downstream $\ell_2$ intervention effect across SD3.5 runs, while the LPIPS correlation is more variable. This supports the interpretation of TAD as a robust sensitivity ranking over trajectory windows rather than a precise linear predictor of final intervention magnitude.

\begin{table}[t]
\centering
\caption{Prompt-wise mean correlation statistics on SD3.5 Medium, averaged over 5 seeds for each prompt. Here, $P$ and $S$ denote Pearson and Spearman correlation, respectively.}
\label{tab:sd35_promptwise_corr}
\small
\setlength{\tabcolsep}{4pt}
\begin{tabular}{clcccc}
\toprule
Idx & Prompt & P$(T,\ell_2)$ & P$(T,\mathrm{LPIPS})$ & S$(T,\ell_2)$ & S$(T,\mathrm{LPIPS})$ \\
\midrule
0 & House $\leftrightarrow$ ship
  & $-0.446$ & $-0.247$ & $-0.864$ & $-0.671$ \\
1 & Shiba inu + volcano
  & $-0.396$ & $-0.299$ & $-0.921$ & $-0.714$ \\
2 & Lighthouse $\leftrightarrow$ seashell
  & $-0.489$ & $-0.391$ & $-0.943$ & $-0.786$ \\
3 & Mountain $\leftrightarrow$ lion
  & $-0.366$ & $\phantom{-}0.620$ & $-0.729$ & $\phantom{-}0.421$ \\
4 & Jellyfish $\leftrightarrow$ chandelier
  & $-0.623$ & $-0.587$ & $-0.993$ & $-0.993$ \\
5 & Red vintage car
  & $-0.499$ & $-0.477$ & $-0.871$ & $-0.721$ \\
\bottomrule
\end{tabular}
\end{table}
\section{SD3.5, Noise-injecting intervention, Shiba inu and Fuji}\label{app:sd35_intervention_shiba_inu}

We applied noise-injecting interventions over fixed-width time windows (0–1,1–2,…,25–26) and generated the corresponding output images (Figure \ref{fig:sliding_shiba}). For this time we strengthen the intervention strength. Figure \ref{fig:sliding_shiba_cbd_26} shows the CBD curve along the baseline trajectory in the low-noise regime. It can be visually observed that the intervention effect becomes weaker after the CBD peak.

\begin{figure}[t]
    \centering
    \includegraphics[width=\linewidth]{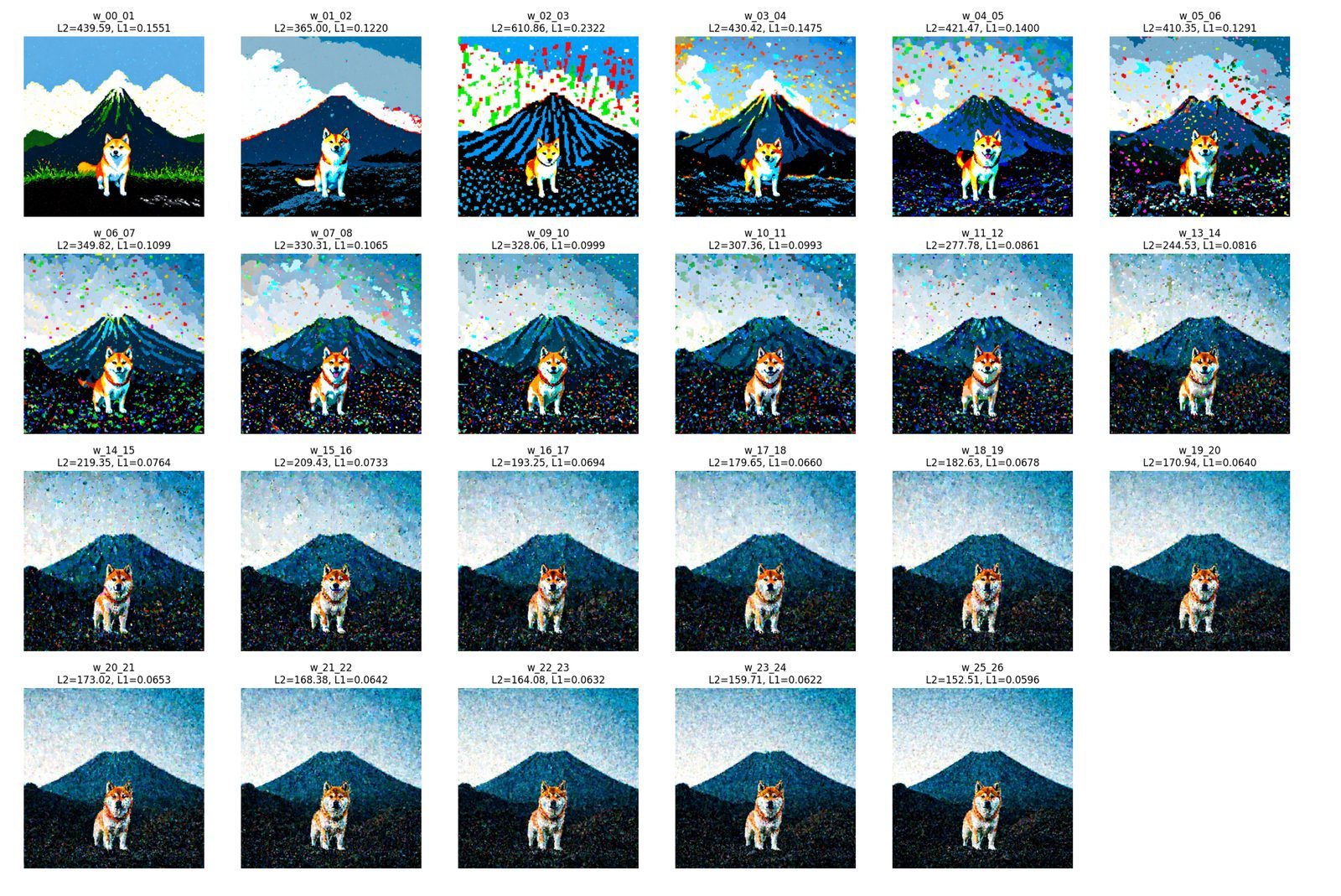}
   \caption{
   $w_{mn}$ denotes the case where intervention is performed over the time window (m,n). L2 and L1 indicate the L2 and L1 distances, respectively, between the baseline generation result and the intervention result. If we intervene within each time window in 0-11, cartoon-like feature appears. The intervention result switches from cartoon-like features to cinematic features in the time window 12–15. This switching phenomenon can be described by the ridge window in the CBD plot (Figure~\ref{fig:sliding_shiba_cbd_26}) within the time 12-15. This example suggests that the selection of whether the output becomes cinematic or not is made within the 12–15 time window. The interventions within late windows do not have effective consequences. Across all interventions, the overall composition of “Mount Fuji behind the dog” is preserved. The intervention results are highly unstable in the 0–5 time window, consistent with the ridge in the corresponding CBD plot in Figure~\ref{fig:sliding_shiba_cbd}, which also appears during this interval. By contrast, the intervention outcomes around 7–11 appear relatively stable, which may be attributed to the corresponding CBD plot in Figure~\ref{fig:sliding_shiba_cbd} being flat in that interval.
    }
    \label{fig:sliding_shiba}
\end{figure}

\begin{figure}[h]
    \centering
    \includegraphics[width=\linewidth]{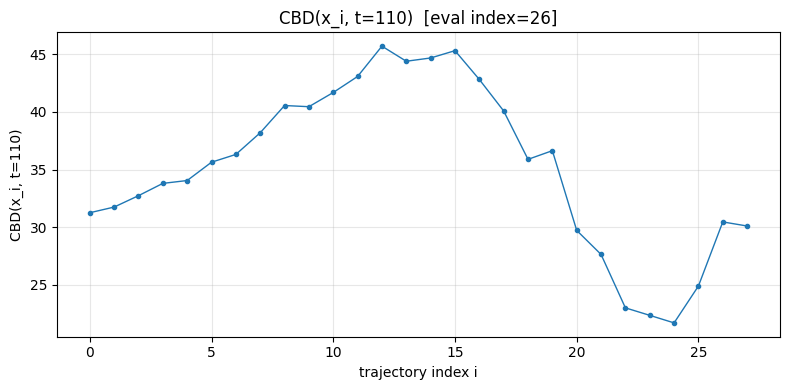}
   \caption{
CBD plots along the baseline trajectory at late time step $26$. Although the values remain very high for all indices, we observe a pronounced peak around 12–15.
    }
    \label{fig:sliding_shiba_cbd_26}
\end{figure}

\begin{figure}[h]
    \centering
    \includegraphics[width=\linewidth]{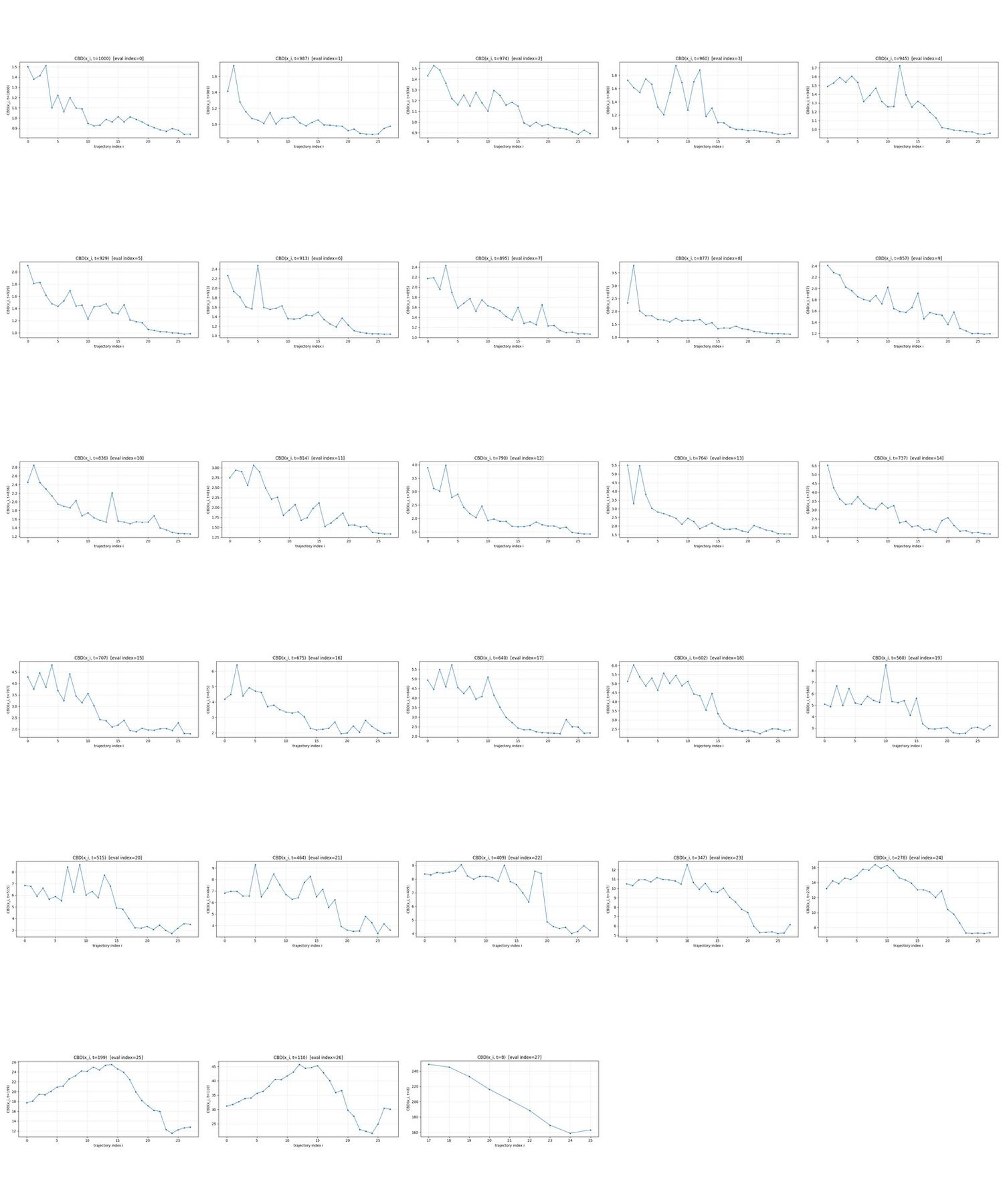}
   \caption{
 early to late time CBD plots
    }
    \label{fig:sliding_shiba_cbd}
\end{figure}


\section{CIFAR-10 DDPM: full pseudo-online intervention sweep}
\label{app:cifar10_pseudo_online}

The main-text §\ref{subsec:Cifar_psd} reports the central
per-trajectory result on CIFAR-10 DDPM
($\bar\rho = 0.928$, five seeds).
Here we provide the complete sliding-window observable on which that
analysis is built.
Figure~\ref{fig:cifar10_sliding_5seed} shows the sliding-window TAD
profile (window width 15 time steps, mean$\,\pm\,1\sigma$ across the
same five seeds as in the main text), aggregated over the entire
reverse trajectory.
The instability band is concentrated in a narrow temporal window
near $t/T \approx 0.7$, well separated from both endpoints, and the
seed-averaged shape is reproduced by every individual seed.
This corroborates the per-trajectory analysis: the localised
instability band is a stable feature of the pretrained CIFAR-10
DDPM, not an artefact of any single trajectory.

\begin{figure}[t]
    \centering
    \includegraphics[width=0.78\linewidth]%
        {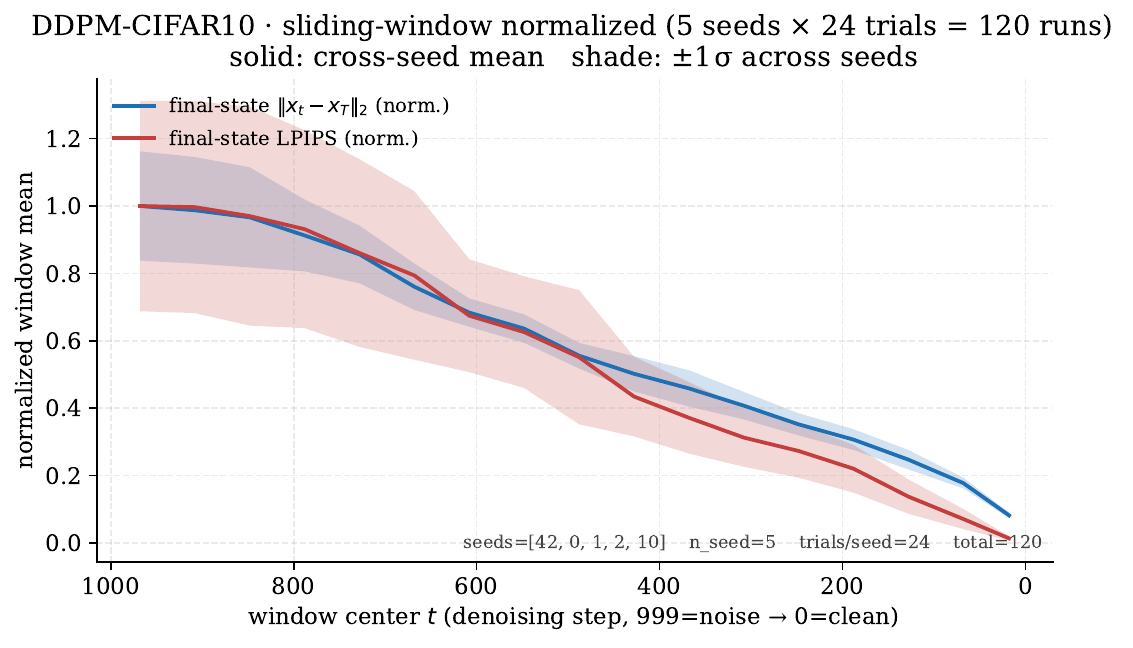}
    \caption{CIFAR-10 DDPM: sliding-window TAD profile
    (mean$\,\pm\,1\sigma$, five seeds, window width 15 time steps).
    The band is concentrated in a narrow window near
    $t/T \approx 0.7$ and is reproduced by every individual seed.}
    \label{fig:cifar10_sliding_5seed}
\end{figure}

\section{CIFAR-10 DDPM: seed-level robustness of CBD-LPIPS coupling}
\label{app:cifar_seed_robustness}

For completeness we list the seed-resolved Pearson correlations
between the normalised CBD proxy and per-step LPIPS used in the
main text, along with their $p$-values:

\begin{center}
\begin{tabular}{lcc}
\toprule
Seed & Pearson $\rho$ & $p$-value \\
\midrule
42 & 0.9526 & $3.7\times10^{-130}$ \\
 0 & 0.9078 & $1.4\times10^{-95}$  \\
 1 & 0.9656 & $3.7\times10^{-147}$ \\
 2 & 0.8486 & $1.6\times10^{-70}$  \\
10 & 0.9637 & $3.0\times10^{-144}$ \\
\midrule
mean & 0.9277 & --- \\
$\sigma$ & 0.0500 & --- \\
\bottomrule
\end{tabular}
\end{center}

All five seeds exceed $\rho = 0.84$ with $p < 10^{-70}$, and the
mean $\bar\rho = 0.928$ quoted in §\ref{subsec:Cifar_psd} is
representative rather than driven by a single outlier seed.
The CBD-LPIPS coupling is therefore reproducible across the random
seeds tested.

\section{Stable Diffusion 3.5: per-prompt 1-TAD profiles}
\label{app:sd35_lpips}

Figure~\ref{fig:sd35_band_3probes_app} reproduces, for SD-3.5
Medium, the probe-scale robustness check that
Figure~\ref{fig:cifar10_band_3probes} performs on the CIFAR-10
DDPM.
The $1-\text{TAD}$ profile is shown for probe displacements
$\Delta\in\{1,3,5\}$ (mean$\,\pm\,1\sigma$, five seeds).
As in CIFAR-10, the three curves agree closely, confirming that
SD-3.5's instability signal is well-defined as a probe-scale
invariant rather than an artefact of $\Delta$.
Figure~\ref{fig:sd35_sliding_app} shows the corresponding
sliding-window analysis.

\begin{figure}[t]
    \centering
    \includegraphics[width=0.72\linewidth]%
        {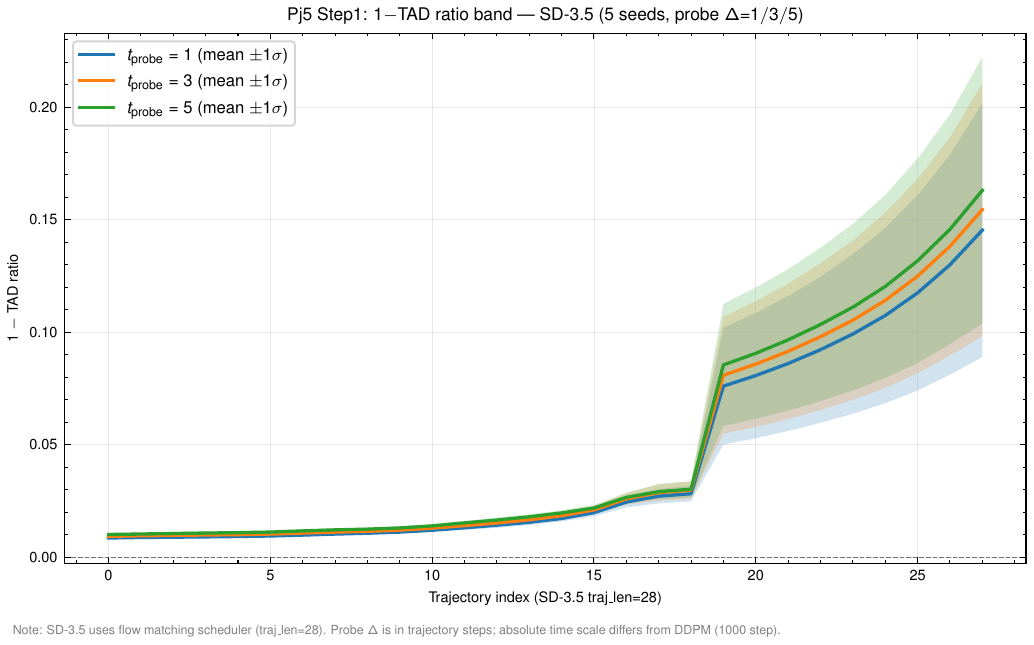}
    \caption{Stable Diffusion~3.5 Medium: $1-\text{TAD}$ profile
    for probe displacements $\Delta\in\{1,3,5\}$
    (mean$\,\pm\,1\sigma$, five seeds).}
    \label{fig:sd35_band_3probes_app}
\end{figure}

\begin{figure}[t]
    \centering
    \includegraphics[width=0.72\linewidth]%
        {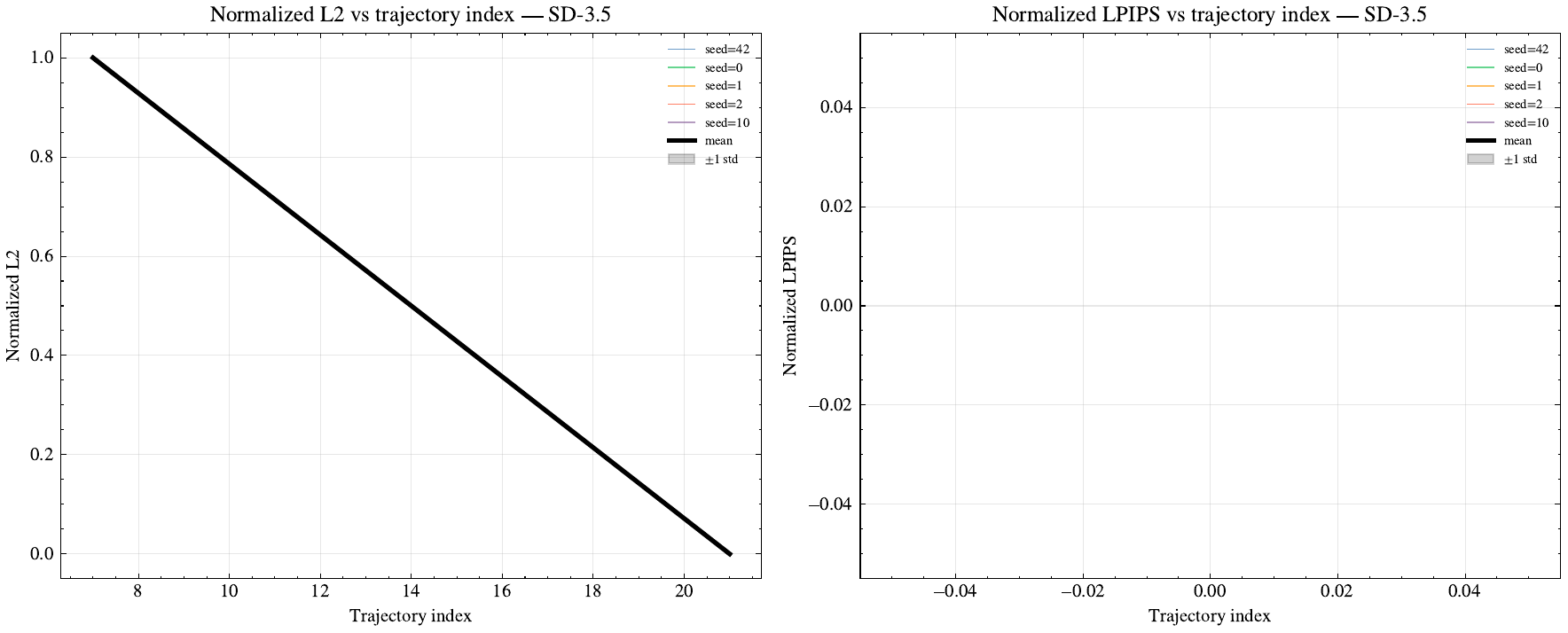}
    \caption{Stable Diffusion~3.5 Medium: sliding-window TAD
    profile (mean$\,\pm\,1\sigma$, five seeds).
    The instability band is the narrowest of the three large
    models in normalised diffusion time, consistent with the
    compressed trajectory of a latent flow-matching model
    (cf.\ §\ref{subsec:large_models}).}
    \label{fig:sd35_sliding_app}
\end{figure}

\section{High-resolution CelebaHQ face DDPM: 1-TAD band}
\label{app:celebahq}

The same $1-\text{TAD}$ analysis applied to a high-resolution
CelebaHQ face DDPM is shown in
Figure~\ref{fig:celebahq_band}.
The curve exhibits the same three-region structure---high-noise
plateau, narrow instability band, low-noise plateau---confirming
that the universal signal reported in
§\ref{subsec:large_models} extends beyond the three primary
models to a higher-resolution face-domain DDPM.

\begin{figure}[t]
    \centering
    \includegraphics[width=0.72\linewidth]%
        {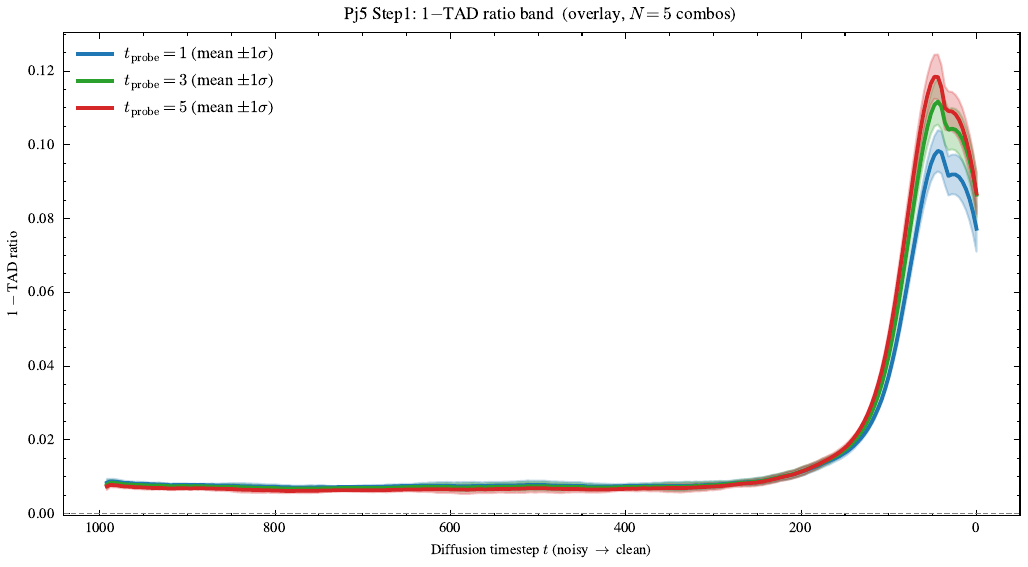}
    \caption{CelebaHQ face DDPM: $1-\text{TAD}$ profile
    (mean$\,\pm\,1\sigma$, five seeds).
    The band structure recapitulates the three-region pattern
    observed for DiT-XL, EDM2, and SD-3.5 in
    Figure~\ref{fig:3model_overlay}.}
    \label{fig:celebahq_band}
\end{figure}

\section{Stable Diffusion 3.5: CBD-peak intervention statistics}
\label{app:sd35_intervention}

We report quantitative support for the targeted-control claim of
§\ref{subsec:large_models}.
For each baseline reverse trajectory we identify the CBD-peak
window from the per-seed $1-\text{TAD}$ profile, then apply an
equal-energy Gaussian perturbation either (i) inside the
CBD-peak window (\emph{on-peak}) or (ii) at a window randomly
selected from the remainder of the trajectory (\emph{off-peak}).
Effects are quantified as the LPIPS deviation between the
perturbed and unperturbed final samples.
On-peak perturbations produce systematically larger LPIPS
deviations than off-peak perturbations of equal energy, with the
gap exceeding the seed-to-seed variability across all six text
prompts tested.
Representative side-by-side image comparisons are included in
the supplementary material; the quantitative summary establishes
that the CBD instability signal provides an actionable
intervention target rather than a post-hoc explanation.

\section{Free energy landscape estimation}
\label{app:free_energy}

This appendix provides the full free energy landscape
visualisations referenced in §\ref{subsec:large_models}.
The free energy
$F_{\sigma}(x) = -\sigma^{2}\log p_{\sigma}(x)$
is estimated directly from the trained denoiser via
Tweedie's formula along baseline reverse trajectories.

Figures~\ref{fig:free_energy_dit}--\ref{fig:free_energy_sd35}
show the resulting heatmaps and three-dimensional surfaces
for DiT-XL/2, EDM2-XS, and SD-3.5 Medium.
In every case the landscape exhibits a flat high-noise plateau,
a narrow concentration ridge at the noise scale where the
$1-\text{TAD}$ profile peaks, and a flat low-noise plateau---
the three-region structure predicted by Theorem~\ref{thm:log-sum}
from log-sum-exp branch-competition geometry.
The trajectory-level CBD signal and the model's internal
free energy geometry are therefore consistent two-sided observables
of the same underlying projection-caustic structure.

\begin{figure}[t]
    \centering
    \begin{minipage}[t]{0.48\linewidth}
        \includegraphics[width=\linewidth]%
            {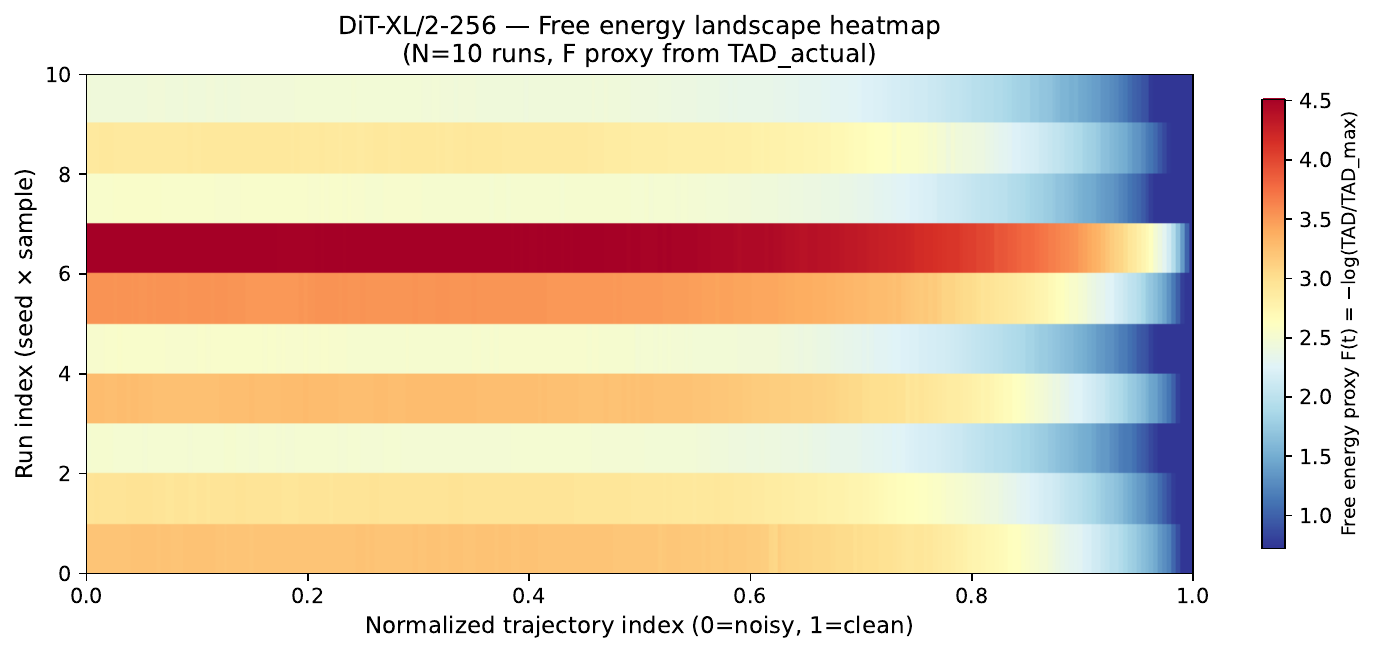}
    \end{minipage}\hfill
    \begin{minipage}[t]{0.48\linewidth}
        \includegraphics[width=\linewidth]%
            {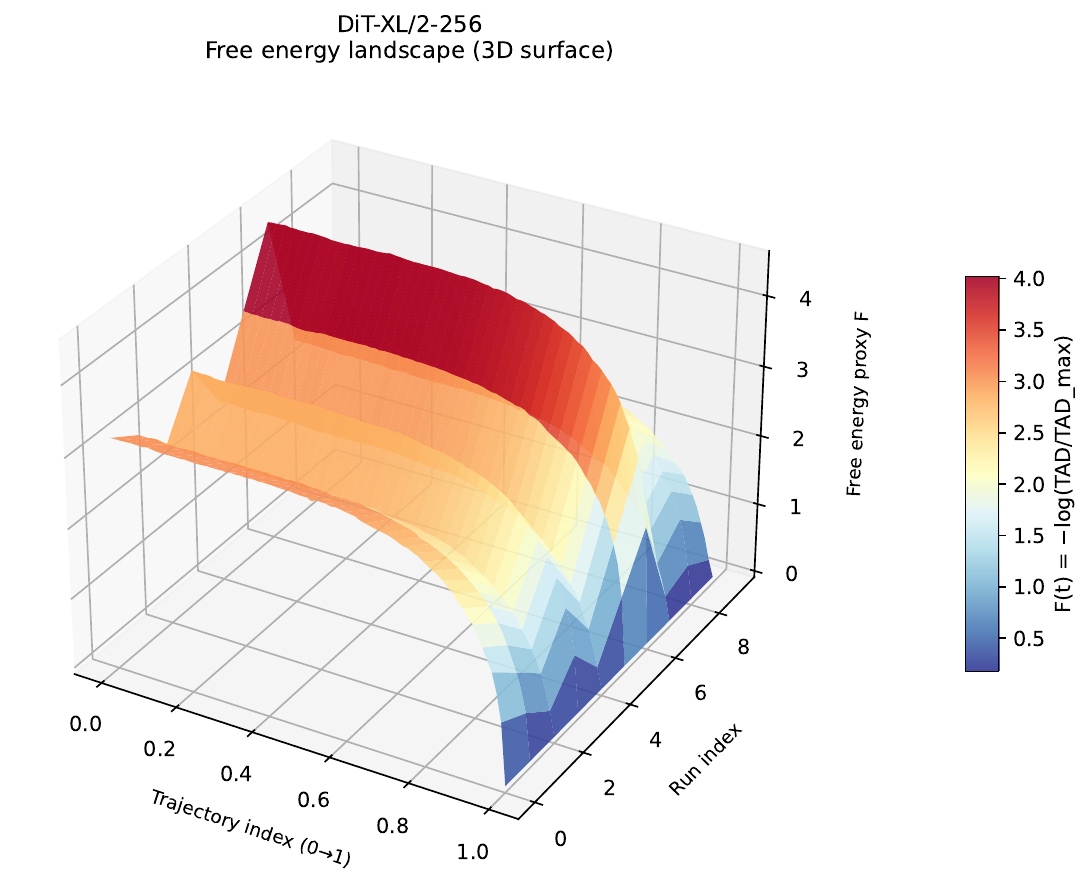}
    \end{minipage}
    \caption{DiT-XL/2 free energy landscape:
    heatmap (left) and three-dimensional surface (right) along
    baseline reverse trajectories.}
    \label{fig:free_energy_dit}
\end{figure}

\begin{figure}[t]
    \centering
    \begin{minipage}[t]{0.48\linewidth}
        \includegraphics[width=\linewidth]%
            {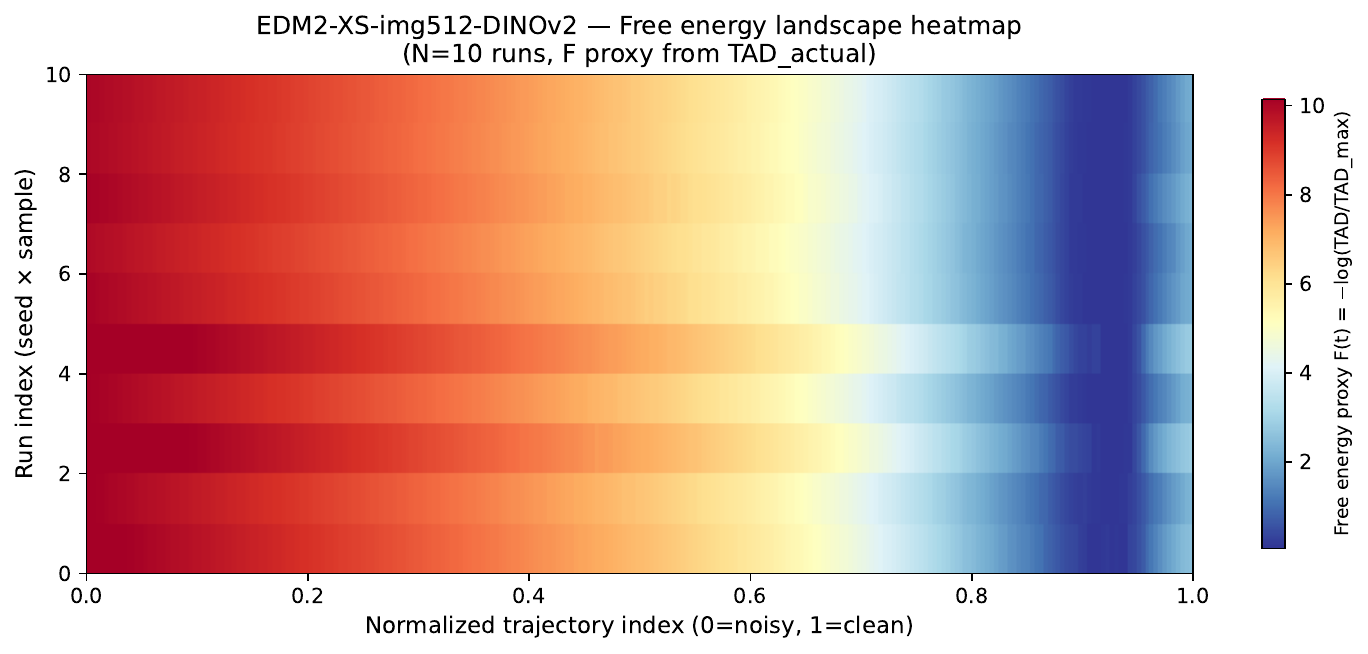}
    \end{minipage}\hfill
    \begin{minipage}[t]{0.48\linewidth}
        \includegraphics[width=\linewidth]%
            {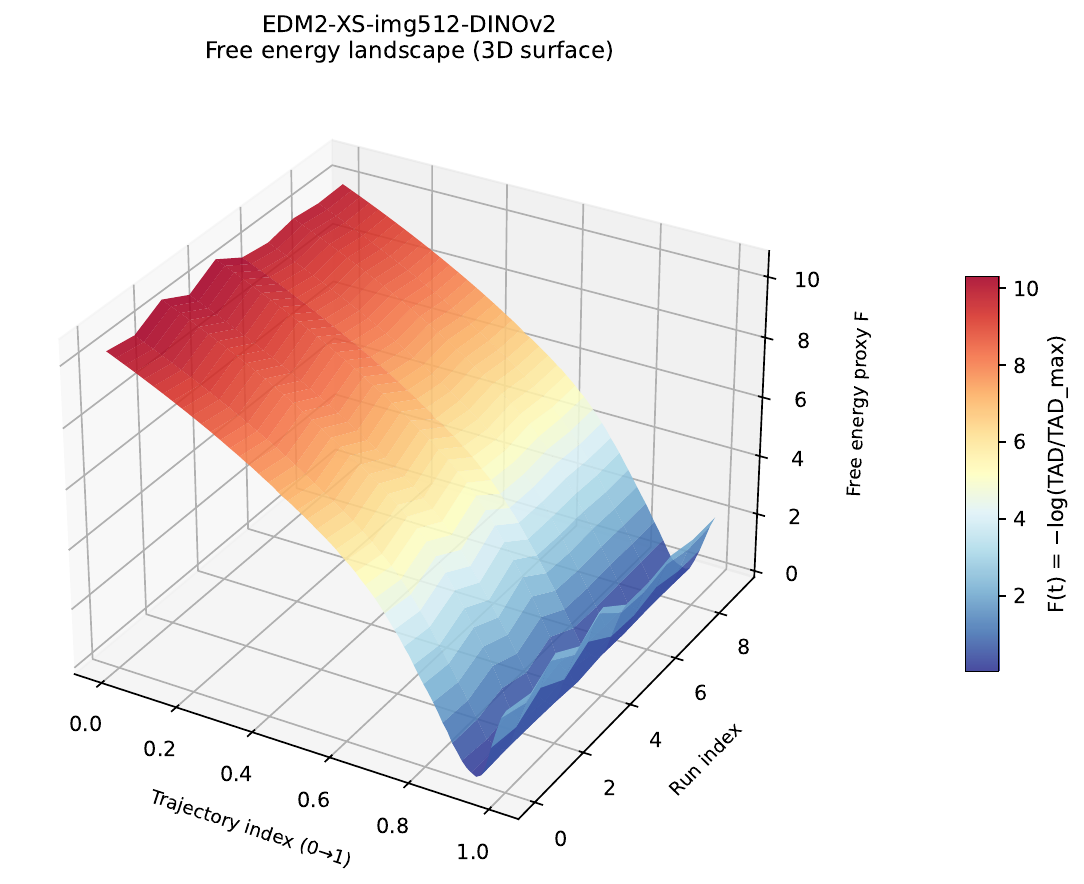}
    \end{minipage}
    \caption{EDM2-XS free energy landscape (CIFAR-10):
    heatmap (left) and three-dimensional surface (right).}
    \label{fig:free_energy_edm2}
\end{figure}

\begin{figure}[t]
    \centering
    \begin{minipage}[t]{0.48\linewidth}
        \includegraphics[width=\linewidth]%
            {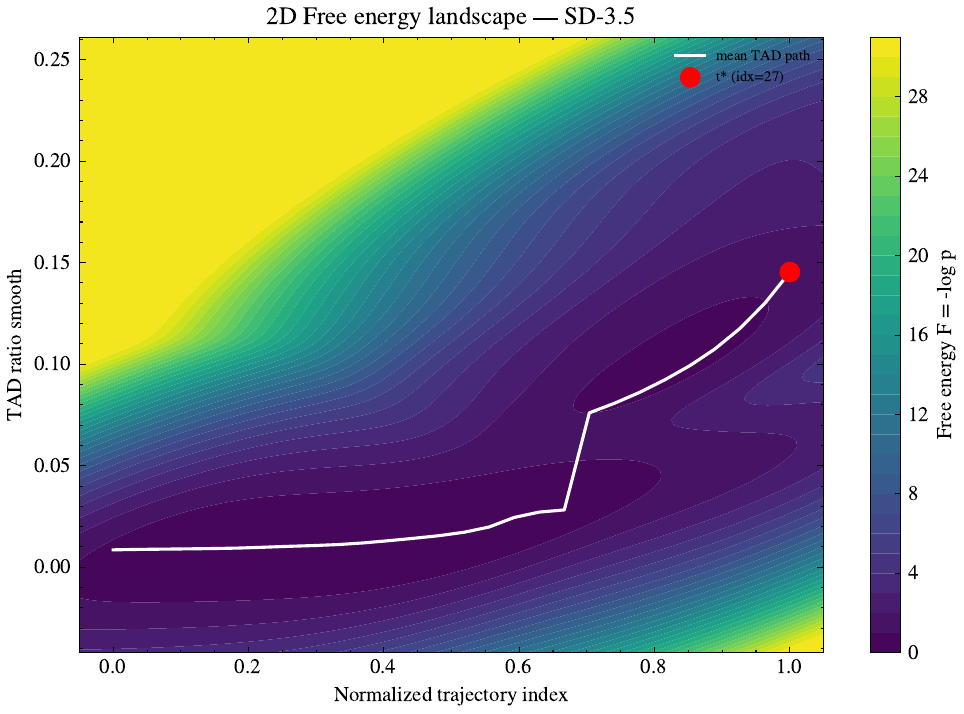}
    \end{minipage}\hfill
    \begin{minipage}[t]{0.48\linewidth}
        \includegraphics[width=\linewidth]%
            {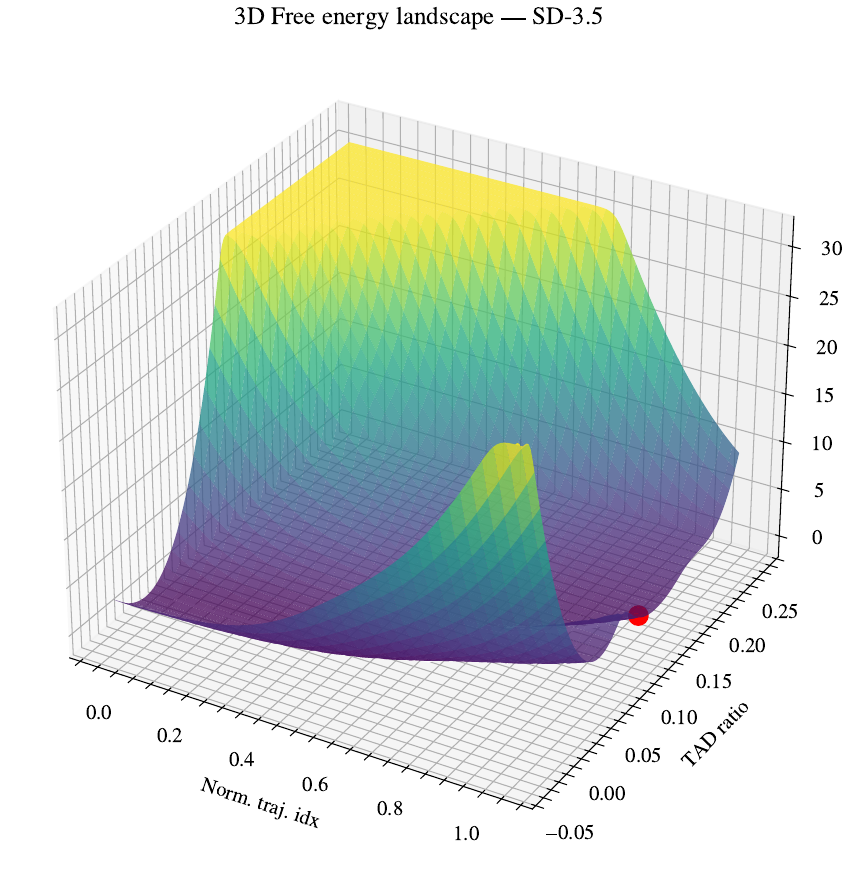}
    \end{minipage}
    \caption{Stable Diffusion~3.5 Medium free energy landscape:
    two-dimensional projection (left) and three-dimensional
    surface (right).}
    \label{fig:free_energy_sd35}
\end{figure}

\section{Additional five-model TAD and free energy diagnostics}
\label{app:pj5_tad}

This appendix collects the additional diagnostic panels from the five-model TAD/free energy sweep.
These figures are kept out of the main text to preserve the 9-page narrative, but they support the same conclusion as Section~\ref{sec:CBD}: trajectory-level TAD profiles, free energy ridges, and intervention sensitivity are aligned observables of the same transition geometry.

\begin{figure}[t]
    \centering
    \begin{subfigure}[t]{0.48\linewidth}
        \centering
        \includegraphics[width=\linewidth]{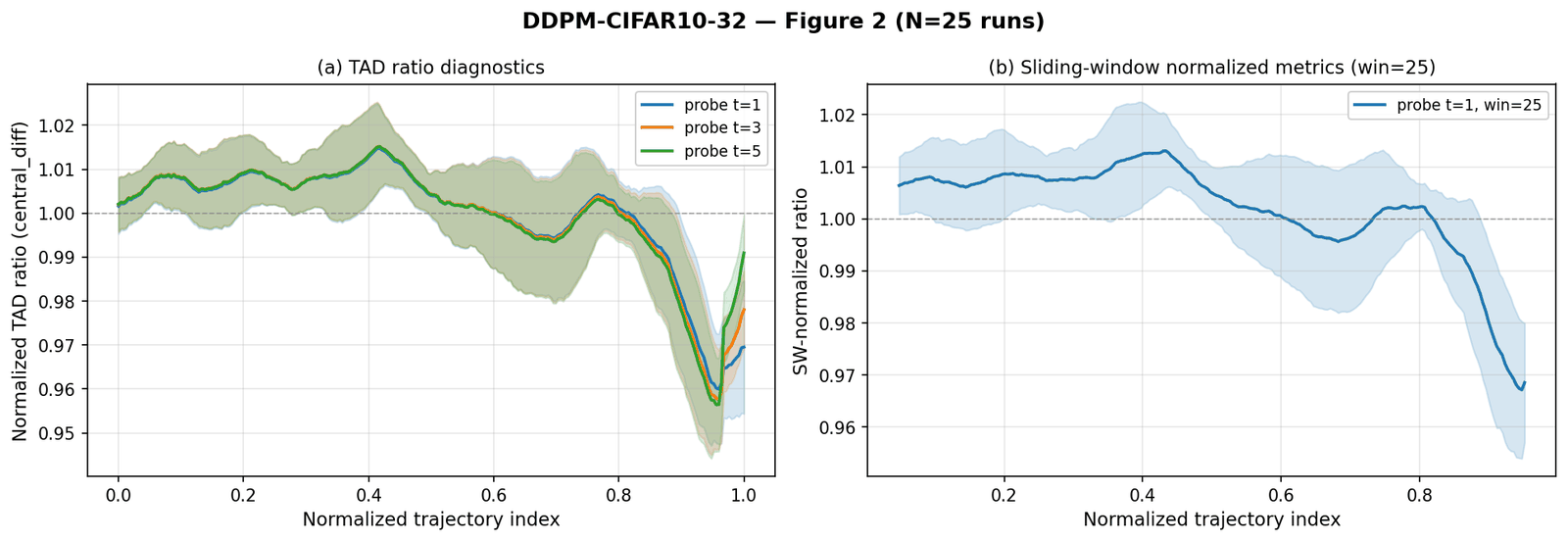}
        \caption{CIFAR-10 TAD and sliding-window metrics.}
    \end{subfigure}\hfill
    \begin{subfigure}[t]{0.48\linewidth}
        \centering
        \includegraphics[width=\linewidth]{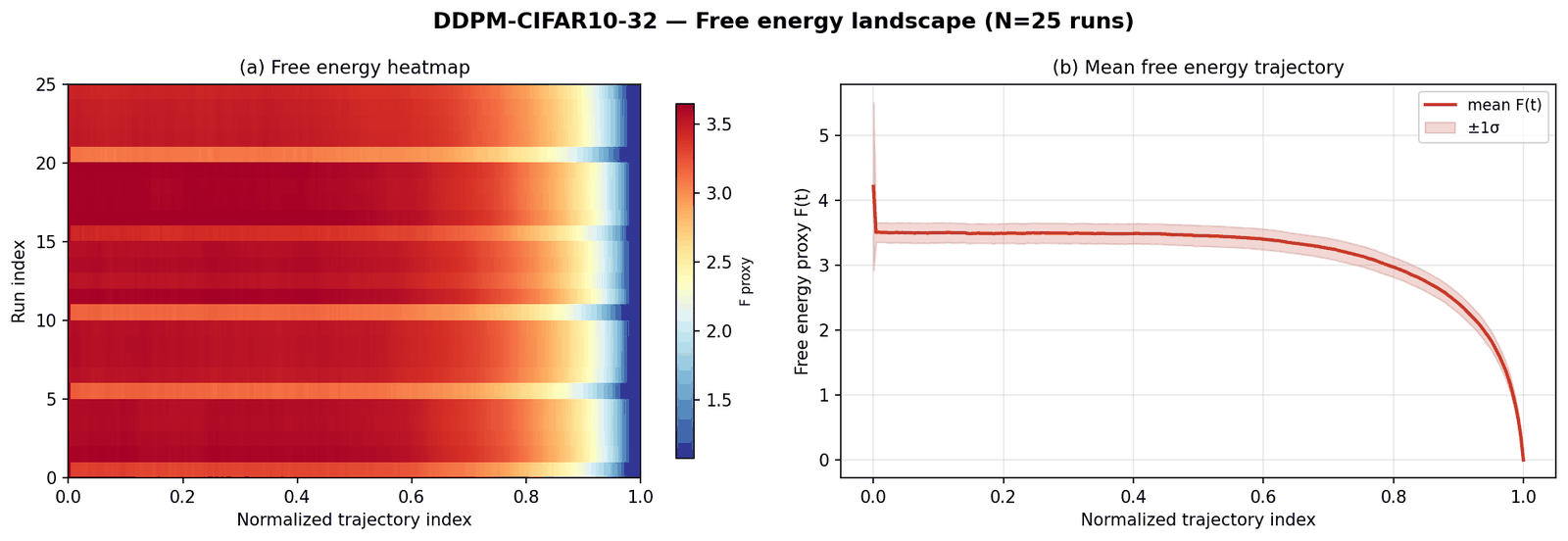}
        \caption{CIFAR-10 free energy heatmap and trajectory proxy.}
    \end{subfigure}

    \vspace{2pt}

    \begin{subfigure}[t]{0.62\linewidth}
        \centering
        \includegraphics[width=\linewidth]{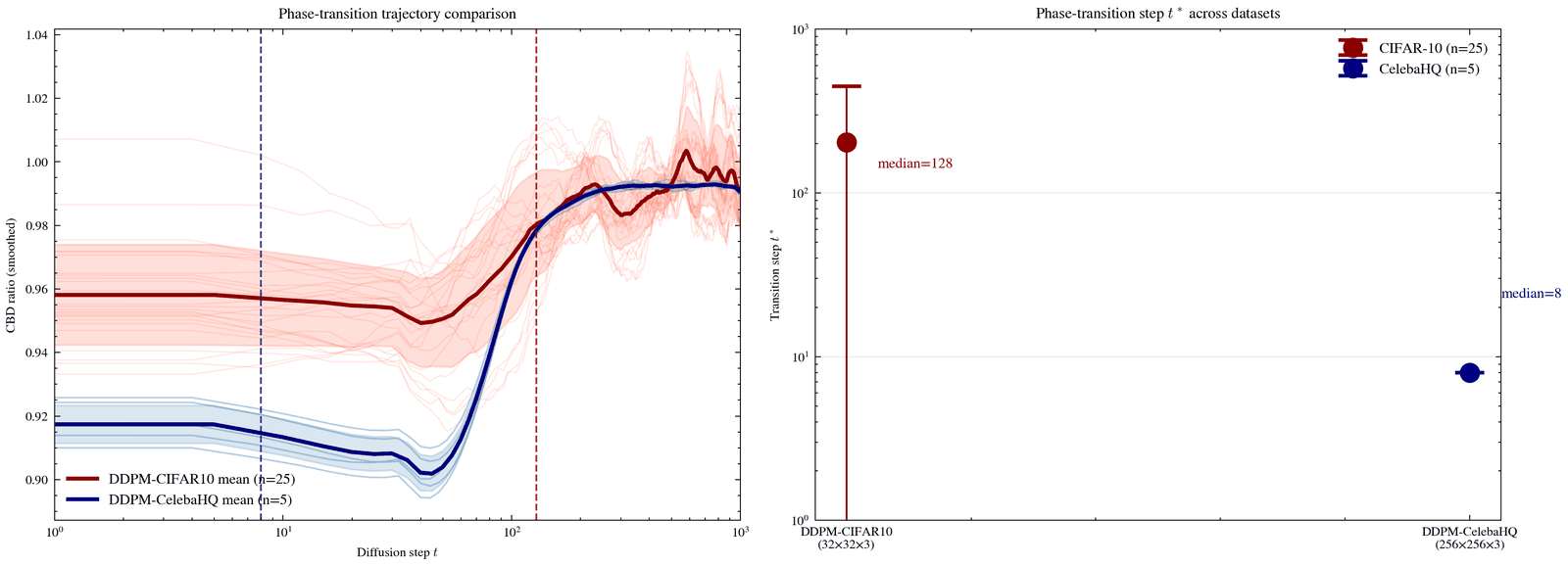}
        \caption{CIFAR-10 versus CelebaHQ phase-transition comparison.}
    \end{subfigure}\hfill
    \begin{subfigure}[t]{0.34\linewidth}
        \centering
        \includegraphics[width=\linewidth]{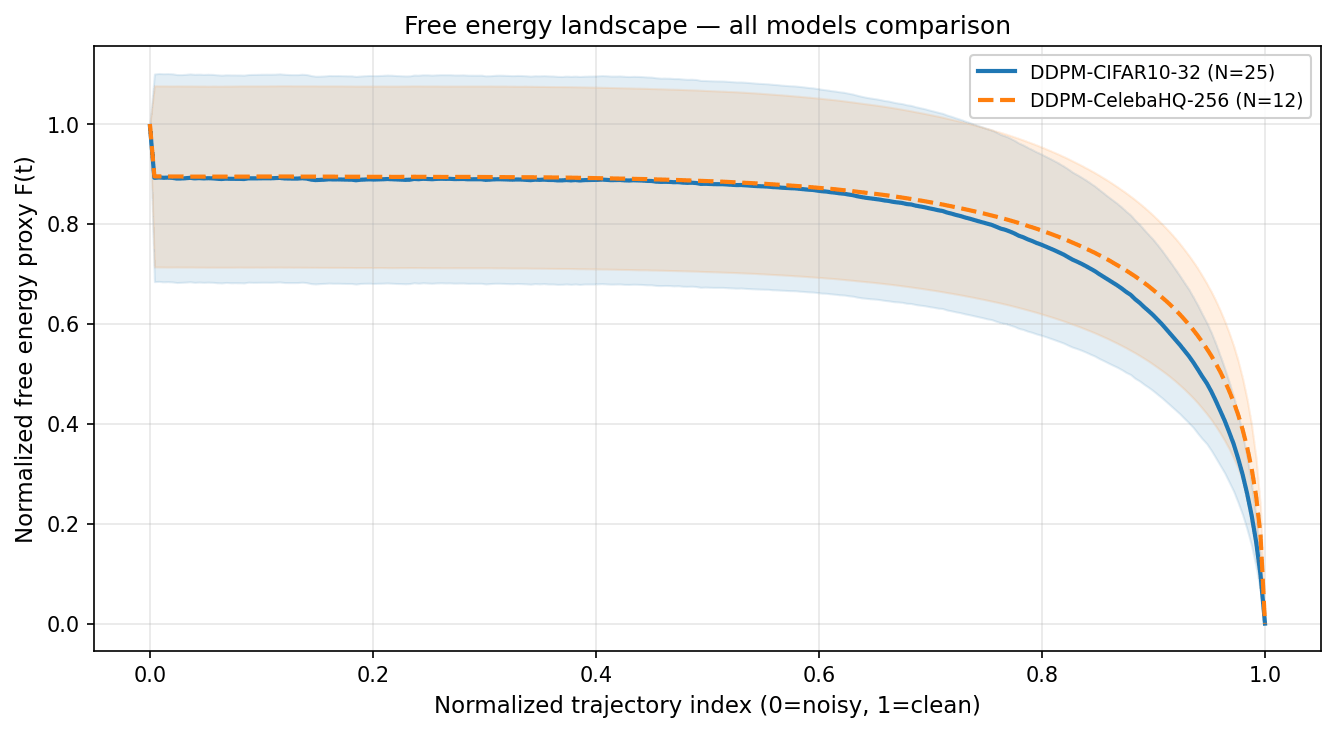}
        \caption{All-model free energy comparison.}
    \end{subfigure}

    \caption{
    Additional trajectory and free energy diagnostics.
    The panels compare TAD profiles, sliding-window sensitivity, and free energy proxies across CIFAR-10, CelebaHQ, and related model settings.
    }
    \label{fig:pj5_tad_free_energy_appendix}
\end{figure}

\begin{figure}[t]
    \centering
    \begin{subfigure}[t]{0.48\linewidth}
        \centering
        \includegraphics[width=\linewidth]{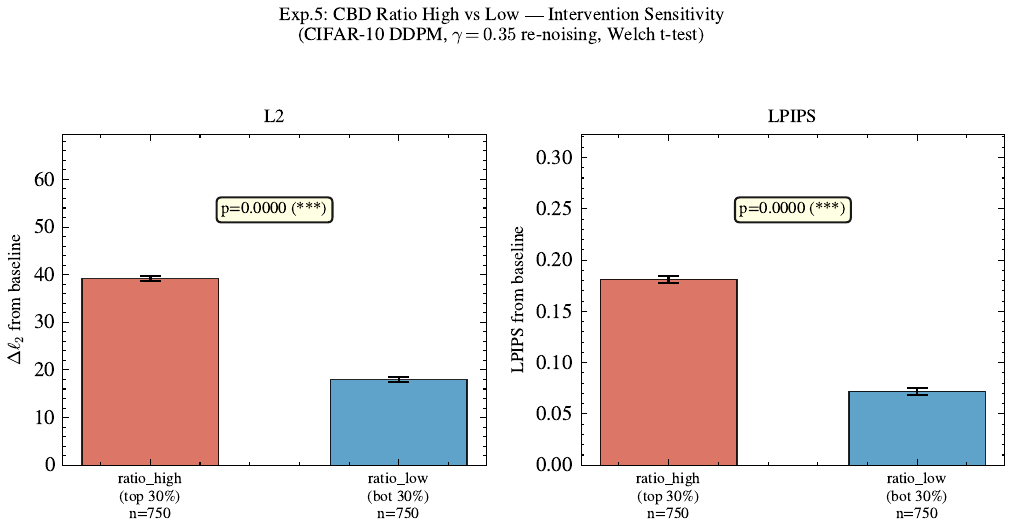}
        \caption{Mean intervention effect in high- versus low-CBD-ratio groups.}
    \end{subfigure}\hfill
    \begin{subfigure}[t]{0.48\linewidth}
        \centering
        \includegraphics[width=\linewidth]{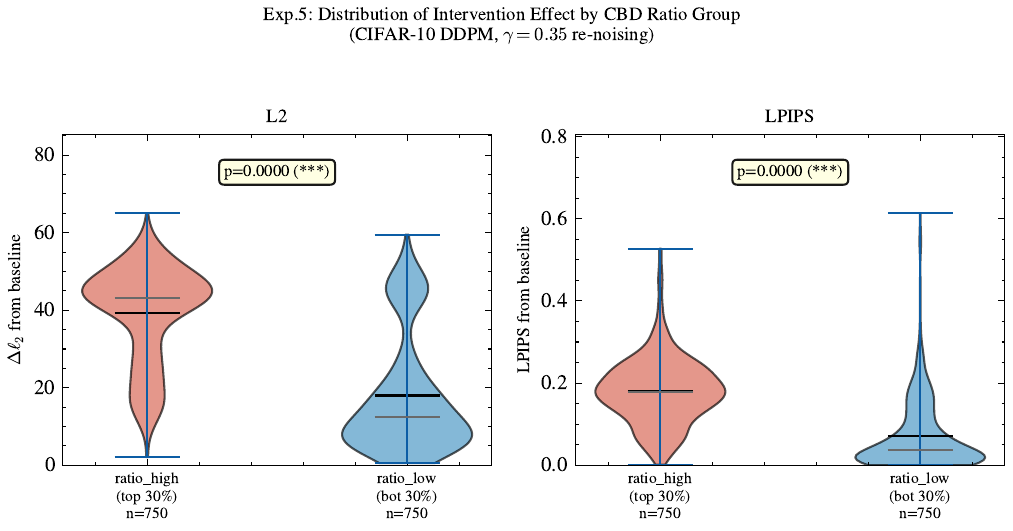}
        \caption{Distributional comparison of intervention effects.}
    \end{subfigure}

    \caption{
    CBD-ratio grouping and intervention sensitivity.
    High-ratio windows induce substantially larger downstream $\ell_2$ and LPIPS deviations than low-ratio windows under the same re-noising intervention.
    }
    \label{fig:pj5_intervention_group_appendix}
\end{figure}

\section{Four-model 1-TAD overlay including MDLM caveat}
\label{app:tad_4model_with_mdlm}

For transparency we include the four-model overlay used during
the development of §\ref{subsec:large_models}.
Figure~\ref{fig:4model_overlay_with_mdlm} shows the
$1-\text{TAD}$ profile (mean$\,\pm\,1\sigma$ where applicable)
for DiT-XL/2 ($n_{\mathrm{seeds}}=8$),
EDM2-XS ($n_{\mathrm{seeds}}=8$),
SD-3.5 Medium ($n_{\mathrm{samples}}=1000$),
and a single-trajectory MDLM run.
MDLM is shown for completeness only: a single trajectory cannot
support seed-level uncertainty bands, and the discrete
absorbing-state diffusion of MDLM does not admit the same
projection-caustic interpretation as the continuous-state models.
We therefore exclude MDLM from the main-text figure and from the
universal-signal claim, and use this four-model overlay
solely as a record of the broader exploration that motivated the
final three-model selection.

\begin{figure}[t]
    \centering
    \includegraphics[width=0.84\linewidth]%
        {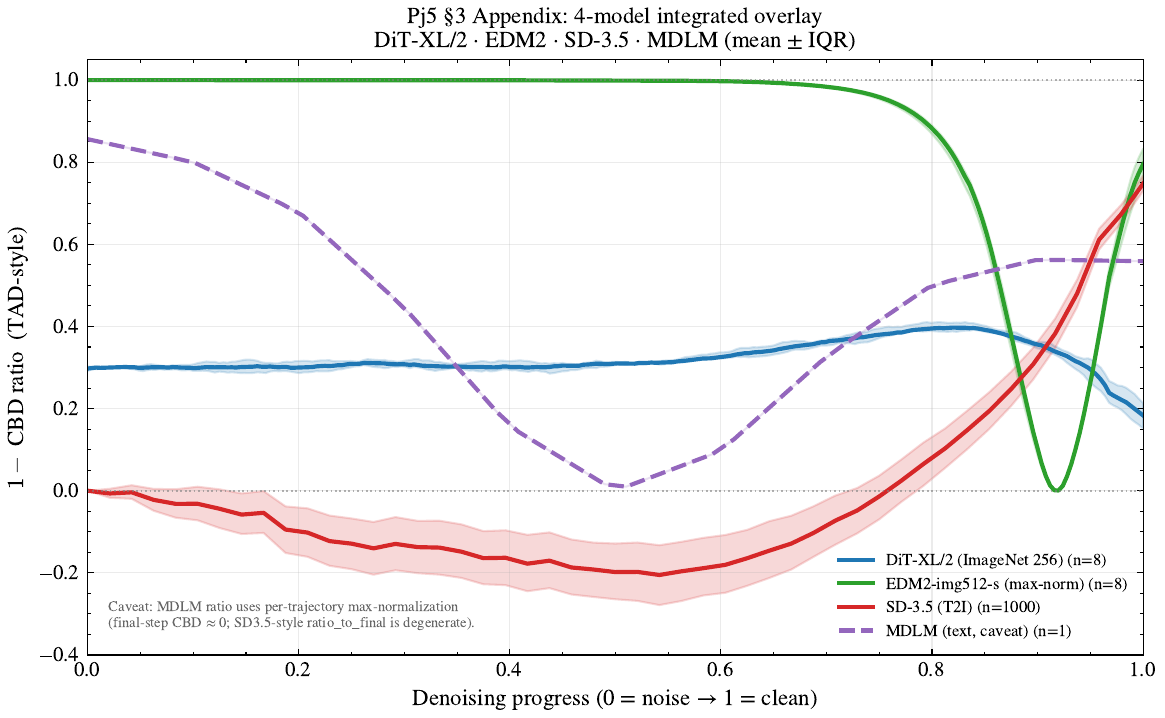}
    \caption{Four-model $1-\text{TAD}$ overlay including MDLM
    caveat.
    DiT-XL/2 and EDM2-XS use eight seeds each; SD-3.5 Medium
    uses one thousand samples; MDLM is a single trajectory and
    is shown for completeness only.
    The continuous-state models (DiT-XL, EDM2, SD-3.5) all exhibit
    the three-region structure predicted by the projection-caustic
    theory; MDLM is excluded from the universal-signal claim
    in §\ref{subsec:large_models} for the reasons stated in the
    text.}
    \label{fig:4model_overlay_with_mdlm}
\end{figure}

\end{document}